\documentclass{article}

\usepackage{amsmath,amssymb,amsfonts}
\usepackage{amsthm}
\usepackage{algorithm}
\usepackage{algorithmic}
\usepackage{multirow}
\usepackage{caption}

\usepackage{thmtools} 
\usepackage{thm-restate}
\usepackage{mathtools}
\mathtoolsset{showonlyrefs}
\newtheorem{thm}{Theorem}[section]
\newtheorem{lem}[thm]{Lemma}
\newtheorem{remark}[thm]{Remark}

\newtheorem{defin}[thm]{Definition}

\usepackage{enumitem}
\setitemize{label=\textbullet, leftmargin=*, nolistsep}
\allowdisplaybreaks[2]
\usepackage{enumitem} 
\usepackage{longtable}
\usepackage{setspace}

\usepackage{bbm} 

\usepackage{microtype}
\usepackage{graphicx}
\usepackage{subfigure}
\usepackage{booktabs} 

\usepackage{hyperref}

\usepackage[accepted]{icml2023}

\newcommand{\bsigma}{\bar{\sigma}}
\newcommand{\upbdcstt}{D}
\newcommand{\Reg}{\mathrm{Reg}}
\newcommand{\Safeset}{\mathcal{S}}
\newcommand{\Subopt}{\mathcal{B}}
\newcommand{\Riskset}{\mathcal{R}}
\newcommand{\Deltav}{\Delta^\mathrm{v}}
\newcommand{\AlgSCombUCB}{{\sc PASCombUCB }}
\newcommand{\rmv}{\mathrm{v}}
\newcommand{\rmR}{\mathrm{R}}
\newcommand{\rmu}{\mathrm{u}}
\newcommand{\rml}{\mathrm{l}}
\newcommand{\phase}{p}
\newcommand{\addsln}{l}
\newcommand{\calG}{\mathcal{G}}
\newcommand{\calGmu}{\mathcal{G}^\mu}
\newcommand{\abcstt}{C}
\newcommand{\bomega}{\bar{\omega}}
\newcommand{\bDelta}{\bar{\Delta}}

\newcommand{\tomega}{\tilde{\omega}}
\newcommand{\rmsum}{\mathrm{sum}}
\newcommand{\instance}{\Lambda}
\newcommand{\Tp}{T^\prime}
\newcommand{\Regsub}{\mathrm{Reg}_{\mathrm{1}}}
\newcommand{\Regsafe}{\mathrm{Reg}_{\mathrm{2}}}
\newcommand{\Regfail}{\mathrm{Reg}_{\mathrm{3}}}
\newcommand{\Deltamu}{\Delta^\mu}
\newcommand{\epsilonv}{\epsilon^{\mathrm{v}}}
\newcommand{\epsilonmu}{\epsilon^{\mu}}
\newcommand{\safe}{\mathrm{safe}}
\newcommand{\bbN}{\mathbb{N}}
\newcommand{\bern}{\mathrm{Bern}}
\newcommand{\CombUCB}{\textsc{CombUCB1} }

\newcommand{\calA}{\mathcal{A}}
\newcommand{\calE}{\mathcal{E}}
\newcommand{\calH}{\mathcal{H}}
\newcommand{\bbE}{\mathbb{E}}
\newcommand{\bbP}{\mathbb{P}}
\newcommand{\rmKL}{\mathrm{KL}}

\DeclareMathOperator*{\argmax}{arg\,max}

\usepackage{hhline}
\makeatletter
\newcommand{\thickhline}{%
	\noalign {\ifnum 0=`}\fi \hrule height 1.5pt
	\futurelet \reserved@a \@xhline
}

\usepackage{url}

\begin{document}
	
	\twocolumn[
 \icmltitle{Probably Anytime-Safe Stochastic Combinatorial Semi-Bandits} 
	\icmlsetsymbol{equal}{*}
	
	%
	
	\begin{icmlauthorlist}
\icmlauthor{Yunlong Hou}{math}     
		\icmlauthor{Vincent Y.~F.~ Tan}{math,ece,iora}  
        \icmlauthor{Zixin Zhong}{cs}  
	\end{icmlauthorlist}
	
	\icmlaffiliation{math}{Department of Mathematics, National University of Singapore, Singapore}
    \icmlaffiliation{iora}{Institute of Operations Research and Analytics, National University of Singapore, Singapore}
    \icmlaffiliation{ece}{Department of Electrical and Computer Engineering, National University of Singapore, Singapore}
    \icmlaffiliation{cs}{Department of Computing Science, University of Alberta, Canada}
	\icmlcorrespondingauthor{Zixin Zhong}{zzhong10@ualberta.ca}
	
	\icmlkeywords{Machine Learning, ICML}
	
	\vskip 0.3in
	]
	
	
	
	\printAffiliationsAndNotice{}  

	\begin{abstract}
	Motivated by concerns about making online decisions that incur undue amount of risk at each time step, in this paper, we formulate the probably anytime-safe stochastic combinatorial semi-bandits problem. In this problem, the agent is given the option to select a subset of size at most $K$ from a set of $L$ ground items. Each item is associated to  a certain mean reward as well as a variance that represents its risk. To mitigate the risk that the agent incurs, we require that with probability at least $1-\delta$, over the entire horizon of time $T$, each of the choices that the agent makes should contain items whose sum of variances does not exceed a certain variance budget. We call this probably anytime-safe constraint. Under this constraint, we design and analyze an algorithm {\sc PASCombUCB} that minimizes the regret over the horizon of time $T$. By developing accompanying information-theoretic lower bounds, we show that under both the problem-dependent and problem-independent paradigms, {\sc PASCombUCB} is almost asymptotically optimal. 
    Experiments are conducted to corroborate our theoretical findings. Our problem setup, the proposed {\sc PASCombUCB} algorithm, and novel analyses are   applicable to domains such as  recommendation systems and transportation in which an agent is allowed to choose multiple items at a single time step and wishes to control the risk over the whole time horizon.
	\end{abstract}
	
		\vspace{-.15in}
	\section{Introduction}


Audrey, a  burgeoning social media influencer, makes profits by posting advertisements (ads) under her account. The advertiser pays her only if an ad is clicked.  
Having taken a class in online optimization, Audrey aims to leverage the theory of bandit algorithms to design an exploration-exploitation strategy to  ensure that the expected number of clicks of the ads she has posted is maximized. Since the platform is space-limited, Audrey can only post no more than $K$   out  of $L$ available ads everyday. Some of these ads, however,   include an innocuous-looking lottery or voucher that asks the viewer of the social media platform to provide personal information that may lead to fraud or information leakage. If a user clicks it and becomes a victim of fraud, this may damage Audrey's reputation. Audrey thus has to be circumspect in which and how many ads she posts. 

On the one hand, Audrey wants to post as many ads that she believes to have high click-through rates as possible; the  {\em expected reward} she obtains is then the sum of expected rewards of the individual ads. On the other hand, she should balance this with the  {\em total risk} of the ads that are posted over a period of time; similarly, the risk of a set of ads posted is modeled as the sum of the risks of the individual ads. How should  Audrey plan the  posts of her ads over a period of time to learn their individual expected rewards and risks to ensure that her total expected reward is maximized and, at the same time,  with high probability, the risk incurred {\em at any point in time} in her exploration-exploitation strategy is bounded by some fixed permissible threshold?

In addition to influencers like Audrey, online platforms that make profits by advertising such as YouTube and TikTok also encounter similar problems.
We are therefore motivated to
formulate the {\em probably anytime-safe stochastic combinatorial semi-bandits} (PASSCSB) problem which is a   {\em regret minimization} problem with an anytime safety constraint. More precisely, we aim to design and analyze the performance of an algorithm that,  with high probability, ensures that the risk (as measured by the variance) {\em at any time} step is below a given threshold and  whose regret is minimized. 

\textbf{\underline{Literature review.}}
There is a large body of works that take risk into account while conducting the exploration and/or exploitation of the unknown reward distributions in the stochastic multi-armed bandits (MABs) literature. 

Under the risk-constrained pure exploration framework, \citet{hou2023}  and \citet{David2018} attempted to identify the optimal arm within those low-risk (based on their variances or $\alpha$-quantiles) arms with probability at least $1-\delta$. 

Under the {\em risk-aware} regret minimization setup, \citet{Sani2013}, \citet{Vakili2016} and \citet{Zhu2020} consider the mean-variance as the measure to be minimized over a fixed time horizon. \citet{Cassel2018} provided a general and systematic instruction to analyzing risk-aware MABs, i.e., the risk was incorporated in the {\em Empirical Distribution Performance Measure} and the U-UCB algorithm is adopted to perform ``proxy regret minimization''.
While these risk-aware algorithms  reduce the overall risk during the exploration and exploitation process, the risk is not strictly enforced to be below a prescribed  threshold; rather the risk measure is penalized within the objective function, similarly to a Lagrangian. 
Another setup similar to the risk-aware setup is the {\em constrained} bandits regret minimization. \citet{Mahdavi2012efficiency} required that the number of times the constraint can only be violated is at most sublinear in the horizon $T$. \citet{Kagrecha2020Constrained} proposed a CVaR constraint and performed exploration on the feasible arm, followed by exploration among the feasible arm set. Unlike our formulation, these algorithm are permitted to sample  risky arms during  exploration.

A  more stringent constraint can be found in the literature on {\em conservative bandits} \citep{Wu2016Conservative}, which requires the cumulative return at any time step to be  above a constant fraction of the return resulting from repeatedly sampling the base arm. \citet{Kazerouni2017ConsContLB} extended this setup to conservative contextual linear bandits and this was  further improved by \citet{Garcelon2020Improved}. A similar problem is {\em bandits with knapsacks} \citep{Badanidiyuru2018}, which imposes a   budget on the cumulative consumed resources  and the algorithm stops when the budget is depleted.

The most stringent constraint can be found in the   {\em safe bandits} problem.
\citet{Khezeli2020Safe} and \citet{Moradipari2020Stagewise} presented the {\sc SEGE}, {\sc SCLUCB}, and {\sc SCLTS} algorithms to tackle this problem. This problem  demands that the expected reward of the pulled arm at each time step be greater than a prescribed threshold with high probability,  i.e.,  the {\em ``stagewise safety constraint''}. The authors assumed the convexity (continuity) of the arm set and performed exploration around the explored safe arms, starting from a baseline safe arm. This continuity of the (super) arm set does not hold under the combinatorial semi-bandits setup. More comparisons are presented in App.~\ref{sec:additional}.

For the (unconstrained) combinatorial semi-bandits (CSB) setup, \citet{Chen2013Combinatorial} presented a UCB-type algorithm {\sc ComUCB1} to balance the trade-off between exploration and exploitation. \citet{Tight2015Kveton} improved the analysis of {\sc ComUCB1} and achieved a tight upper bound (within a specific set of instances). \citet{Matroid2014Kveton} introduced matroid structure to CSB and leveraged the matroid structure to design and analyze a greedy algorithm {\sc OMM}.
The risk-aware CSB problem is less studied by the community. \citet{Ayyagari2021Risk} utilized CVaR as the risk-aware measure within the CSB problem, where the risk constraint was not explicitly specified. 

We observe that the existing literature mentioned above are not directly applicable to Audrey, while
our setting  (described formally below) dovetails neatly with her problem.
Audrey can utilize our algorithm to sequentially and adaptively select different sets of ads everyday and almost always (i.e., with high probability) avoids  sets of ads with unacceptably high risks.
Beyond any specific applications,  we believe that this problem is of fundamental theoretical importance in the broad context of regret minimization in combinatorial multi-armed bandits.

\noindent \underline{{\bf Main Contributions.}}
Our first contribution lies at the formulation of a novel PASSCSB problem.
In the PASSCSB problem, there are $L$ items with different reward distributions. At each time step, a random reward is generated from each item's distribution.
Based on the previous observations, the learning agent selects a {\em solution} at each time step. A solution consists of at most $K$ items. The expected return (variance) of a solution is the summation of the reward (variance) of its constituents. 
Given $T\in\mathbb{N}$, the agent aims to maximize the cumulative return over $T$ time steps and ensure that with probability $1-\delta$ the variance of all selected solutions are below a given threshold.

The key challenge of regret minimization under the PASSCSB lies in handling two distinct tasks---we seek optimality in the mean and safeness in the
variance  of each chosen solution. 
Our second contribution is the design and analysis of the \underline{P}robably \underline{A}nytime-\underline{S}afe \underline{Comb}inatorial \underline{UCB}  (or {\sc PASCombUCB}) algorithm. 

Thirdly, we also derive a problem-dependent upper bound on the regret of {\sc PASCombUCB}, which involves a {\em hardness} parameter $H( \Delta(\instance))$. We see that $H( \Delta(\instance))$ characterizes the effectiveness of ascertaining the safety of potential solutions in the regret.
To assess the optimality of {\sc PASCombUCB}, we prove an accompanying problem-dependent lower bound on the regret of any variance-constrained consistent
algorithm. The upper and lower problem-dependent bounds match in almost all the  parameters (except in $K$). Additionally, we show that if $\delta_T$ decays exponentially fast in $T$, the problem-dependent regret cannot be logarithmic in $T$.  
We further present a problem-independent upper bound on the regret of {\sc PASCombUCB} and a lower bound for any algorithm. Just as the problem-dependent bounds, these bounds also match in almost all the parameters.  

Lastly, experiments are conducted to illustrate the empirical performance and corroborate our theoretical findings.

In summary, this paper is the first to explore the regret minimization problem in the combinatorial bandits with an {\em anytime} constraint on the variance.
When $\delta\to 1$ and $\bsigma^2$ is large (so that the optimal safe solution is the one with the highest mean regardless of safety considerations), our problem reduces to the standard combinatorial semi-bandits~\citep{Tight2015Kveton}, and the regret incurred by the safety constraint vanishes, resulting in the same upper bound as the unconstrained case.
Furthermore, the framework and analysis of {\sc PASCombUCB} can be extended to other risk measures as long as there are appropriate concentration bounds, e.g.,   \citet{Bhat2019} or \citet{chang2021unifying} enables us to use CVaR or certain continuous functions as  risk measures within the generic {\sc PASCombUCB} framework. 

	\section{Problem Setup}\label{sec:prob_setup}
	For $ m\in\mathbb{N} $,  let $ [m]:=\{1,2,\ldots,m\} $.
	An instance of a  {\em variance-constrained stochastic combinatorial semi-bandit}  is a tuple $\instance=(E,\mathcal{A}_K,\nu,\bsigma^2)$. We describe the four elements of $\instance$ in the following. 
    Firstly,  the finite set $E = [L]$ is known as the {\em ground set} in which each $i\in E$ is known as an {\em item}. 
    Secondly, the family 
    $\mathcal{A}_K\subset\{S\in 2^E: |S|\leq K\}$
    is a collection of subsets of $E$ with cardinality at most $K$. Each element $S\in\mathcal{A}_K$ is known as a {\em solution} and $\mathcal{A}_K$ satisfies the condition that all subsets of $S \in \mathcal{A}_K$ remain solutions, i.e., $\mathcal{A}_K$ is  downward-closed. 
    Thirdly, the vector of probability distributions $\nu = (\nu_1, \nu_2, \ldots, \nu_L)$ contains $\sigma^2$-sub-Gaussian distributions $\{\nu_i\}_{i\in E}$  with means $\{\mu_i\}_{i\in E}$ and variances $\{\sigma_i^2\}_{i\in E}$.  
	The final element of an instance $\bsigma^2>0$ denotes the permissible upper bound on the variance. To avoid trivialities, we assume that $\bsigma^2>\sigma^2$ and $K\geq 2$.

    The {\em return} of item $i \in E$ is the random variable $W_i$   with distribution $\nu_i$.  
	The {\em (stochastic) return} of a solution $S\in\mathcal{A}_K$
    is $\sum_{i\in S}W_i$ where $W_i\sim \nu_i$.
	The {\em expected return} and {\em variance} of $S\in\mathcal{A}_K$ are 
	\begin{align}
		\mu_S:=\sum_{i\in S}\mu_i\quad\mbox{and}\quad \sigma_S^2:=\sum_{i\in S}\sigma^2_i
	\end{align}
 respectively.
    We further assume that every instance  $\instance$ satisfies  $\sigma_S^2\neq\bsigma^2$ for all  $S\in\mathcal{A}_K$ and each distribution $\nu_i$ is supported  in the interval $[0,1]$.

 
	Define $\Safeset:=\{S\in\mathcal{A}_K:\sigma_S^2< \bsigma^2\}$ to be the {\em safe set} which contains all the {\em safe} solutions. Let the complement of $\Safeset$ be the {\em unsafe set} $\Safeset^c$. 
	Denote the {\em optimal safe solution}  as
	$S^\star:=\argmax\{\mu_S:S\in\mathcal{S}\}$ with return $\mu^\star$. For simplicity, we assume that $S^\star$ is unique.  
	Denote the {\em suboptimal set} $\Subopt:=\{S\in\mathcal{A}_K:\mu_S<\mu^\star \}$ and the
	{\em risky set} $\Riskset:=\{S\in\mathcal{A}_K:\mu_S\geq\mu^\star ,S\neq S^\star\}$.
	For a solution $S$, let the mean gap $\Delta_S:=\mu^\star -\mu_S$ and the variance gap $\Deltav_S:=|\sigma_{S}^2-\bsigma^2|$.

An instance $\instance$, time horizon $T \in \mathbb{N}$ and confidence parameter $\delta\in(0,1)$ are specified. An agent, who knows  $E,\mathcal{A}_K$ and $\bsigma^2$ but not the vector of probability distributions $\nu$, interacts adaptively with the instance over $T $ time steps as follows.  At time step  $t\in[T]$, the agent uses a stochastic function $ \pi_t $ that selects a solution $ S_t\in \mathcal{A}_K $ based on the  observation history 
    $\mathcal{H}_{t-1}:=\left((S_s,\{W_i(s)\}_{i\in S_s})\right)_{s\in[t-1]}$. In other words, $S_t = \pi_t (\mathcal{H}_{t-1})$ is a  stochastic  function of the history $\mathcal{H}_{t-1}$. 
 	The agent  receives the random return $\sum_{i\in S_t}W_i(t)$, where $\{W(s) = \{W_i(s)\}_{i\in E}\}_{s\in[T]}$ are i.i.d.\ according to $\nu$  across time. The weights of the selected items  $\{W_i(t):i\in S_t\}$  are observed by the agent at each time $t\in [T]$. The collection of stochastic functions $\pi=\{\pi_t\}_{t\in [T]}$ is known as the agent's {\em policy}.

  {\setlength{\abovedisplayskip}{4.5pt}
\setlength{\belowdisplayskip}{4.5pt}
	The goal of the agent is to 
	minimize {\em the expected cumulative regret} (or simply {\em regret}) $\Reg (T)$ over the horizon $T$, subject to a certain risk constraint. More precisely, the  {\em regret} suffered by a policy $\pi$ employed by the agent   is defined as 
	\begin{align}
		\Reg^\pi(T):=\mathbb{E}_\pi\left[\sum_{t=1}^T \left(\sum_{i\in S^\star}W_i(t)-\sum_{i\in S_t}W_i(t)\right)\right] \label{eqn:regret}
  \end{align}
  The policy $\pi$ should satisfy the condition that all the solutions chosen $\{S_t^\pi\}_{t\in [T]} \subset \mathcal{A}_K$ are safe with probability at least $1-\delta$, i.e., 
  \begin{align}
		  \mathbb{P}_\pi\big[ \forall\, t\in [T], S_t^\pi\in\Safeset \big]\geq 1-\delta.\label{constraint}
	\end{align}
 This is referred to as the {\em probably anytime-safe} constraint. 

 
 In the problem-dependent lower bounds, we will refer to a certain class of ``good'' policies that operate as the time horizon $T\to\infty$ and the probability of being safe in the sense of~\eqref{constraint} tends to $1$. This is formalized in the following. 
    \begin{defin} 
    Fix an instance $\nu$ and a   vanishing sequence  $\{\delta_T\}_{T=1}^\infty\subset(0,1)$. A policy $\pi =\{\pi_t\}_{t=1}^\infty$ is said to be a  {\em $\{\delta_T\}_{T=1}^\infty$-variance-constrained consistent algorithm} if 
    \begin{itemize}
        \item $\Reg^\pi(T)=o(T^a)$ for all $a>0$ and 
        \item $\mathbb{P}_\pi\vphantom{\bigg[}\big[\forall\, t\in[T] , S_t^\pi\in\Safeset\big]\geq 1-\delta_T$ for all $T\in\mathbb{N}$.
    \end{itemize}
    \end{defin}
    \vspace*{-1.5em}
	We often omit the superscripts $\pi$ in $\Reg^\pi, S_t^\pi$ (or $A_t^\pi$ and $A_{t,r}^\pi$ in {\sc PASCombUCB}) and the subscripts $\pi$ in the probabilities and expectations if there is no risk of confusion.
}

\section{Our Algorithm: {\sc PASCombUCB}}
	Our algorithm \underline{P}robably \underline{A}nytime-\underline{S}afe \underline{Comb}inatorial \underline{UCB}  (or {\sc PASCombUCB}), presented in Algorithm~\ref{SafeCombUCB}, is designed to satisfy the probably anytime-safe constraint. In particular, we apply (and analyze) the {\sc Greedy-Split} subroutine in Line $11$; this subroutine has not been involved in an algorithm designed for standard combinatorial semi-bandits such as {\sc CombUCB1} \citep{Chen2013Combinatorial}.
	\begin{algorithm}[htbp]
		\caption{{\sc PASCombUCB}}
		\begin{algorithmic}[1]\label{SafeCombUCB}
			\STATE \textbf{Input}: 
			An instance $\instance$ (with unknown $\nu$), the horizon $T$ and the confidence parameter $\delta\in(0,1)$.
			\STATE Set phase counter $\phase=1$ and  time step counter $t=1$.
			\WHILE{$\exists\,  i\in E$ such that $ T_i(\phase-1)<2$}
			\STATE Pull $A_\phase\!=\!\argmax_{S:|S|\leq q} |\{i\! \in \!S: T_i(\phase-1)\!<\! 2\}|$.
			\STATE $\phase\leftarrow \phase+1$, $t\leftarrow t+1$.
			\ENDWHILE
			\STATE Update the sample mean, sample variance and confidence bounds according to \eqref{equ:Aestimates}.
			\STATE Update the empirically safe set $\mathcal{S}_\phase$ and possibly safe set $\bar{\mathcal{S}}_\phase$ according to \eqref{empsafe} and \eqref{possafe} respectively.
			\WHILE{$t <T$}
			\STATE Identify a solution $A_\phase\!=\!\argmax_{A\in\bar{S}_{\phase-1}}U^\mu_A(\phase\! -\! 1)$.
			\STATE Invoke {\sc Greedy-Split} to split the solution $A_\phase$ into 
 $n_p$ sub-solutions $\{A_{\phase,1},\ldots,A_{\phase,n_\phase}\}\subset\mathcal{S}_{\phase-1}$.
			\STATE Set $n_\phase\leftarrow\min\{n_\phase,T-t\}$.
			\STATE Choose solution $\{A_{\phase,1},\ldots,A_{\phase,n_\phase}\}$.
			\STATE Update the statistics of all solutions based on~\eqref{equ:Aestimates}.
			\STATE Update the empirical sets based on~\eqref{empsafe} and \eqref{possafe}.
			\STATE Set $t=t+n_\phase$ and $\phase=\phase+1$, 
			\ENDWHILE
		\end{algorithmic}
	\end{algorithm}

	\underline{\bf Statistics.}
 Since each item $i\in E$ is $\sigma^2$-sub-Gaussian, any  solution that contains at most $q:=\lfloor \frac{\bsigma^2}{\sigma^2}\rfloor$ items  is safe with probability (w.p.) $1$. We call such a solution  {\em absolutely safe}. 
    Algorithm~\ref{SafeCombUCB} ({\sc PASCombUCB}) is conducted  in  {\em phases}, where each phase consists of multiple time steps and each item can be pulled at most once during each phase. Thus we adopt a different notation ``$A$'' to denote the solution in our algorithm.
    Define $T_i(\phase):= \sum_{s=1}^{\phase}\mathbbm{1}\{i\in A_\phase \}$ as the number of times item $ i $ is pulled up to and including phase~$\phase$.
	Denote the sample mean and   sample variance of item $i $ at phase $\phase$ respectively as 
	\begin{align}
		\hat{\mu}_i(\phase) &:=\frac{1}{T_i(\phase)}\sum_{s=1}^{\phase}W_i(s)\cdot\mathbbm{1}\{i \in A_s\},\quad\mbox{ and}\label{sm}\\
		\hat{\sigma}^2_i(\phase)& :=\frac{1}{T_i(\phase)}\sum_{s=1}^{\phase}\left(W_i(s)-\hat{\mu}_i(\phase) \right)^2\cdot\mathbbm{1}\{i\in A_s\}.
		\label{sv}
	\end{align} 
	The bound based on the Law of Iterated Logarithms (LIL) is used to construct the confidence radii. For a fixed $\epsilon\in (0,1)$,  define
	$\mathrm{lil}(t,\rho):=\left(1+\sqrt{\epsilon}\right)\Big(\frac{1+\epsilon}{2t}\ln\big(\frac{\ln((1+\epsilon)t)}{\rho}\big) \Big)^{1/2}$
	and denote the confidence radius for the mean as 
	\begin{align}
		\alpha(t) :=\mathrm{lil}(t,\omega_\mu), \label{confidencebounds_mean}
    \end{align}
  where $\omega_\mu$ is a parameter to be chosen. 
  The confidence radii for the variance are asymmetric about the empirical variance and are parameterized by $\omega_{\rmv}$ and $\omega_{\rmv}'$ that may not necessarily be the same. They are defined as 
  \begin{align}
  \beta_\rmu(t):=3\cdot\mathrm{lil}(t,\omega_{\rmv})\quad\mbox{and}\quad 
		\beta_\rml(t):=3\cdot\mathrm{lil}(t,\omega_{\rmv}^\prime).\label{confidencebounds_var}
	\end{align}
 We denote the {\em upper} and {\em lower confidence bounds} (UCB and LCB) for the mean  of item $i$ as 
	\begin{align}
		&U_i^\mu(\phase):= \hat{\mu}_i(\phase)+\alpha(T_i(\phase))\quad\mbox{and}\\
		&L_i^\mu(\phase):= \hat{\mu}_i(\phase)-\alpha(T_i(\phase))\label{cb}
	\end{align} 
	respectively. The   UCB and LCB for the variance of item $i$  are defined as 
	\begin{align}
		&U_i^{\rmv}(\phase) := \min\{\hat{\sigma}^2_i(\phase)+\beta_\rmu(T_i(\phase)),\sigma^2\}\quad\mbox{and}\\
		&L_i^{\rmv}(\phase) := \max\{\hat{\sigma}^2_i(\phase)-\beta_\rml(T_i(\phase)),0\} \label{cbv}
	\end{align}
 respectively. 
	With the sample mean, sample variance, and confidence bounds for the items, we define the following statistics for all solution $S\in \mathcal{A}_K$:
	\begin{alignat}{3} 
		\hat{\mu}_S(\phase)&=\sum_{i\in S}\hat{\mu}_i(\phase),\quad &&\hat{\sigma}^2_S(\phase)=\sum_{i\in S}\hat{\sigma}^2_i(\phase),
		\\
		U^\mu_S(\phase)&=\sum_{i\in S}U^\mu_i(\phase),\quad && L^\mu_S(\phase) =\sum_{i\in S}L^\mu_i(\phase) ,\label{equ:Aestimates} \\
		U^\rmv_S(\phase)&=\sum_{i\in S}U^\rmv_i(\phase),\quad &&L^\rmv_S(\phase) =\sum_{i\in s}L^\rmv_i(\phase).
	\end{alignat}
	Denote the {\em empirically safe set} as 
	\begin{align}\label{empsafe}
		\mathcal{S}_\phase:=\{S\in\mathcal{A}_K:U^\rmv_S(\phase)<\bsigma^2\}
	\end{align}
	 and the {\em possibly safe set} as 
	 \begin{align}\label{possafe}
		\bar{\mathcal{S}}_\phase:=\{S\in\mathcal{A}_K:L^\rmv_S(\phase)<\bsigma^2\}.
	\end{align}
	The solutions in $\mathcal{S}_t$ and $\bar{\mathcal{S}}_t$  are called {\em empirically safe}  and   {\em possibly safe} solutions respectively.  
 
 \underline{\bf Dynamics.}
 In the {\em initialization stage} (lines $3$ to  $6$), {\sc PASCombUCB} greedily pulls the absolutely safe solutions. When each item has been pulled at least twice, this stage is terminated. 
	After initialization, during phase $\phase$, {\sc PASCombUCB} {\bf firstly} identifies a solution $A_\phase=\argmax_{A\in\bar{S}_\phase}U^\mu_A(\phase-1)$ via an {\em optimization oracle} (Line~$10$). 
	It {\bf then} calls a subroutine {\sc Greedy-Split} to greedily partition the solution $A_\phase$ into empirically safe sub-solutions (Line $11$, see Figure~\ref{Split_variance} for illustration).
	{\bf Subsequently}, these solutions are chosen and the stochastic rewards from the corresponding items are observed (Line~$13$). 
	{\bf Lastly}, the empirical estimates, the confidence bounds, and the empirical sets are updated (Lines $14$ and $15$).
 	\begin{algorithm}[ht]
		\caption{{\sc Greedy-Split}}
		\begin{algorithmic}[1]\label{alg:split}
			\STATE \textbf{Input}: 
			A solution $A_\phase$ and the upper confidence bound on the variance $U^\rmv(\phase-1)$ at phase $p-1$.
			\STATE Set $n_\phase=1,s=1$ and $A_{\phase,1}=\emptyset$.
            \STATE Index the items in $A_\phase$ by $i_1,\ldots,i_{|A_\phase|}$.
			\WHILE{$s\leq |A_\phase|$}
			\IF{$U_{A_{\phase,n_\phase}}^\rmv(\phase-1)+U_{i_s}^\rmv(\phase-1)\leq\bsigma^2$}
			\STATE Set $A_{\phase,n_\phase}\leftarrow A_{\phase,n_\phase}\cup\{i_s\}$.
			\ELSE
			\STATE  $n_\phase\leftarrow n_\phase+1$ and $A_{\phase,n_\phase}=   \{i_s\}$.
			\ENDIF
			\STATE $s \leftarrow s+1$.
			\ENDWHILE
			\STATE \textbf{return} $\{A_{\phase,1},\ldots,A_{\phase,n_\phase}\}$.
		\end{algorithmic}
	\end{algorithm}
    \begin{figure}[t]
        \centering
        \includegraphics[width=0.4\textwidth]{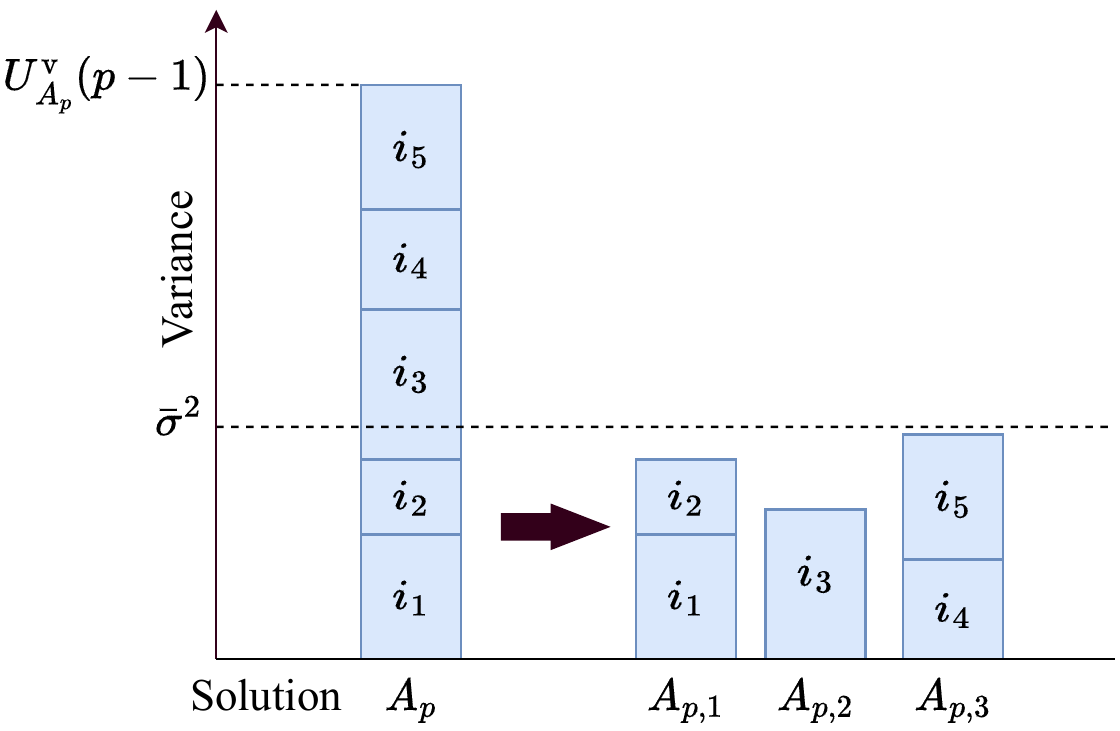}
        \caption{A diagram of a split to a solution $A_\phase$ containing $5$ items.}
        \label{Split_variance}
        \vspace{-1em}
    \end{figure}
         \begin{figure}[t]
		\centering
		\includegraphics[width=0.4\textwidth]{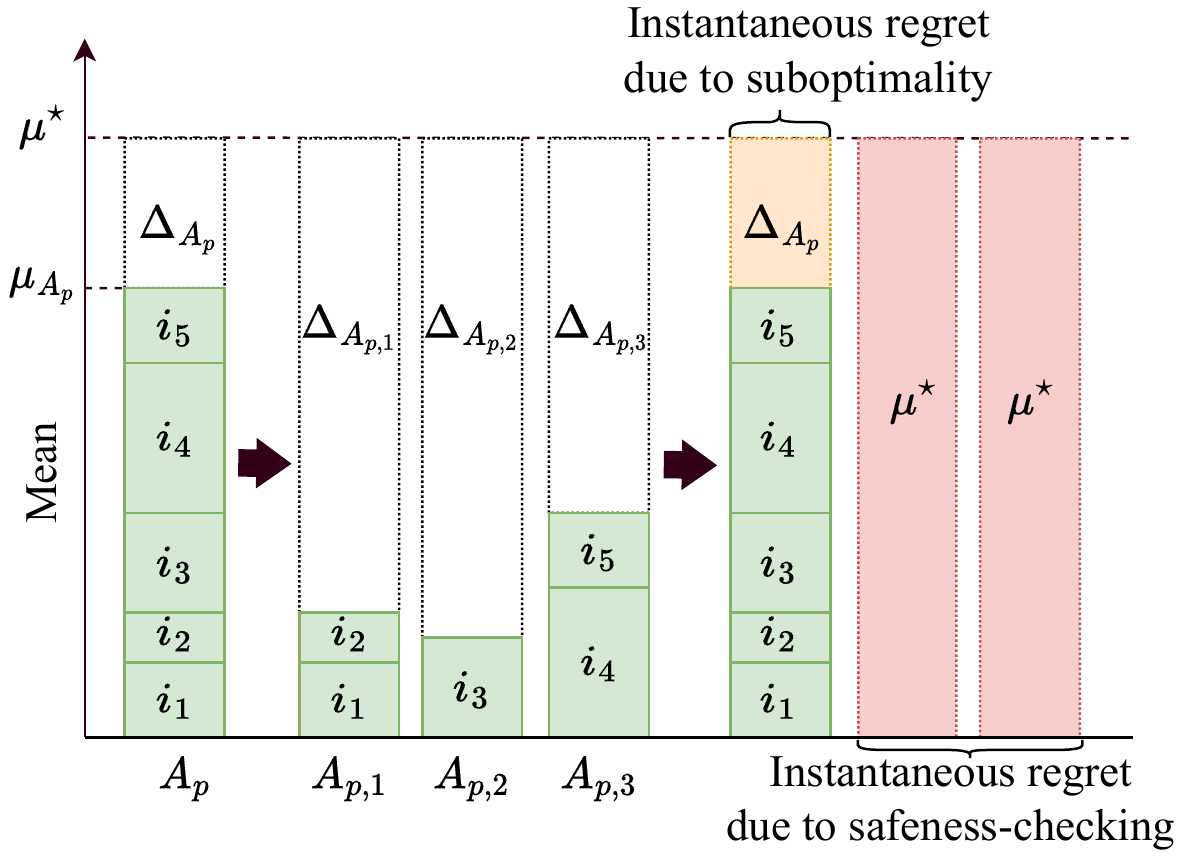}
		\caption{Solution $A_\phase$ is split into $n_\phase=3$ sub-solutions, the instantaneous regret at phase $p$ can be divided into the instantaneous regret due to suboptimality and the instantaneous regret due to safeness-checking.}
		\label{Split_mean}
        \vspace{-1em}
	\end{figure}
    \begin{figure}[t]
		\centering
		\includegraphics[width=0.4\textwidth]{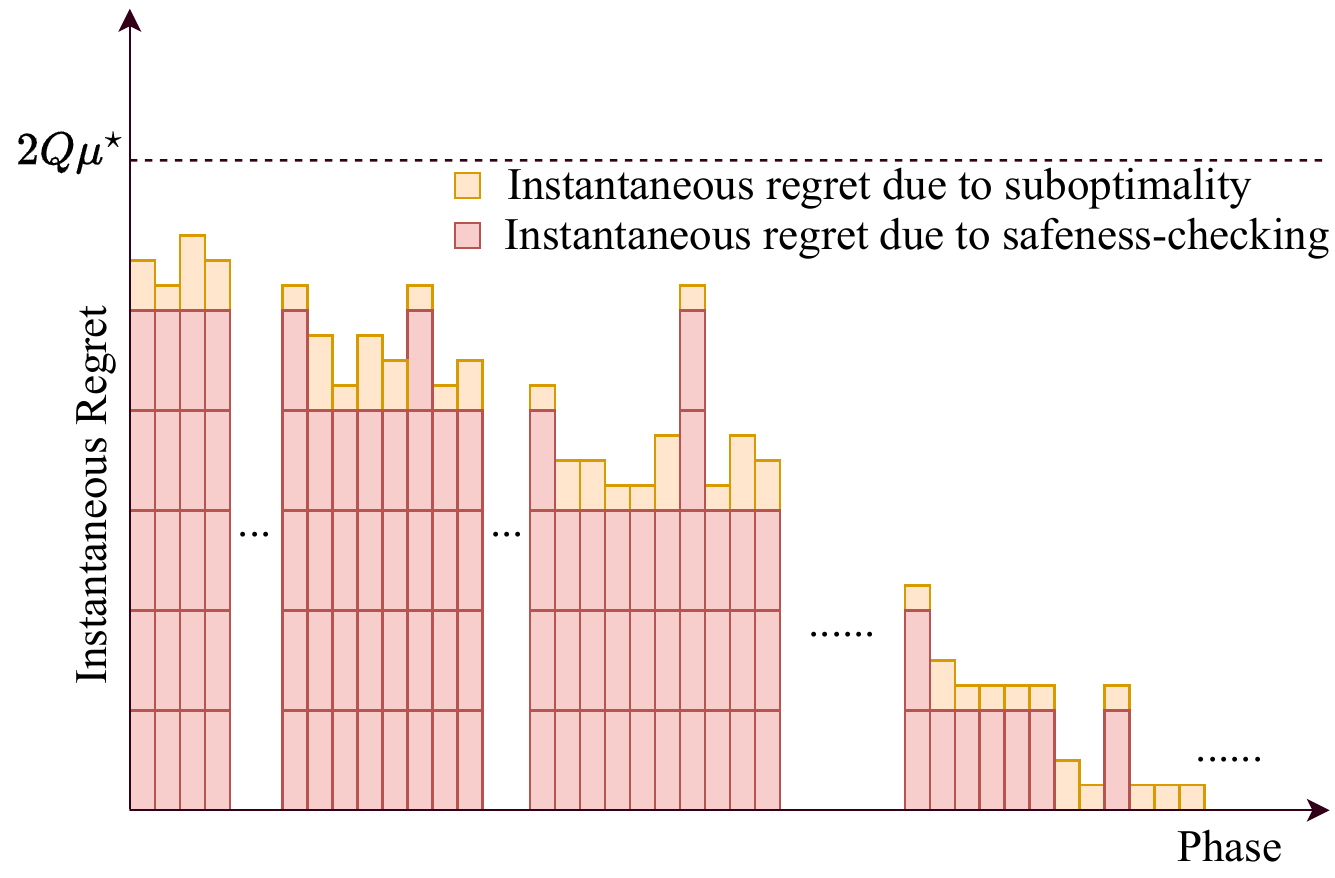}
		\caption{An illustration of the instantaneous regret yielded by {\sc PASCombUCB}. As the variances of the items are more determined, less regret due to safeness-checking is generated.}
		\label{Instant_Regret}
        \vspace{-1em}
	\end{figure}

    \underline{\bf Illustration.}
    Figures~\ref{Split_mean} and~\ref{Instant_Regret} illustrate the regret accumulated during phase $p$ and over the whole $T$ horizon respectively.
    As shown in Figure~\ref{Split_mean}, the regret accumulated during phase $p$
    can be decomposed into two parts
    \begin{align}
        \sum_{r=1}^{n_\phase}(\mu^\star -\mu_{A_{\phase,r}})
           =\Delta_{A_\phase}+\mu^\star (n_\phase-1)
    \end{align}
    where $\Delta_{A_\phase}$ is {\em the (phase-wise) instantaneous regret due to suboptimality} and $\mu^\star (n_\phase-1)$ is {\em the instantaneous regret due to safeness-checking};  the latter term results from the safeness constraint.
    At the beginning, since the upper confidence bounds of the variances of all solutions are large, each solution will be split into up to $2Q$ sub-solutions, where $Q:=\lceil\frac{K}{q}\rceil$, and hence the regret due to safeness checking can be large. 
    As the algorithm progresses, we obtain more observations of items and get more confident about their variances ($U_i^\rmv(\phase)$ decreases). Hence, during some later phase, it suffices to split some solutions into fewer sub-solutions and the regret due to safeness-checking reduces. Furthermore, when most items are sampled sufficiently many times, the unsafe solutions are excluded from the possibly safe set $\bar{\mathcal{S}}_p$, and the only contribution to the regret is via the suboptimality of the solution $A_\phase$.

    \begin{remark}
            The two parameters $\omega_\rmv$ and $\omega_\rmv'$ determine the confidence radii of variances  and do not necessarily have to be the same.
            The confidence parameter $\omega_\rmv^\prime$ is solely a parameter of {\sc PASCombUCB}; its choice does not rely on the confidence parameter $\delta$ and only affects $L_S^\rmv(p)$, the lower confidence bound of the variance, which determines when we ascertain a solution 
            to be unsafe. 
            The choice of $\omega_\rmv$ depends on $\delta$ and it influences $U_S^\rmv(p)$, the upper confidence bound of the variance, which guides {\sc PASCombUCB} to split the solution  to satisfy the probably anytime-safe constraint.  
    \end{remark}


	\section{Problem-dependent Bounds}
    For simplicity,
    when a time horizon $T$ and a confidence parameter $\delta=\delta_T$ are given, 
    we set the confidence parameters $\omega_\mu=\omega_\rmv^\prime=\frac{1}{T^2}$ and $\omega_\rmv=\frac{\delta_T}{T^2}$.
    
    We introduce various suboptimality gaps that 
    contribute to the regret due to the suboptimality.
    \begin{itemize}
        \item for $i\in E\setminus S^\star$, let  the {\em minimum safe-suboptimal gap} be 
        \begin{align}\label{equ:safemeangap}
            \Delta_{i,\Safeset\cap\Subopt,\min}:=\min_{S\ni i,S\in\Safeset\cap\Subopt}\Delta_S;
        \end{align}
        \item for $i\in E$, let   the {\em  minimum unsafe-suboptimal gap} be
        \begin{align}\label{equ:unsafemeangap}
            \Delta_{i,\Safeset^c\cap\Subopt,\min}:=\min_{S\ni i,\  S\in\Safeset^c\cap\Subopt}\Delta_S;
        \end{align}
        and  let  the {\em tension parameter between the mean gap $\Delta_S$ and variance gap} $\Delta_S^\rmv$ be
        \begin{align}
        c_i:=\max_{S\ni i,\  S\in\Safeset^c\cap\Subopt}\left(\frac{\Delta_S}{\max\{\Delta_S,\Deltav_S/3\}}\right)^2.
        \end{align}
    \end{itemize}
    We also define following safeness gaps that induce the conservative sampling strategy to guarantee the probably anytime-safe constraint. For $i\in E$, and
    \begin{itemize}
        \item for the risky set $\Riskset$, define \emph{the minimum unsafeness gap} $\Deltav_{i,\Riskset}:=\min_{S\ni i,S\in\Riskset}\Deltav_S.$
        \item for the safe and suboptimal set $\Safeset\cap\Subopt$, let 
        \begin{align}
            \Psi_{i, \Safeset \cap \Subopt} :=\max_{S\ni i,\  S\in\Safeset\cap\Subopt}\min\left\{ \frac{\ln T}{\Delta_S^2},\frac{9\ln (T/\delta_T)}{(\Delta_S^\rmv)^2}  \right\}
        \end{align} 
        which characterizes the order of the number of times that item $i$ needs to be sampled in order to identify the suboptimality of all safe and suboptimal solutions $A\ni i$ while satisfying the safeness constraint. We further define a variant of $\Psi_{i,\Safeset \cap \Subopt}$ as
        $$
            \Psi^\prime_{i, \Safeset \cap \Subopt} :=\max_{S\ni i, S\in\Safeset\cap\Subopt}\min\left\{ \frac{\ln T}{\Delta_S^2},\frac{9\ln (1/\delta_T)}{(\Delta_S^\rmv)^2}  \right\}
        $$
        which will be used to characterize the lower bound.
        \item for the unsafe and suboptimal set $\Safeset^c\cap\Subopt$, let
        \begin{align}\label{equ:Phi}
            \Phi_{i, \Safeset^c \cap \Subopt} :=\max_{S\ni i,\  S\in\Safeset^c\cap\Subopt}\min\left\{ \frac{\ln T}{\Delta_S^2},\frac{9\ln T}{(\Delta_S^\rmv)^2}  \right\}
        \end{align} 
        which characterizes the  hardness of identifying the unsafeness of suboptimality of all unsafe and suboptimal solutions that contain item $i$.
    \end{itemize}

    Define $\xi(\omega):=\frac{2+\epsilon}{\epsilon}\big(\frac{\omega}{\ln (1+\epsilon)}\big)^{1+\epsilon}$, where $\epsilon \in (0,1)$ is fixed.
    \subsection{Problem-dependent Upper Bound}
    \begin{restatable}[Problem-dependent upper bound]{thm}{thmUpperBdProbDep}\label{thm:upbd}
	Let   $\instance=(E,\mathcal{A}_K,\nu,\bsigma^2)$ be an instance and let $\{\delta_T\}_{T=1}^\infty \in o(1)$ be a sequence that satisfies $\ln(1/\delta_T)=o(T^b)$ for all $b>0$ (i.e., $\{\delta_T\}$ is not exponentially decaying).
        Then, {\sc PASCombUCB} is a $\{\delta_T\}_{T=1}^\infty$-variance-constrained consistent algorithm. 
        More precisely, given a time budget $T$, the probably anytime-safe constraint is satisfied and
        the regret of {\sc PASCombUCB} $\Reg(T)$ is upper bounded by
		\begin{align}
            &\min\left\{T\mu^\star ,\Regsub(T)+\Regsafe(T)\right\}+\Regfail(T),
		\end{align}
    where
    \begin{align}
        \!\Regsub(T)&=O\bigg(\sum_{i\in E\setminus S^\star}\!
            \frac{K\ln T}{\Delta_{i,\Safeset\cap\Subopt,\min}} 
            \!+\!\sum_{i\in E}\!
            \frac{c_i K \ln T}{\Delta_{i,\Safeset^c\cap\Subopt,\min}}\bigg)\\
        \!\Regsafe(T)&=
        2\mu^\star 
        H\left(\Delta(\instance)\right),\quad
        \Regfail(T)=2\mu^\star  (L+1)
    \end{align}
    where $\Delta(\instance)=\{\Deltav_{S^\star}\}\cup\{\Deltav_{i,\Riskset},\Psi_{i,\Safeset\cap\Subopt},\Phi_{i,\Safeset^c\cap\Subopt}\}_{i\in E}$ and 
    $H\left(\Delta(\instance)\right):=H(1,\instance)$ is defined in \eqref{equ:Hinstance} in App.~\ref{sec:safe_reg}.
 \end{restatable} 
    \begin{remark}\label{rem:H}
            If the gaps in $\Delta(\instance)$ are sufficiently small and $\delta_T=T^{-\lambda}$ for a fixed $\lambda>0$, 
                \begin{align}\label{equ:HDelta}
                    &H\left(\Delta(\instance)\right)
                    =O\bigg(
                    \frac{(\lambda+1)K^2\ln T}{(\Deltav_{S^\star})^2}
                    +
                    K\sum_{i\in E}\Big(\frac{\ln T}{(\Deltav_{i,\Riskset})^2}
                    \\
                    &+
                    \max_{S\ni i,\atop S\in\Safeset\cap\Subopt}\min\left\{ \frac{\ln T}{\Delta_S^2},\frac{(\lambda+1)\ln T}{(\Delta_S^\rmv)^2}\right\}
                    +\Phi_{i, \Safeset^c \cap \Subopt}\Big)
                    \bigg).
                \end{align}
            See~\eqref{equ:Hinstance} for more details of this calculation.
    \end{remark}
    
    The first term $T\mu^\star $  in the regret bound provides a na\"ive upper bound
    for the expected regret conditional on the variance constraint holds. 
    The order of the regret ($o(T^a)$ for all $a >0$) implies the regret will be asymptotically  bounded by the second term when the time budget $T$ is sufficiently large.
    The second term is comprised of two parts---{\em the regret due to suboptimality} $\Regsub(T)$ and {\em the regret due to safeness-checking} $\Regsafe(T)$.
    The intuition for the regret due to suboptimality $\Regsub(T)$ is that
    \begin{itemize}
    \item Each item in any safe and suboptimal solution will be sampled $O(\frac{K\ln T}{\Delta_{i,\Safeset\cap\Subopt,\min}^2})$ times to ascertain the suboptimality of all safe and suboptimal solutions to which this item belongs to.
    \item Each item in an unsafe and suboptimal solution $S$ will be sampled 
    $O\left(\frac{K\ln T}{\max\{\Delta_{S},\Deltav_{S}/3\}^2 }\right)$
    times  to ascertain either the suboptimality or the unsafeness of  $S$.
    As this should be done for all the unsafe and suboptimal solutions, we need to take the maximum of the above time complexity.
    More precisely, when $c_i=1$, suboptimality identification of the unsafe and suboptimal solutions to which item $i$ belongs dominates the regret; and when $c_i<1$, the ascertaining of the unsafeness dominates the regret.
    \end{itemize}
    The intuition for the regret due to safeness checking $\Regsafe(T)$ is that $H\left(\Delta(\instance)\right)$
    provides an upper bound for the number of time steps needed for guaranteeing the safeness of all solutions. {\sc PASCombUCB} achieves this in a judicious manner  since it does not check the safeness of all the solutions at the start, followed by exploration and exploitation of the possibly high-return safe solutions. Instead, it takes advantage of the fact that 
    when a (safe or unsafe) suboptimal solution is ascertained to be   suboptimal, its safeness can be disregarded,
    as reflected in  the terms 
    $\Psi_{i,\Safeset\cap\Subopt}$ and $\Psi_{i,\Safeset^c\cap\Subopt}$.
    In addition, it will not sample an unsafe solution if it is identified as unsafe w.p.\ at least $1-2\xi(\omega_\rmv^\prime)$.
    The last term  $\Regfail(T)$ corresponds to the regret due to failure of the ``good'' event and at the initialization stage. A proof sketch is presented in Section~\ref{sec:proofSketchProDep}.

    


    \subsection{Problem-dependent Lower Bound}

\begin{restatable}[Problem-dependent lower bound]{thm}{thmLowBdProbDep}
\label{thm:low_bd_prob_dep}
Let $\{\delta_T\}_{T=1}^\infty \in o(1)$ be a sequence that satisfies $\ln(1/\delta_T)=o(T^b)$ for all $b>0$. There exists an instance $\instance$ such that for any $\{\delta_T\}_{T\in\mathbb{N}} $-variance-constrained consistent algorithm $\pi$, the regret is lower bounded by 
\begin{align}
    \Omega&\bigg(\sum_{i\in E}\frac{\ln T}{\Delta_{i,\Safeset\cap\Subopt,\min}}
    \bigg)+
    \frac{\mu^\star }{K}
    \cdot\Omega\bigg(\frac{K\ \ln(1/\delta_T)}{(\Deltav_{S^\star})^2}
    \\
    & +\sum_{i\in E}\Big(\Psi^\prime_{i,\Safeset\cap\Subopt} 
    +\frac{\ln T}{(\Deltav_{i,\Riskset})^2}
    +\Phi_{i,\Safeset^c\cap\Subopt} \Big)
        \bigg).
\end{align}
\end{restatable}
The proof is presented at App.~\ref{sec:proofUppBd}.
With Theorem~\ref{thm:low_bd_prob_dep}, the problem-dependent upper bound is tight for polynomially decaying~$\{\delta_T\}_{T\in\mathbb{N}}$.
\begin{restatable}[Tightness of problem-dependent bounds]{cor}{corTightProbdep}\label{cor:tightnessDep}
Let $\delta_T=T^{-\lambda}$ with a fixed $\lambda>0$, the regret  
 \begin{align}
        \Reg(T)\in\;&\Omega\bigg(\sum_{i\in E}\frac{\ln T}{\Delta_{i,\Safeset\cap\Subopt,\min}}
        +\frac{\mu^\star}{K^2} H\left(\Delta(\instance)\right)
        \bigg)\\
        \cap&\,
        O\bigg(
        \sum_{i\in E}\frac{K\ln T}{\Delta_{i,\Safeset\cap\Subopt,\min}}
        +\mu^\star  H\left(\Delta(\instance)\right)
        \bigg)
    \end{align}
    where $H\left(\Delta(\instance)\right)$ is defined in Remark~\ref{rem:H}.
The upper bound above is achieved by {\sc PASCombUCB}.
\end{restatable}
\vspace{-.1in}
Under  different rates of decay of $\{\delta_T\}_{T\in\mathbb{N}}$ (see App.~\ref{sec:tightness} for the cases where $\ln(1/\delta_T)=\omega(\ln T)$ and $o(\ln T)$), the upper bound of the regret due to suboptimality $\Reg_1(T)$ (the first term in the total regret) and the upper bound of the regret due to safeness-checking $\Reg_2(T)$ (the latter term) match their corresponding lower bounds up to factors of $K$ and $K^2$ respectively; this gap is  acceptable as $K$ (e.g., number of ads displayed) is usually small relative to $L$ (total number of ads). 
More discussions are postponed to App.~\ref{sec:additional}.
One may  naturally wonder whether we can tolerate a much more stringent   probably anytime-safe constraint. 
The following theorem (with $b=1$) indicates no algorithm is $\{\delta_T\}_{T\in\mathbb{N}} $-variance-constrained consistent if $\delta_T$ decays {\em exponentially fast} in $T$. Detailed proofs are postponed to~App.~\ref{sec:lower_bd}.
 \begin{restatable}[Impossibility result]{thm}{lemReqOnDeltaT}
\label{thm:impos}
    Let $\{\delta_T\}_{T=1}^\infty \in o(1)$ be a sequence that satisfies that there exists $b\in(0,1]$ such that $\ln(1/\delta_T)=\Omega(T^b)$. For any instance $\instance$,
    the regret of any algorithm is lower bounded by
    $\Omega(T^b)$.
\end{restatable}

    \section{Problem-independent Bounds}
We can derive a problem-independent upper bound on the regret of {\sc PASCombUCB} from the problem-dependent one in Theorem~\ref{thm:upbd} with some delicate calculations (see App.~\ref{sec:proofUppBd}).
\vspace{-1em}
    \begin{restatable}[Problem-independent upper bound]{thm}{thmproblemindependentupbd}
\label{thm:problemindependent_upbd}
        Let $\{\delta_T\}_{T=1}^\infty \in o(1)$ be a sequence that satisfies $\ln(1/\delta_T)=o(T^b)$ for all $b>0$.
        If $T>L$, for any instance $\instance$ with variance gaps lower bounded by $\Delta^\rmv \leq\min_{S\in \calA_K} \Deltav_S$, the regret of {\sc PASCombUCB} is upper bounded by
        \begin{align}
             O\bigg( \sqrt{KLT\ln T} +
        \frac{ LK^2}{(\Deltav)^2}\ln \Big(\frac{1}{\delta_T}\Big)\bigg).
        \end{align}
    \end{restatable} \vspace{-.1in}


\begin{restatable}[Problem-independent lower bound]{thm}{thmLowBdProbIndep}
\label{thm:low_bd_prob_indep}
    Let the minimum variance gap be $\Delta^\rmv : = \min_{S\in \calA_K} \Delta_S^\rmv$.
When $K^3 \ge L^2$, we have
\begin{align*}
    \Reg (T)
    =  \Omega\bigg(   \sqrt{KLT} + \min \Big\{  \frac{L}{(\Delta^\rmv)^2}  \ln\Big(\frac{1}{\delta_T }\Big) , T \Big\} \bigg).
\end{align*}
\end{restatable} \vspace{-.1in}

\begin{remark}
    The assumption that the variance gaps of all solutions are lower bounded by $\Deltav$ is needed to achieve a non-vacuous problem-independent bound. Given any algorithm and time budget $T$, the variance gap of $S^\star$ can be arbitrarily small if $\Deltav$ is  not  bounded away from zero, so the $\min$ in  Theorem~\ref{thm:low_bd_prob_indep} will be dominated by the linear term $T$, and hence, no algorithm can attain sublinear regret. 
\end{remark}


\begin{restatable}[Tightness of problem-independent bounds]{cor}{corTightProbIndep}\label{cor:tightnessIndep}
    Let $K^3\le L^2$, and $\{\delta_T\}_{T=1}^\infty \in o(1)$ satisfies $\ln(1/\delta_T)=o(T^b)$ for all $b>0$. We have
\begin{align*}
    \Reg (T)
    \in  
    &
    ~\Omega\bigg(   \sqrt{KLT} + \frac{L  }{ (\Delta^\rmv )^2 } \  \ln\Big(\frac{1}{\delta_T}\Big)\bigg)
    \\&
    \cap 
    O\bigg( \sqrt{KLT\ln T} +
        \frac{ LK^2}{(\Deltav)^2}\ln\Big(\frac{1}{\delta_T}\Big)\bigg).
\end{align*}
The upper bound is achieved by  {\sc PASCombUCB}.
\end{restatable}\vspace{-.1in}

We observe that the gap between the upper and lower bounds is manifested on $\sqrt{\ln T}$ and $K^2$. The presence of  $\sqrt{\ln T}$ is not unexpected as it is also involved in the gap between the bounds on the regret for the (unconstrained) combinatorial bandits \citep{Tight2015Kveton}.  
Besides, the term $K^2$ is induced by the design of  {\sc PASCombUCB}. Additional discussions are provided in App.~\ref{sec:additional}.

    \section{Proof Sketch of the Problem-Dependent Upper Bound (Theorem~\ref{thm:upbd})}\label{sec:proofSketchProDep}
     {\setlength{\abovedisplayskip}{4pt}
 \setlength{\belowdisplayskip}{4pt}
    Assume that \AlgSCombUCB has processed $T^\prime$ phases with $T$ time steps, we have $\mathbb{P}[T^\prime\leq T]=1$ since each phase is composed by multiple time steps.
	Denote the expected regret of \AlgSCombUCB with $\phase$ phases as $\mathbb{E}[\rmR(\phase)]$.
	The expected regret of \AlgSCombUCB  after $T$ time steps is
	\begin{align}\label{expreg}
		\mathbb{E}[\rmR(T^\prime)]:=\mathbb{E}\bigg[\sum_{\phase=1}^{T^\prime}\sum_{r=1}^{n_\phase}(\mu^\star -\mu_{A_{\phase,r}})\bigg].
	\end{align}%
 }
    

	In the proof of Theorem~\ref{thm:upbd}, we first  show a regret decomposition lemma (Lemma~\ref{lem:upbd_decomposition}) that separates the total regret into {\em the regret due to suboptimality} $\mathbb{E}[\rmR_1(T^\prime)]$, {\em the regret due to safeness-checking} $\mathbb{E}[\rmR_2(T^\prime)]$ and the regret due to the failure of the ``good'' event and the initialization. 
    Then we upper bound $\rmR_1(T^\prime)$ and $\rmR_2(T^\prime)$ separately.
    To elucidate the dependence of the regret on the confidence parameters $\omega_\mu,\omega_\rmv$ and $\omega_\rmv^\prime$, we retain these notations henceforth. 
    Detailed proofs are presented in App.~\ref{sec:proof_upper_bd}.
    
	For $ \phase\in [T], i\in E$, define the ``good'' events that the sample mean and the sample variance are near their ground truths:
        $\mathcal{E}^\mu_{i,T_i(\phase)}:=
        \left\{\hat{\mu}_i(\phase)-\alpha(T_i(\phase))\leq \mu_i\leq\hat{\mu}_i(\phase)+\alpha(T_i(\phase))\right\}$
        and 
        $\mathcal{E}^\rmv_{i,T_i(\phase)}(\rho):=
        \{\hat{\sigma}^2_i(\phase)-3\cdot \mathrm{lil}(T_i(p),\rho)\leq \sigma_i^2 \leq
                    \hat{\sigma}^2_i(\phase)+3\cdot \mathrm{lil}(T_i(p),\rho)\}$
    and 
    \begin{align}
		\mathcal{E}_{i,T_i(\phase)}&:=\mathcal{E}^\mu_{i,T_i(\phase)}\cap \mathcal{E}^\rmv_{i,T_i(\phase)}(\omega_\rmv)\cap\mathcal{E}^\rmv_{i,T_i(\phase)}(\omega_\rmv^\prime)
        \\
        \mathcal{E} &:=\bigcap_{i\in E}\bigcap_{\phase\in[T']}\mathcal{E}_{i,T_i(\phase-1)}.
	\end{align}
    For $r\in[Q-1]$, define $\mathcal{U}_\phase(r):=\{U_{A_\phase}^\rmv(\phase-1)>r\bsigma^2\}.$
	 When event $\mathcal{U}_p(r)$ occurs at phase $\phase$, it indicates at least $r+1$ sub-solutions are needed in order to sample the items in $A_\phase$ and guarantee the safeness constraint.

    \begin{restatable}{lem}{lemUpBdDecom}\label{lem:upbd_decomposition}
        Assume that \AlgSCombUCB has processed $T^\prime$ phases with $T$ time steps, the expected regret of \AlgSCombUCB can be decomposed into three parts
        as follows
        \begin{align}
             \mathbb{E}[\rmR(\Tp)]&
             \leq
             \mathbb{E}[\rmR_1(T^\prime)|\mathcal{E}]+\mathbb{E}[\rmR_2(T^\prime)|\mathcal{E}]+\rmR_3(T)
        \\
           \text{where}\quad &\rmR_1(\Tp):=\sum_{\phase=1}^{\Tp}
            \mathbbm{1}\{A_\phase\in\Subopt\}
            \Delta_{A_\phase}
            \\
			&\rmR_2(\Tp):=\mu^\star \sum_{\phase=1}^{\Tp}
            \bigg[2\sum_{r=1}^{Q-1}\mathbbm{1}\{\mathcal{U}_p(r)\}\bigg]\\
   \rmR_3(T) & :=2\mu^\star L\big(1+T\big(\xi(\omega_\mu)+2\xi(\omega_v)+2\xi(\omega_v^\prime\big)\big)
		\end{align}
    \end{restatable}\vspace{-.1in}
    
    In  Lemma~\ref{lem:upbd_decomposition}, the first term $\rmR_1(\Tp)$ is  the  {\em (high-probability) regret due to suboptimality}, in the sense that only the mean gaps of the suboptimal solutions contribute to $\rmR_1(T)$. 
    The second term $\rmR_2(\Tp)$ is called {\em the (high-probability) regret due to safeness-checking}, since it depends on the variance gaps and  goes to $0$ if $\bsigma^2$ is sufficiently large. 
    The last  term $\rmR_3(T)$ contains the regret from the initialization stage and the regret results from the failure of the ``good'' event $\mathcal{E}$. 

    The regret due to suboptimality can be bounded in terms of the minimum safe/unsafe-suboptimal gaps as follows.
    \begin{restatable}{lem}{lemsubreg}\label{lem:subreg}
        Conditioned on event $\mathcal{E}$, the regret due to suboptimality $\rmR_1(T^\prime)$ can be bounded by
        \begin{align}
            \! O\bigg(\sum_{i\in E\setminus S^\star}\!
            \frac{K}{\Delta_{i,\Safeset\cap\Subopt,\min}} 
            \ln\frac{1}{\omega_\mu}\!+\!\sum_{i\in E}\!
            \frac{c_i K}{\Delta_{i,\Safeset^c\cap\Subopt,\min}} \ln\frac{1}{\omega_\rmv^\prime}\bigg).\!
        \end{align}  
    \end{restatable}\vspace{-.3in}
    The regret due to safeness-checking involves more critical parameters of the instance and we encode them in  $\Tp_{r^\prime}$ and $H(r^\prime,\instance)$ for $r^\prime\in[Q]$ (see Figure~\ref{Regret_safeness_H}); these terms   are
    defined formally in~\eqref{equ:Tpr} and~\eqref{equ:Hinstance} respectively.
    
    \begin{figure}[t]
        \centering
        \includegraphics[width=0.40\textwidth]{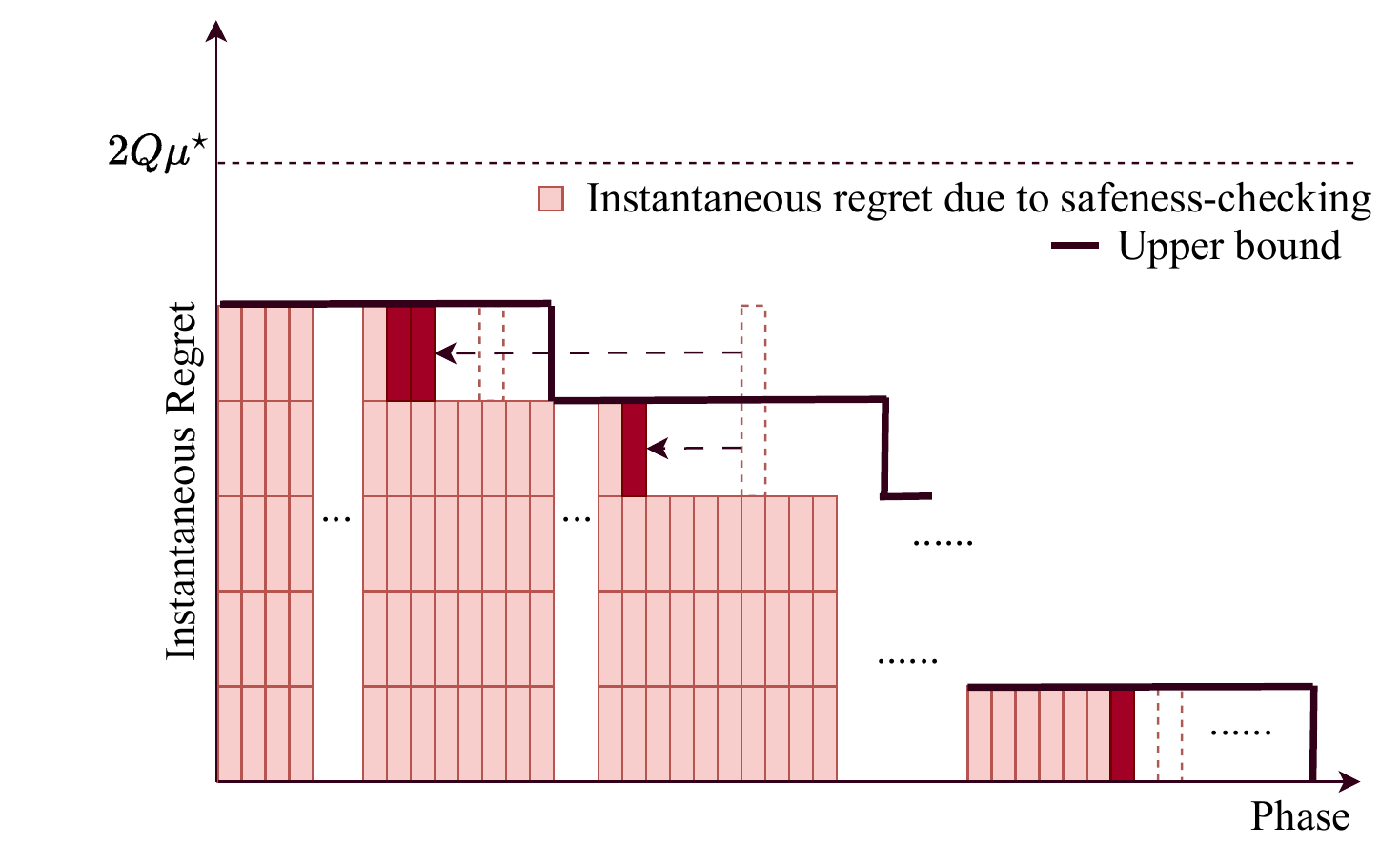}
        \vspace*{-1.2em}
        \caption{We assume the algorithm will sample those solutions with large $U_A^\rmv(p)$, i.e., those phases in which more sub-solutions are sampled are moved forward (the dark red ones). Based on this, an upper bound can be  derived (the thick black lines).}
        \label{Regret_safeness}
        \vspace*{-.5em}
    \end{figure}
    \begin{figure}[t]
        \centering
        \includegraphics[width=0.40\textwidth]{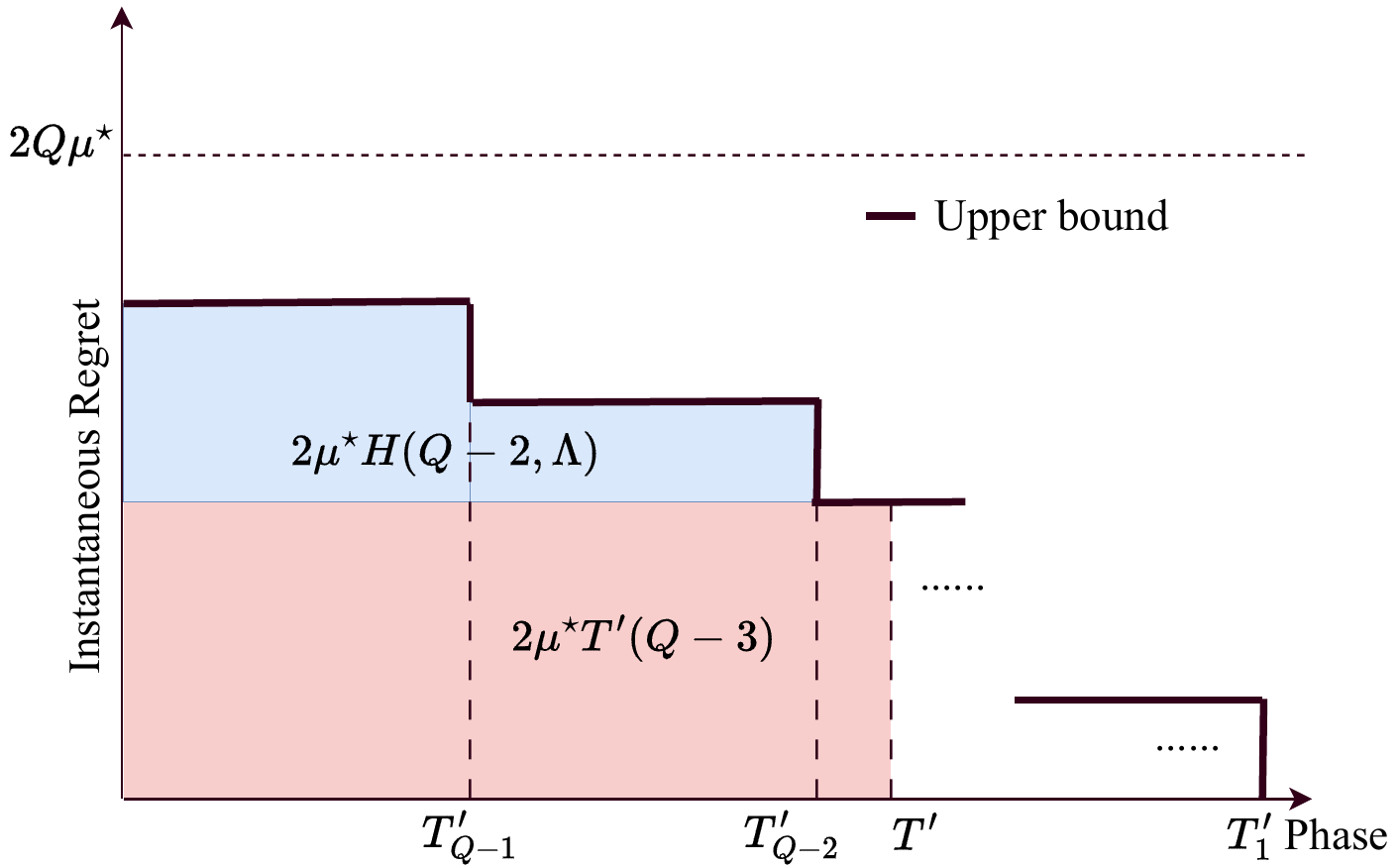}
        \vspace*{-.8em}
        \caption{An illustration of the upper bound of $\rmR_2(\Tp)$ for  phase $\Tp_{Q-2}\leq\Tp<\Tp_{Q-3}$. When $r^\prime=1$, $2\mu^\star H(1,\instance)$ is the area below the thick line, i.e., the upper bound for $\rmR_2(\Tp)$ for any~$\Tp$.}
        \label{Regret_safeness_H}
        \vspace*{-1em}
    \end{figure}
    \begin{restatable}{lem}{lemsafereg}\label{lem:safereg}
     {\setlength{\abovedisplayskip}{4pt}
\setlength{\belowdisplayskip}{4pt}
        On the event $\mathcal{E}$, if $\Tp\in[\Tp_{r^\prime},\Tp_{r^\prime-1})$ 
        then 
        \begin{align}
            \rmR_2(\Tp)\leq 2\mu^\star [\Tp (r^\prime-1)+H(r^\prime,\instance)]\leq 2\mu^\star H(1,\instance)
    \end{align}%
    }
    \end{restatable}
    To  upper bound  $\rmR_2(\Tp)$, 
    we assume the algorithm samples  solutions with large $U_A^\rmv(p)$ in $\bar{\mathcal{S}}_p$, which will then be split into several sub-solutions (see Figure~\ref{Regret_safeness}).
    Furthermore,
    for $r^\prime=Q-1,Q-2,\ldots, 1$, we derive an upper  bound for the number of phases in which event $\mathcal{U}_{\phase}(r^\prime)\cap(\mathcal{U}_{\phase}(r^\prime+1))^c$ occurs (at most $2r^\prime+1$ sub-solutions are being pulled in these phases). To be more specific (see Figure~\ref{Regret_safeness_H}), for $r^\prime=Q-1$, we compute the maximum number of phases $\Tp_{Q-1}$ in which at most $2Q-1$ sub-solutions are  sampled. Then for $r^\prime=Q-2$, we compute the maximum number of phases $\Tp_{Q-2}-\Tp_{Q-1}$ in which at most $2Q-3$ sub-solutions are  sampled. We  do this  until the time budget runs out.
    As $\Tp$ increases, $r^\prime$ decreases and $H(r^\prime,\instance)$ increases.
    When $r^\prime=1$, i.e. $\Tp\geq \Tp_1$, $H(1,\instance)$ is an upper bound for the total number of sub-solutions being pulled (up to a constant) for the safeness-checking or the price of satisfying the probably anytime-safe constraint. The upper bound for the regret due to safeness-checking is the  instance-dependent constant $2\mu^\star H(1,\instance) $ when $\Tp\geq \Tp_1$. More  discussions are postponed to Step~3 in the proof in App.~\ref{sec:safe_reg}. 


    \section{Experiments}
    \begin{figure*}[t]
		\centering
		\subfigure[Experiment 1]{
                    \includegraphics[width=0.33\textwidth]{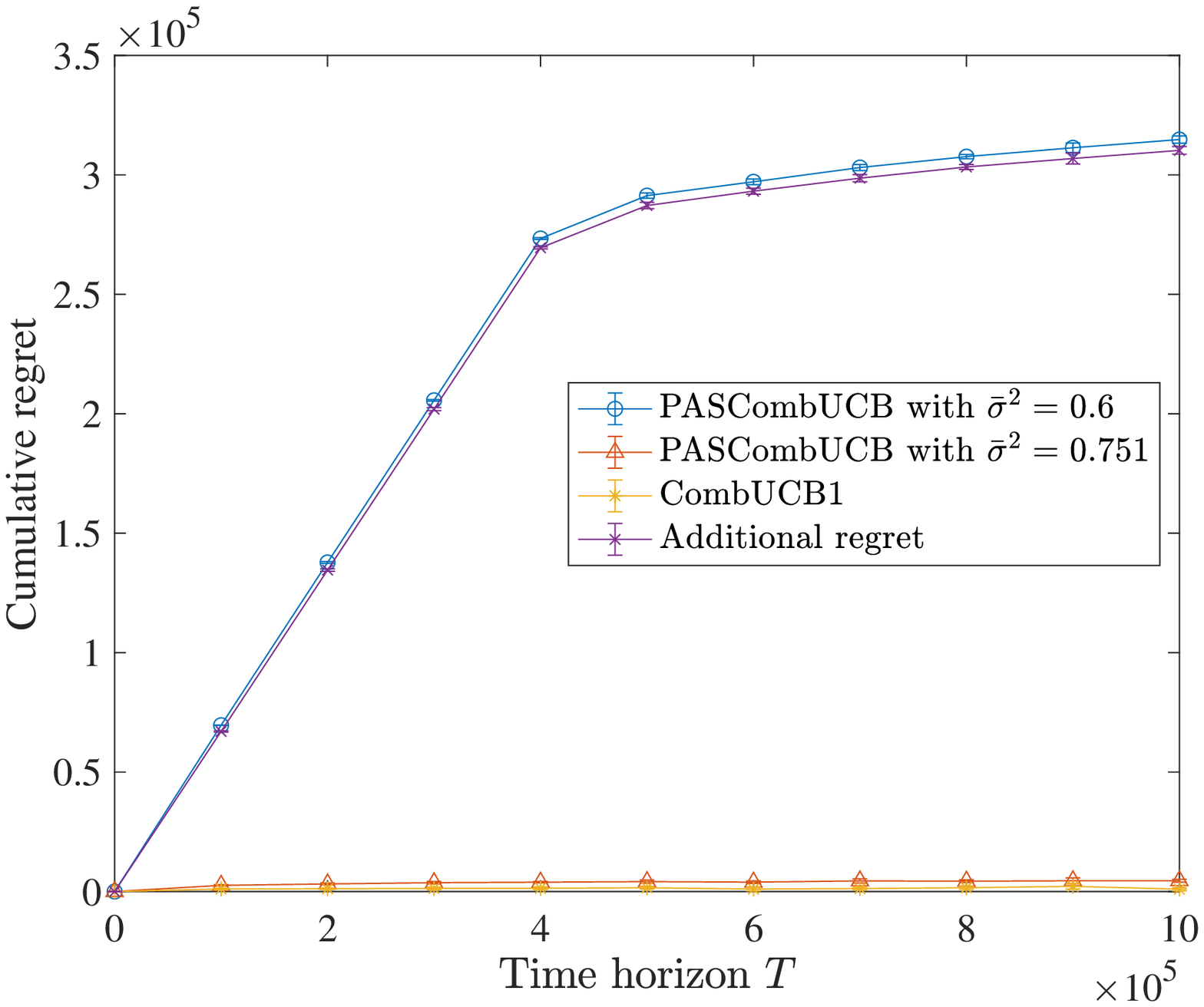}
                    \label{exp:experiment1cumreg}
            } 
            \hspace{-.1in}
		\subfigure[Experiment 2]{
            	\includegraphics[width=0.33\textwidth]{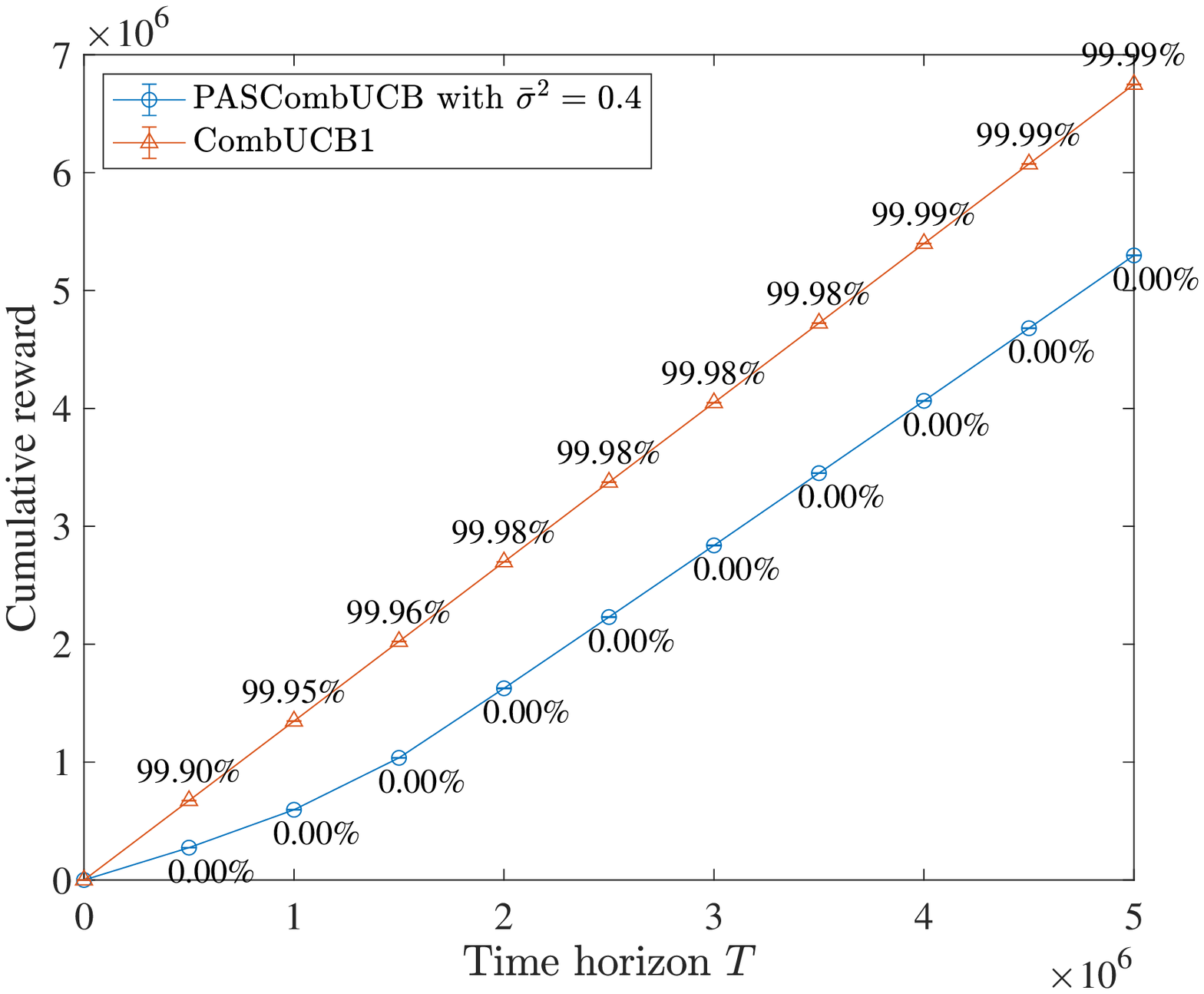}
             \label{exp:experiment2cumrew}
            }  
            \hspace{-.1in}
		\subfigure[Experiment 3]{
            	\includegraphics[width=0.33\textwidth]{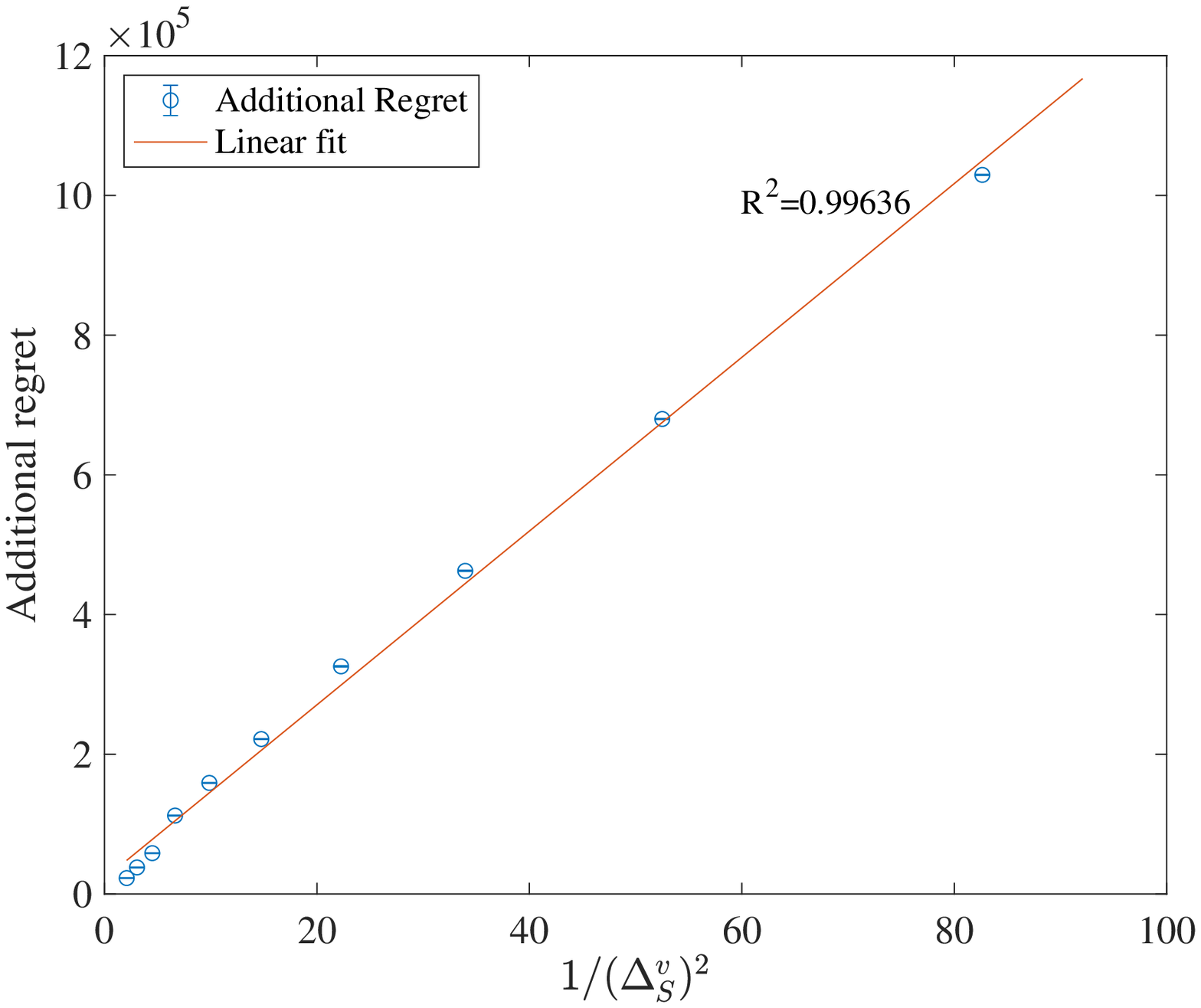}
             \label{exp:experiment3AddReg}
             }
            \vspace{-.1in}
		\caption{Results of the Experiments: (a) Experiment 1: Cumulative regret v.s.\ Time horizon; (b) Experiment 2: Cumulative reward v.s.\ Time horizon with the percentages of violations 
         besides the data points; (c) Experiment 3: Additional regret v.s.\ $1/(\Delta_S^{\rmv})^2$.}
		\label{example}
    \vspace{-.1in}
	\end{figure*} 
    \vspace{-.1in}
    In this section, we ran $3$ sets of experiments to illustrate the empirical performance of \AlgSCombUCB and to corroborate its theoretical guarantees. As {\sc CombUCB1} \citep{Tight2015Kveton,Khezeli2020Safe,Amani2019Linear} has tight regret guarantees, we adopt it as the benchmark in the unconstrained case. Codes are accessible at \url{https://github.com/Y-Hou/PASSCSB.git}.

    \underline{{\bf Experimental Design:}} We design two instances where the rewards are  $\mathrm{Beta}$ distributed with   means and variances as in Table~\ref{table:items}. 
    There are $L=10$ base arms and the admissible solution set $\mathcal{A}_K$ contains all subsets of $[L]$ with cardinality no greater than $K=3$ (so $|\mathcal{A}_K|=175$). Since the arm distributions are supported on $[0,1]$, the sub-Gaussian parameter $\sigma^2=0.25$. The confidence parameter $\delta=0.05$. 

    \begin{table}[ht]
        \begin{tabular}{|c|c|c|c|c|c|}
        \hline
        \textbf{Item index} & \textbf{1} & \textbf{2} & \textbf{3} & \textbf{4} & \textbf{5 to 10} \\ \hline
        Means                           & $0.5$        & $0.45$       & $0.4$        & $0.35$       & $0.3 $             \\ \hline
        Variances  (\textbf{Set 1})                      & $0.24 $      & $0.24$       & $0.04$       & $0.01 $      & $0.01$             \\ \hline
         Variances     (\textbf{Set 2})                  & $0.01 $      &$ 0.01$       & $0.01 $      & $0.01$       & $0.01$             \\ \hline
        \end{tabular}
        \caption{Two sets of items with equal means for each item.}
        \label{table:items} \vspace{-.25in}
    \end{table}

    \textbf{In Experiment 1}, we quantify the additional regret due to the safeness checking and evaluate the performance of \AlgSCombUCB with \textbf{Set 1} under the unconstrained case. We run (1) \AlgSCombUCB with $\bsigma=0.6$ which needs to check the safeness of the solutions; (2) \AlgSCombUCB with $\bsigma=0.751$ which can be regarded as a variant of \AlgSCombUCB without the safeness constraint, since our algorithm is aware of the safeness of all solutions; (3) {\sc CombUCB1} which is a baseline algorithm.
    \newline
    \textbf{In Experiment 2}, we illustrate the effectiveness of \AlgSCombUCB in satisfying the safety constraint. Furthermore, we show that if an algorithm ignores the safety constraint, it will violate the safety constraint $\Omega(T)$ times if there exists a risky solution. We run \AlgSCombUCB and {\sc CombUCB1} with \textbf{Set 1} under the constrained case where $\bsigma=0.4$, the optimal safe solution is $\{1,3,4\}$, and the optimal solution under the unconstrained case $\{1,2,3\}$ is unsafe (risky). 
    \newline
    \textbf{In Experiment 3}, we empirically verify the dependence of the additional regret on the hardness parameter $H(\Delta(\Lambda))$ in \eqref{equ:Hinstance} using \textbf{Set 2}. We fix  the time horizon $T = 2\times 10^6$ and vary the threshold on the variance from $0.14$  to $0.72$ (i.e., $\bar{\sigma}^{2}= 0.14\times 1.2^k$ for $k=0,1,\dots,9$). As any solution that is comprised of $3$ items has variance $0.03$, we have $\Delta_S^{v}=\bar{\sigma}^2-0.03$. We compare the \emph{additional regret} with respect to $1/(\Delta_S^v)^2$, which is proportional to $H(\Delta(\Lambda))$ under this setup according to \eqref{equ:Hinstance}.

    \underline{{\bf   Experimental Results:}} For \textbf{ Experiment 1}, we present the results in Figure~\ref{exp:experiment1cumreg}.
    We first observe when $\bar{\sigma}^2=0.751$, the regret incurred by \AlgSCombUCB is similar to that by \CombUCB  for all $T$ considered, which suggests that \AlgSCombUCB is comparable to \CombUCB under the unconstrained case, and hence in the following experiments we refer the difference between the regret of \AlgSCombUCB and the regret of \CombUCB as the ``additional regret''. Secondly, when $\bar{\sigma}^2=0.6$, the regret of \AlgSCombUCB increases rapidly at the beginning and plateaus when $T>4\times 10^5$. This corroborates the design of \AlgSCombUCB: (i) at the beginning, \AlgSCombUCB pulls solutions conservatively to meet the anytime-safe constraint w.h.p.; (ii) after a number of time steps ($T>4\times 10^5$), the safeness of the optimal (safe) solution can be ascertained, it then exploits the optimal solution aggressively and eventually matches the performance of {\sc CombUCB1}.
    \newline
    For \textbf{Experiment 2}, we plot
    the percentage of times each algorithm violates the safeness constraint $\sigma_{A_t}^2<\bar{\sigma}^2$ as well as the cumulative rewards in Figure~\ref{exp:experiment2cumrew}.
    The reward of \AlgSCombUCB increases slowly at the start and then more rapidly when $T>1.5\times 10^6$, when the safeness of the optimal safe solution has been ascertained. However, while the reward of \CombUCB increases linearly (as it pulls the risky  solution $\{1,2,3\}$ $\Omega(T)$ times), it  violates the safeness constraint $\sigma_{S_t}^2<\bar{\sigma}^2$ at almost all times. This implies that the safety constraint is almost always violated by \CombUCB ($\Omega(T)$ times) whereas \AlgSCombUCB can meet the probably anytime-safe requirement. 
    \newline
    For \textbf{Experiment 3}, the results are in Figure~\ref{exp:experiment3AddReg}.
    As suggested by Theorem~\ref{thm:upbd}, the regret due to safeness checking is proportional to $H(\Delta(\Lambda))$. Figure \ref{exp:experiment3AddReg} indicates that empirically, the additional regret  scales linearly in $1/(\Delta_S^{\rmv})^2$, which corroborates our theoretical results.

    Additional discussions on the tightness results, the problem formulation and comparisons with other literature, as well as future research directions, are presented in App.~\ref{sec:additional}.


\vspace{-0.1in}
\subsection*{Acknowledgements}
\vspace{-0.05in}
The authors are  supported by Singapore Ministry of Education (MOE) grants  (Grant Numbers: A-0009042-01-00, A-8000189-01-00, A-8000980-00-00, A-8000423-00-00) and funding from CIFAR through Amii and NSERC.
\bibliographystyle{icml2023}
\bibliography{reference_SafeComb}

	\newpage
	\onecolumn
 \begin{center}
    {\Large{\bf Appendices}} 
 \end{center}
 \noindent
 The contents of the appendices are organized as follows:
 \begin{itemize}
     \item In App.~\ref{sec:auxiliary}, we list $3$ useful lemmas concerning the LIL concentration bound.
     \item In App.~\ref{sec:proof_upper_bd}, we present detailed proofs of the upper bounds.
     \begin{itemize}
         \item App.~\ref{sec:proof_UpBd_pre}: preliminary results for the proof of the upper bound;
         \item App.~\ref{sec:proof_UpBd_decomp}: the proof of the decomposition lemma Lemma~\ref{lem:upbd_decomposition};
        \item  App.~\ref{sec:proof_UpBd_sub}: the proof of Lemma~\ref{lem:subreg} (the regret due to suboptimality);
        \item App.~\ref{sec:safe_reg}: the proof of Lemma~\ref{lem:safereg} (the regret due to safeness-checking);
        \item App.~\ref{sec:proofUppBd}: the proofs of Theorem~\ref{thm:upbd} (problem-dependent upper bound) and Theorem~\ref{thm:problemindependent_upbd} (problem-independent upper bound).
     \end{itemize}
     \item In App.~\ref{sec:lower_bd}, we present detailed proofs of the lower bounds.
     \begin{itemize}
         \item App.~\ref{sec:proof_LwBd_imp}: preliminary results for the proof of the lower bound and the proof of the impossibility result Theorem~\ref{thm:impos}.
         \item App.~\ref{sec:proof_LwBd_dep}: the proof of Theorem~\ref{thm:low_bd_prob_dep} (problem-dependent lower bound);
         \item App.~\ref{sec:proof_LwBd_indep}: the proof of Theorem~\ref{thm:low_bd_prob_indep} (problem-independent lower bound);
     \end{itemize}
     \item In App.~\ref{sec:tightness}, we present a corollary characterizing the tightness of the upper bound in Theorem~\ref{thm:upbd}.
     \item In App.~\ref{sec:additional}, we provide additional discussions on the tightness results, the problem formulation and comparisons with other literature, as well as future research directions.
 \end{itemize}

	\appendix
	\section{Auxiliary results}\label{sec:auxiliary}
	\begin{lem}[Lemma 3 in \cite{jamieson14lil}]\label{lilbound}
		Let $\{X_i\}_{i=1}^\infty$ be a sequence of i.i.d.\ centered sub-Gaussian random variables with scale parameter $\sigma$. Fix any $\epsilon\in(0,1)$ and $\delta\in(0,\ln(1+\epsilon)/e)$. Then one has 
		\begin{align}
			&\mathbb{P}\bigg[\forall\, t\!\in\!\mathbb{N}:\!\sum_{s=1}^t X_s\!\leq\! (1\!+\!\sqrt{\epsilon})\sqrt{2\sigma^2\left(1\!+\!\epsilon\right)t\ln\left(\frac{\ln\left((1\!+\!\epsilon)t\right)}{\delta}\right)}
			\bigg]
            \geq
			1-\xi(\delta),
		\end{align}
		where $\xi(\delta):=\frac{2+\epsilon}{\epsilon}\big(\frac{\delta}{\ln (1+\epsilon)}\big)^{1+\epsilon}$.
	\end{lem}

    \begin{lem}\label{lem:comlilbound}
        For $t\geq 1,\epsilon\in(0,1),\omega\in(0,1]$ and $u>0$, let 
        $\gamma:=\frac{(1+\epsilon)(1+\sqrt{\epsilon})^2}{2}$, 
        $c:=\frac{(u\cdot s)^2}{\gamma}=\frac{2(u\cdot s)^2}{(1+\epsilon)(1+\sqrt{\epsilon})^2}$ and
        $m:=\frac{\gamma}{u^2\cdot s^2} \left(2\ln\frac{1}{\omega}+\ln\ln_+\frac{1}{s^2}+\ln \frac{2\gamma(1+\epsilon) }{u^2}\right)$. 
        If $t> m $, it holds that
        $$s> \mathrm{lil}(t,\omega)=(1+\sqrt{\epsilon})\sqrt{\frac{1+\epsilon}{2t}\ln\left(\frac{\ln\left((1\!+\!\epsilon)t\right)}{\omega}\right)}.$$
    \end{lem}
    \begin{proof}[Proof of Lemma~\ref{lem:comlilbound}]
        Note that fact that 
        \begin{align}
            u\cdot s\leq (1+\sqrt{\epsilon})\sqrt{\frac{1+\epsilon}{2t}\ln\left(\frac{\ln\left((1\!+\!\epsilon)t\right)}{\omega}\right)}
            \Longleftrightarrow
            c=\frac{2(u\cdot s)^2}{(1+\epsilon)(1+\sqrt{\epsilon})^2}
            \leq
            \frac{1}{t}\ln\left(\frac{\ln\left((1\!+\!\epsilon)t\right)}{\omega}\right)
        \end{align}
        According to the computations in \citet{jamieson14lil} equation (1), i.e.,
        \begin{align}
            \frac{1}{t} \ln \left(\frac{\ln ((1+\epsilon) t)}{\omega}\right) \geq c^\prime \Rightarrow t \leq \frac{1}{c^\prime} \ln \left(\frac{2 \ln ((1+\epsilon) /(c^\prime \omega))}{\omega}\right)
        \end{align}
        for $t\geq 1,\epsilon\in(0,1),c^\prime>0,\omega\in(0,1]$. 
        We take $c^\prime=c$, thus
        \begin{align}
            t&\leq\frac{1}{ c} \ln \left(\frac{2 \ln ((1+\epsilon) /( c \omega))}{\omega}\right)
            \\
            &=
                \frac{1}{c} \left(\ln\frac{2}{\omega}
                +\ln \left(\ln \frac{\gamma(1+\epsilon)}{u^2\cdot\omega}
                +\ln\frac{1}{s^2}\right)\right)
            \\
            &\overset{(a)}{\leq}
             \frac{1}{ c} \left(\ln\frac{2}{\omega}+\ln \frac{\gamma(1+\epsilon)}{u^2\cdot \omega}+\ln\ln_+\frac{1}{s^2}\right)
            \\
            &=
             \frac{\gamma}{u^2\cdot s^2} \left(2\ln\frac{1}{\omega}+\ln\ln_+\frac{1}{s^2}+\ln \frac{2\gamma(1+\epsilon) }{u^2}\right)=m
        \end{align}
        where we adopt $\ln(x+y)\leq x+\ln\ln_+y,\forall x,y\in\mathbb{R}_+$ in $(a)$.
        Therefore, if $t>m$, we must have $$u\cdot s> (1+\sqrt{\epsilon})\sqrt{\frac{1+\epsilon}{2t}\ln\left(\frac{\ln\left((1\!+\!\epsilon)t\right)}{\omega}\right)}.$$
    \end{proof}
 
	\begin{lem}\label{lem:lil_var_bd}
		With the choice of the confidence radii in \eqref{confidencebounds_mean} and \eqref{confidencebounds_var}, for all $i\in E$, we have 
		\begin{align}
			&\mathbb{P}\left[\forall\, \phase\!\in\!\mathbb{N}:|\hat{\mu}_i(\phase)-\mu_i|\leq\alpha(T_i(\phase))\right]\geq	1-2\xi(\omega_\mu)\label{lil:mean}\\
			&\mathbb{P}\left[\forall\, \phase\!\in\!\mathbb{N}:|\hat{\sigma}_i^2(\phase)-\sigma_i^2|\leq\beta_\rmu(T_i(\phase))\right]\geq	1-4\xi(\omega_\rmv)\label{lil:upvar}\\
			&\mathbb{P}\left[\forall\, \phase\!\in\!\mathbb{N}:|\hat{\sigma}_i^2(\phase)-\sigma_i^2|\leq\beta_\rml(T_i(\phase))\right]\geq	1-4\xi(\omega_\rmv^\prime)\label{lil:lowvar}\\
		\end{align}
	\end{lem}
	
	\begin{proof}
		Note the fact that any distribution supported on $[0,1]$ is $1/4$-sub-Gaussian.
		By a direct application of Lemma~\ref{lilbound} to the sample mean $\hat{\mu}_i(\phase)$ and the  sample second  moment $\hat{M}_{2,i}(\phase):=\frac{1}{T_i(\phase)}\sum_{s=1}^{\phase}W_i(s)^2\mathbbm{1}\{i \in A_s\}$ of arm $i\in[L]$, \eqref{lil:mean} can be derived and 
		\begin{align}
			&\mathbb{P}\left[\forall\, \phase\!\in\!\mathbb{N}:|\mu_i-\hat{\mu}_i(\phase)|\leq\mathrm{lil}(T_i(\phase),\omega_\rmv^\prime)\right]\geq	1-2\xi(\omega_\rmv^\prime),\quad\mbox{and}
			\\
			 &\;\mathbb{P}\left[\forall\, \phase\!\in\!\mathbb{N}: |\hat{M}_{2,i}(\phase)-(\mu_i^2+\sigma_i^2)|\leq \mathrm{lil}(T_i(\phase),\omega_\rmv^\prime)\right]
            \geq	1-2\xi(\omega_\rmv^\prime)
		\end{align}
		Since the rewards are in $[0,1]$, $|\mu_i^2-\hat{\mu}_i^2(\phase)|=|\mu_i+\hat{\mu}_i(\phase)| \cdot |\mu_i-\hat{\mu}_i(\phase)|\le 2\cdot\mathrm{lil}(T_i(\phase),\omega_\rmv^\prime)$. Using this and the 
		 triangle inequality, we obtain for every $\phase\ge1$,
		\begin{align}		
			|\hat{\sigma}^2_i(\phase)-\sigma_i^2|
			&=|\mu_i^2-\hat{\mu}_i^2(\phase)|+|(\mu_i^2+\sigma_i^2)-\hat{M}_{2,i}(\phase)|\\*
			&
			\leq 2\cdot\mathrm{lil}(T_i(\phase),\omega_\rmv)+\mathrm{lil}(T_i(\phase),\omega_\rmv)
			=\beta_\rmu(T_i(\phase)).
		\end{align}
		Therefore, \eqref{lil:upvar} is proved. \eqref{lil:lowvar} can be similarly obtained.
	\end{proof}

    \section{Proof of the Upper Bound}\label{sec:proof_upper_bd}
    \subsection{Proof scheme of the problem-dependent upper bound}\label{sec:proof_UpBd_pre}
    In this subsection, we provide technical lemmas that can upper bound the components in $\rmR_1(T^\prime)$ and $\rmR_2(T^\prime)$.
    
    Note that at phase $p$, the identified solution $A_\phase$ belongs to one of the $4$ disjoint sets: (1) $A_\phase=S^\star$; (2) $\mathcal{S}\cap\mathcal{B}$; (3) $\mathcal{R}$ and (4) $\mathcal{S}^c\cap\mathcal{B}$, i.e.
    \begin{align}
        1&= \mathbbm{1}\left\{A_\phase=S^\star\right\}+\mathbbm{1}\left\{A_\phase\in\Safeset\cap\Subopt\right\}
        +\mathbbm{1}\left\{A_\phase\in\Riskset\right\}+\mathbbm{1}\left\{A_\phase\in\Safeset^c\cap\Subopt\right\}
    \end{align}
    and $\mathbbm{1}\left\{A_\phase\in\Subopt\right\}=\mathbbm{1}\left\{A_\phase\in\Safeset\cap\Subopt\right\}+\mathbbm{1}\left\{A_\phase\in\Safeset^c\cap\Subopt\right\}$.
    Define two events (the $\mathcal{F}$ events) that connect the instance and the confidence radii
    \begin{align}
		 \mathcal{F}_\phase^\mu &:=\bigg\{\Delta_{A_\phase}\leq 2\sum_{i\in A_\phase\setminus S^\star}\alpha(T_i(\phase-1))\bigg\}
        \\
		\mathcal{F}_\phase(x,\rho)&:=\bigg\{x\leq 2\sum_{i\in A_\phase}\mathrm{lil}(T_i(\phase-1),\rho)\bigg\}
	\end{align}
	where $x$ is a constant and $\omega$ is a confidence parameter.  
    When $A_\phase\in\Subopt$, it indicates solution $A_\phase$ has not been sampled sufficiently many times and its suboptimality has not been ascertained. 
    When $A_\phase\in\Safeset^c$, it implies the unsafeness of $A_\phase$ has not been recognized.
    We formalize this in the following lemma.
    \begin{lem}\label{lem:upbd_decomposition1}
        Conditional on the event $\mathcal{E}$,
        given any $p\in[T]$, we have 
        \begin{itemize}
            \item $S^\star\in\bar{\mathcal{S}}_{\phase-1}$;
            \item If $A_\phase\in\Safeset\cap\Subopt$,
                \begin{align}
                    \mathbbm{1}\left\{A_\phase\in\Safeset\cap\Subopt\right\}\leq \mathbbm{1}\left\{\mathcal{F}_\phase^\mu\right\};
                \end{align}
            \item If $A_\phase\in\Riskset$,
            \begin{align}
                \mathbbm{1}\left\{A_\phase\in\Riskset\right\}\leq \mathbbm{1}\left\{\mathcal{F}_\phase\left(\frac{\Deltav_{A_\phase}}{3},\omega_\rmv^\prime\right)\right\};
            \end{align}
            \item If $A_\phase\in\Safeset^c\cap\Subopt$,
            \begin{align}
                \mathbbm{1}\left\{A_\phase\in\Safeset^c\cap\Subopt\right\}\leq \mathbbm{1}\left\{\mathcal{F}_\phase^\mu,\mathcal{F}_\phase\left(\frac{\Deltav_{A_\phase}}{3},\omega_\rmv^\prime\right)\right\}.
            \end{align}
        \end{itemize}
    \end{lem}
          \begin{proof}[Proof of Lemma~\ref{lem:upbd_decomposition1}]
            By the design of \AlgSCombUCB, $A_\phase\in\bar{\mathcal{S}}_{\phase-1}$.
    
            (1) We firstly prove that $S\in\bar{\mathcal{S}}_{\phase-1},\forall S\in\Safeset$. On the event $\mathcal{E}$, we have 
            \begin{align}
                L_{S}^\rmv(\phase-1)&=\sum_{i\in A} \max\{\hat{\sigma}^2_i(\phase-1)-\beta_\rml(T_i(\phase-1)),0\}
                \\
                &\leq \sum_{i\in A} \max\{\sigma^2_i,0\}
                \\
                &=\sigma_S^2<\bsigma^2
            \end{align}
            Thus, $S\in\bar{\mathcal{S}}_{\phase-1}$, and in particular, $S^\star\in\bar{\mathcal{S}}_{\phase-1}$.
            
            (2) If $A_\phase\in\Subopt$, according to the sampling strategy in Line~$10$ of \AlgSCombUCB and $S^\star\in\bar{\mathcal{S}}_{\phase-1}$, 
            we have $U_{S^\star}^\mu(\phase-1) \leq U_{A_\phase}^\mu(\phase-1)$ 
            which indicates
            $
                \sum_{i\in S^\star\setminus A_\phase}U_i^\mu(\phase-1)
                \leq
                \sum_{i\in A_\phase\setminus S^\star }U_i^\mu(\phase-1)
                $.
            Thus,
                \begin{align}
                    \sum_{i\in S^\star\setminus A_\phase}\mu_i
                    &\leq
                    \sum_{i\in S^\star\setminus A_\phase}U_i^\mu(\phase-1)
                    \\
                    &\leq
                    \sum_{i\in A_\phase\setminus S^\star }U_i^\mu(\phase-1)
                    \\
                    &\leq
                    \sum_{i\in A_\phase\setminus S^\star }\mu_i+2\alpha_i(T_i(\phase-1))
                    \\
                    \Longrightarrow\qquad&\Delta_{A_\phase}\leq 2\sum_{i\in A_\phase\setminus S^\star}\alpha(T_i(\phase-1))\label{equ:Apinsub}
                \end{align}

            (3) If $A_\phase\in\Safeset^c$, according to the sampling strategy,
            we have $A_\phase\in\bar{\mathcal{S}}_{\phase-1}$ 
            which indicates
            $
                L_{A_\phase}^\rmv(\phase-1)=\sum_{i\in A_\phase}L_i^\rmv(\phase-1)< \bsigma^2
            $.
            Thus,
                \begin{align}
                   \bsigma^2&>\hat{\sigma}_{A_\phase}^2(\phase-1)-\sum_{i\in A_\phase}\beta_\rml(T_i(\phase-1))
                   \\
                   &\geq
                   \sigma_{A_\phase}^2-2\sum_{i\in A_\phase}\beta_\rml(T_i(\phase-1))
                   \\
                    \Longrightarrow\qquad&\Deltav_{A_\phase}\leq 2\sum_{i\in A_\phase\setminus S^\star}\beta_\rml(T_i(\phase-1))\label{equ:Apinunsafe}
                \end{align}
    
            Note that if $A_p\in\Safeset\cap\Subopt$, according to \eqref{equ:Apinsub}, $$\mathbbm{1}\left\{A_\phase\in\Safeset\cap\Subopt\right\}\leq \mathbbm{1}\left\{\mathcal{F}_\phase^\mu\right\}.$$
            If $A_\phase\in\Riskset\subset\Safeset^c$, by \eqref{equ:Apinunsafe}
            $$
                    \mathbbm{1}\left\{A_\phase\in\Riskset\right\}\leq \mathbbm{1}\left\{\mathcal{F}_\phase\left(\frac{\Deltav_{A_\phase}}{3},\omega_\rmv^\prime\right)\right\}.
            $$
            If $A_\phase\in\Safeset^c\cap\Subopt$, by \eqref{equ:Apinunsafe} and \eqref{equ:Apinsub}
            $$
                    \mathbbm{1}\left\{A_\phase\in\Safeset^c\cap\Subopt\right\}\leq \mathbbm{1}\left\{\mathcal{F}_\phase^\mu,\mathcal{F}_\phase\left(\frac{\Deltav_{A_\phase}}{3},\omega_\rmv^\prime\right)\right\}
            $$
        \end{proof}
        
    At phase $\phase$, we define two sequences of mutually-exclusive events $\{\calGmu_{j,\phase}\}_{j\in\mathbb{N}}$ and $\{\calG_{j,\phase}(x,\omega\}_{j\in\mathbb{N}}$  (the $\mathcal{G}$ events) which can further bound the number of times $\mathcal{F}^\mu_\phase$ and $\mathcal{F}^\rmv_\phase(x,\omega)$ occur respectively.
    These events are indexed by two strictly-decreasing sequences of constants:
    \begin{align}
         a_1&>a_2>\ldots>a_k>\ldots\\
        1&> b_1>b_2>\ldots>b_k>\ldots
    \end{align}
    where $\lim_{j\to\infty}a_j=\lim_{j\to\infty}b_j=0$. For simplicity, we set $a_j= \frac{4}{9^{j-2}}$, $b_j=\frac{1}{4^j},\forall j\in\bbN$ and denote the constant $\abcstt=\sum_{j\in\bbN}\frac{a_j}{b_j}=259.2$. For $x\in\mathbb{R}_+$ and $\omega\in(0,\ln(1+\epsilon)/e)$, define 
      \begin{align}
        m_{j}(x,\omega):=
        \frac{a_j\cdot \gamma K^2}{x^2}\left(2\ln\frac{1}{\omega}+\ln\ln_+\frac{1}{x^2}+\upbdcstt\right)
    \end{align}
    and $m_{j}(x,\omega):=\infty$ otherwise, 
    where 
    (1) $\gamma=\frac{(1+\epsilon)(1+\sqrt{\epsilon})^2}{2}$ and $\epsilon$ is the constant in the confidence bounds \eqref{confidencebounds_var}, 
    (2) $\ln\ln_+(x)=\ln\ln x$ if $x\geq e$ and it equals to $0$ otherwise,
    (3) $\upbdcstt=\ln\left(324 K^2 (1+\epsilon)^2(1+\sqrt{\epsilon})^2\right)$.
    Denote 
        \begin{align}
             G_{j,\phase}^\mu &:=\left\{i\in A_\phase\setminus S^\star: T_i(\phase-1)\leq m_j(\Delta_{A_\phase},\omega_\mu)\right\}\quad \text{and}
            \\
            G_{j,\phase}(x,\omega) &:=\left\{i\in A_\phase: T_i(\phase-1)\leq m_j(x,\omega)\right\}
        \end{align}
    as the sets of items that were not chosen sufficiently often.
    For $j\in\mathbb{N}$, the events at phase $p$ are sequentially defined as
    \begin{align}
        \calGmu_{j,\phase}&:= \Big\{\text{at least }b_j K\text{ items in }A_\phase\setminus S^\star\text{ were chosen}
        \\
        &\text{ at most }m_{j}(\Delta_{A_\phase},\omega_\mu)\text{ times}\Big\}
        \bigcap
        \left(\bigcup_{k\in[j-1]}\calGmu_{k,\phase}\right)^c
        \\
        &= \Big\{\left|G_{j,\phase}^\mu\right|\geq b_j K\Big\}
        \bigcap
        \left(\bigcup_{k\in[j-1]}\calGmu_{k,\phase}\right)^c \qquad\mbox{and}
    \end{align}
      \begin{align}
        \calG_{j,\phase}(x,\omega)&:=\Big\{\text{at least }b_j K\text{ items in }A_\phase\text{ were chosen}
        \\
        &\text{ at most }m_{j}(x,\omega)\text{ times}\Big\}
        \bigcap
        \left(\bigcup_{k\in[j-1]}\calG_{k,\phase}(x,\omega)\right)^c
        \\
        &=\Big\{ \left|G_{j,\phase}(x,\omega)\right|\geq b_j K\Big\}
        \bigcap
        \left(\bigcup_{k\in[j-1]}\calG_{k,\phase}(x,\omega)\right)^c
    \end{align}

    \begin{lem}\label{lem:FtoG}
        With our choice of $\{a_j\}_{j\in\bbN}$ and $\{b_j\}_{j\in\bbN}$,
        \begin{itemize}
            \item if $\mathcal{F}_\phase^\mu$ occurs, $\calGmu_{j,\phase}$ occurs for some $j$, i.e. 
            \begin{align}
                \mathbbm{1}\left\{\mathcal{F}_\phase^\mu\right\}\leq \mathbbm{1}\left\{\bigcup_{j\in\mathbb{N}}\calGmu_{j,\phase}\right\}.
            \end{align}
            \item if $\mathcal{F}_\phase(x,\omega)$ occurs, $\calG_{j,\phase}(x,\omega)$ occurs for some $j$, i.e. 
            \begin{align}\label{equ:FtoG}
                \mathbbm{1}\left\{\mathcal{F}_\phase(x,\omega)\right\}\leq \mathbbm{1}\left\{\bigcup_{j\in\mathbb{N}}\calG_{j,\phase}(x,\omega)\right\}.
            \end{align}
        \end{itemize}
    \end{lem}
    \begin{proof}[Proof of Lemma~\ref{lem:FtoG}]
        We prove \eqref{equ:FtoG} in the following and the other statement can be proved by the same procedures.

        To ease the notations, we omit the parameters $x,\omega$ and $\phase$ in $\mathcal{F}_\phase(x,\omega),\calG_{j,\phase}(x,\omega)$ and $m_{j}(x,\omega)$, since they are fixed when given $\mathcal{F}_\phase(x,\omega)$.
        The event $\calG_{j}$ can be rewritten as
    \begin{align}
        \calG_{j}
        &=
        \Big\{ \left|G_{j}\right|\geq b_j K\Big\}
        \bigcap
        \left(\bigcap_{k\in[j-1]}\calG_{k}^c\right)
        \\
        &=\Big\{ \left|G_{j}\right|\geq b_j K\Big\}
        \bigcap
        \left(\bigcap_{k\in[j-1]}
        \Big\{ \left|G_{k}\right|< b_k K\Big\}\right)
    \end{align}
        The statement is proved by contradiction. We assume that when $\mathcal{F}$ ($\mathcal{F}_\phase(x,\omega)$) occurs, none of event $\calG_{j}$ occurs.
        Hence,
        \begin{align}
            &\left(\bigcup_{j\in\mathbb{N}}\calG_{j}\right)^c
            =
            \bigcap_{j\in\mathbb{N}}\calG_{j}^c
            \\
            &=
            \bigcap_{j\in\mathbb{N}}
            \left[
            \Big\{ \left|G_{j}\right|< b_j K\Big\}
            \bigcup
            \left(\bigcup_{k\in[j-1]}
            \Big\{ \left|G_{k}\right|\geq b_k K\Big\}\right)
            \right]
            \\
            &=
            \bigcap_{j\in\mathbb{N}}
            \Big\{ \left|G_{j}\right|< b_j K\Big\}
        \end{align}
        Let $\bar{G}_j:=A_\phase\setminus G_j$ and define $G_0=A_\phase$. According to the definition of $G_j$, we have $G_j\subset G_{j-1}$ and 
        $\bar{G}_{j-1}\subset \bar{G}_{j},\forall j\in\mathbb{N}$. Because $\lim_{j \to \infty}m_j=0$, there exists $j_0$ such that $\bar{G}_{j}=A_\phase, \forall j\geq j_0$. Therefore, we can write $A_\phase$ by the ``telescoping'' sum, i.e., $A_\phase=\cup_{j\in\mathbb{N}}\left(\bar{G}_j\setminus\bar{G}_{j-1}\right)$.
        $\mathcal{F}_\phase(x,\omega)$ indicates
        \begin{align}\label{equ:contradiction}
            x&\leq 2\sum_{i\in A_\phase}\mathrm{lil}(T_i(\phase-1),\omega)
            \\
            &=2\sum_{j\in\mathbb{N}}\sum_{i\in \bar{G}_j\setminus\bar{G}_{j-1}}\mathrm{lil}(T_i(\phase-1),\omega)
        \end{align}
        Note that for $i\in \bar{G}_j\setminus\bar{G}_{j-1}=G_{j-1}\setminus G_j$, we have $T_i(\phase-1)\in(m_j,m_{j-1}]$.
        By Lemma~\ref{lem:comlilbound} with parameters $t=T_i(\phase-1), s=x,u=\sqrt{\frac{1}{a_jK^2}}$ and note $a_1>a_j$, we have
        \begin{align}
            \sqrt{\frac{1}{a_jK^2}}\cdot x>\mathrm{lil}(T_i(\phase-1),\omega).
        \end{align}
        Note our choice of $a_j$ and $b_j$ satisfy
        \begin{align}
            2\sum_{j\in\mathbb{N}}\frac{b_{j-1}-b_j}{\sqrt{a_j}}\leq 1
        \end{align}
        Thus, \eqref{equ:contradiction} can be further bounded by
        \begin{align}
            x&\leq 2\sum_{j\in\mathbb{N}}\sum_{i\in \bar{G}_j\setminus\bar{G}_{j-1}}\mathrm{lil}(T_i(\phase-1),\omega)
            \\
            &<
            2\sum_{j\in\mathbb{N}}|\bar{G}_j\setminus\bar{G}_{j-1}|\sqrt{\frac{1}{a_j K^2}}\cdot x
            \\
            &\leq
            2\sum_{j\in\mathbb{N}}\frac{(b_{j-1}-b_j)K}{\sqrt{a_j}K}x
            \\
            &\leq x
        \end{align}
        which constitutes a contradiction. Thus when $\mathcal{F}$ occurs, there must exists $j\in\mathbb{N}$ such that $\calG_j$ occurs.
    \end{proof}
    When none of the $\mathcal{G}_{j,p}^\mu$ (resp. $\calG_{j,\phase}(x,\omega)$) occurs, $\mathcal{F}_\phase^\mu$ (resp. $\mathcal{F}_\phase(x,\omega)$) must not occur, which indicates all of the items in $A_p$ have been sampled sufficiently many times such that the suboptimality (resp. unsafeness) of $A_\phase$ is identified, thus $A_\phase$ will not been sampled in future phases.

    As there will be multiple $\mathcal{F}$ events happening, we provide the following useful lemma that merges all $\mathcal{F}$ events.
    \begin{lem}\label{lem:FFtoF}
        Given two confidence parameters $\omega_1\geq\omega_2\in(0,\ln(1+\epsilon)/e)$
        \begin{itemize}
            \item If $\mathcal{F}_\phase^\mu$ occurs, then $\mathcal{F}_\phase(\Delta_{A_\phase},\omega_\mu)$ occurs.
            \item If both events $\mathcal{F}_\phase(x,\omega_1)$ and $\mathcal{F}_\phase(y,\omega_2)$ occur, then event 
            $\mathcal{F}_\phase\left(\max\{x,\sqrt{
            \frac{\ln\frac{1}{\omega_1}}{\ln\frac{1}{\omega_2}}
            } y\},\omega_1\right)$ occurs.
        \end{itemize}
        \end{lem}
    \begin{proof}[Proof of Lemma~\ref{lem:FFtoF}] 
        If $\mathcal{F}_\phase^\mu$ occurs, we have 
        \begin{align}
            \Delta_{A_\phase}\leq 2\sum_{i\in A_\phase\setminus S^\star}\alpha(T_i(\phase-1))
            \leq 2\sum_{i\in A_\phase}\alpha(T_i(\phase-1))
            = 2\sum_{i\in A_\phase}\mathrm{lil}(T_i(\phase-1),\omega_\mu)
        \end{align}
        Thus, $\mathcal{F}_\phase(\Delta_{A_\phase},\omega_\mu)$ occurs.
    
        For the second statement, notice the fact that 
        \begin{align}
            \frac{\mathrm{lil}(T_i(\phase-1),\omega_1)}{\mathrm{lil}(T_i(\phase-1),\omega_2)}
            &=
            \frac{\left(1+\sqrt{\epsilon}\right)\Big(\frac{1+\epsilon}{2T_i(\phase-1)}\ln\big(\frac{\ln((1+\epsilon)T_i(\phase-1))}{\omega_1}\big) \Big)^{1/2}}{\left(1+\sqrt{\epsilon}\right)\Big(\frac{1+\epsilon}{2t}\ln\big(\frac{\ln((1+\epsilon)T_i(\phase-1))}{\omega_2}\big) \Big)^{1/2}}
            \\
            &=
            \sqrt{
            \frac{\ln\frac{1}{\omega_1}+\ln\ln((1+\epsilon)T_i(\phase-1))}{\ln\frac{1}{\omega_2}+\ln\ln((1+\epsilon)T_i(\phase-1))}
            }
            \\
            &\overset{(a)}{\geq}
            \sqrt{
            \frac{\ln\frac{1}{\omega_1}}{\ln\frac{1}{\omega_2}}
            }
            =:\rho
        \end{align}
        where $(a)$ utilizes the trick that $\frac{a+c}{b+c}\geq\frac{a}{b},\forall a,b,c\in\mathbb{R}_+$ and $a\leq b$.
        When event $\mathcal{F}_\phase(y,\omega_2)$ occurs,
        \begin{align}
            &y
            \leq 
            2\sum_{i\in A_\phase}\mathrm{lil}(T_i(\phase-1),\omega_2)
            \leq 2\sum_{i\in A_\phase}\frac{1}{\omega}\mathrm{lil}(T_i(\phase-1),\omega_1)
            \\
            \Longrightarrow
            \quad
            &
            \rho\cdot y\leq 2\sum_{i\in A_\phase}\mathrm{lil}(T_i(\phase-1),\omega_1)
            \\
            \Longrightarrow
            \quad
            &
            \mathcal{F}_\phase(\rho y,\omega_1)
        \end{align}
        Thus if both events $\mathcal{F}_\phase(x,\omega_1)$ and $\mathcal{F}_\phase(y,\omega_2)$ occur, we must have that event 
        $\mathcal{F}_\phase\left(\max\{x,\rho y\},\omega_1)\right)=\mathcal{F}_\phase\left(\max\{x,\sqrt{
            \frac{\ln\frac{1}{\omega_1}}{\ln\frac{1}{\omega_2}}
            } y\},\omega_1)\right)$ occurs.
    
        \end{proof}

    \subsection{Upper bound decomposition}\label{sec:proof_UpBd_decomp}
        
    \lemUpBdDecom*
        \begin{proof}[Proof of Lemma~\ref{lem:upbd_decomposition}]
    		The expected regret can be decomposed as:
    	\begin{align}
    		\mathbb{E}\left[\rmR(\Tp)\right]&=\mathbb{E}\left[\sum_{\phase=1}^{\Tp}\sum_{r=1}^{n_\phase}(\mu^\star -\mu_{A_{\phase,r}})\right]
    		\\
    		&=\mathbb{E}\left[\sum_{\phase=1}^{\Tp}\sum_{r=1}^{n_\phase}(\mu^\star -\mu_{A_{\phase,r}})
            \mathbbm{1}\left\{\mathcal{E}\right\}
            \right]
            +
    		\mathbb{E}\left[\sum_{\phase=1}^{\Tp}\sum_{r=1}^{n_\phase}(\mu^\star -\mu_{A_{\phase,r}})
            \mathbbm{1}\left\{\mathcal{E}^c\right\}
            \right]
            \\
            &=\mathbb{E}\left[\sum_{\phase=1}^{\Tp}\sum_{r=1}^{n_\phase}(\mu^\star -\mu_{A_{\phase,r}})
            \bigg|\mathcal{E}
            \right]\mathbb{P}[\mathcal{E}]
            +
    		\mathbb{E}\left[\sum_{\phase=1}^{\Tp}\sum_{r=1}^{n_\phase}(\mu^\star -\mu_{A_{\phase,r}})
            \bigg|\mathcal{E}^c
            \right]\mathbb{P}[\mathcal{E}^c]
    	\end{align}
    
    		In the initialization stage, it will take at most $2L$ time steps, since each pulled solution $A_\phase$ contains at least one item $i$ with $T_i(\phase)<2$. Thus the regret is at most $2L\cdot\mu^\star $.
    
    		The expected regret when the good events fail can be upper bounded by
            \begin{align}
                \mathbb{E}\left[\sum_{\phase=1}^{\Tp}\sum_{r=1}^{n_\phase}(\mu^\star -\mu_{A_{\phase,r}})
                \bigg|\mathcal{E}^c
                \right]\mathbb{P}\left[\mathcal{E}^c\right]
                &\leq
                \mathbb{E}\left[\sum_{\phase=1}^{\Tp}\sum_{r=1}^{n_\phase}\mu^\star 
                \bigg|\mathcal{E}^c
                \right]\mathbb{P}[\mathcal{E}^c]
                \\
                &\leq
                \mathbb{E}\left[\sum_{t=1}^{T}\mu^\star 
                \bigg|\mathcal{E}^c
                \right]
                L\cdot2(\xi(\omega_\mu)+2\xi(\omega_v)+2\xi(\omega_v^\prime))
                \\
                &\leq
                2\mu^\star TL\cdot(\xi(\omega_\mu)+2\xi(\omega_v)+2\xi(\omega_v^\prime))
            \end{align}
    		where $\mathbb{P}[\mathcal{E}^c] $ can be bounded using Lemma~\ref{lem:lil_var_bd}.

    		Conditional on the good event $\mathcal{E}$, the high-probability regret can be upper bounded by 
    		\begin{align}
    			\hat{\rmR}(\Tp)&:=\sum_{\phase=1}^{\Tp}\sum_{r=1}^{n_\phase}(\mu^\star -\mu_{A_{\phase,r}})\label{equ:high_pro_regret}
    			\\
    			&=\sum_{\phase=1}^{\Tp}\left[\Delta_{A_\phase}+\mu^\star (n_\phase-1)\right]
                \\
                &\overset{(a)}{\leq}\sum_{\phase=1}^{\Tp}\left[\Delta_{A_\phase}+\mu^\star \sum_{r=0}^{Q-1}\left(2\cdot\mathbbm{1}\{U_{A_\phase}^\rmv(\phase-1)>r\bsigma^2\}-2\right)\right]
                \\
                &=\sum_{\phase=1}^{\Tp}\left[\Delta_{A_\phase}+\mu^\star \sum_{r=1}^{Q-1}2\cdot\mathbbm{1}\{U_{A_\phase}^\rmv(\phase-1)>r\bsigma^2\}\right]
        \end{align}
        where $(a)$ makes use of Lemma~\ref{lem:np} and the fact if
    		$U_{A_\phase}^\rmv(p-1)\in((m-1)\bsigma^2,m\bsigma^2] $, then $m=\sum_{r=0}^{Q-1}\mathbbm{1}\{U_{A_\phase}^\rmv(p-1)>r\bsigma^2\}=\sum_{r=0}^{Q-1}\mathbbm{1}\{\mathcal{U}_p(r)\}$. Note the fact that only the suboptimal solutions yields positive mean gap and that the negative mean gap of the risky solutions can be upper bounded by $0$, thus the above equation can be further divided into two parts
    		\begin{align}\label{equ:R12}
    			&\sum_{\phase=1}^{\Tp}\left[\Delta_{A_\phase}+\mu^\star \sum_{r=1}^{Q-1}2\cdot\mathbbm{1}\{U_{A_\phase}^\rmv(\phase-1)>r\bsigma^2\}\right]
                \\
                &\leq
                \sum_{\phase=1}^{\Tp}\mathbbm{1}\{A_\phase\in\Subopt\}
                \Delta_{A_\phase}
                +
                \sum_{\phase=1}^{\Tp}\mu^\star \sum_{r=1}^{Q-1}2\cdot\mathbbm{1}\{U_{A_\phase}^\rmv(\phase-1)>r\bsigma^2\}
                \\
                &=:
                \rmR_1(\Tp)+\rmR_2(\Tp)
        \end{align}
        In conclusion, by summarizing the regret from the initialization stage, the regret due to failure of the good event and the high-probability regret, the expected regret can be bounded by
        \begin{align}
             \mathbb{E}[\rmR(\Tp)]
             &
             \leq \mathbb{E}[\rmR_1(\Tp)|\mathcal{E}]\mathbb{P}[\mathcal{E}]+\mathbb{E}[\rmR_2(\Tp)|\mathcal{E}]\mathbb{P}[\mathcal{E}]+2\mu^\star \cdot TL\left(\xi(\omega_\mu)+2\xi(\omega_v)+2\xi(\omega_v^\prime)\right)+2\mu^\star L
             \\
            &\leq \mathbb{E}[\rmR_1(\Tp)|\mathcal{E}]+\mathbb{E}[\rmR_2(\Tp)|\mathcal{E}]
            +2\mu^\star \cdot TL\left(\xi(\omega_\mu)+2\xi(\omega_v)+2\xi(\omega_v^\prime)\right)+2\mu^\star L
              \\
             &= \mathbb{E}[\rmR_1(T^\prime)|\mathcal{E}]+\mathbb{E}[\rmR_2(T^\prime)|\mathcal{E}]+\rmR_3(T)
        \end{align}
        \end{proof}

    \subsection{Regret due to suboptimality}\label{sec:proof_UpBd_sub}
    \lemsubreg*
        \begin{proof}[Proof of Lemma~\ref{lem:subreg}]
        Notice the fact that the suboptimal solution $A_\phase$ can be safe, i.e. $A_\phase\in\Safeset\cap\Subopt$, or unsafe, i.e., $A_\phase\in\Safeset^c\cap\Subopt$, we upper bound the regret under these the two scenarios separately.

        
        \noindent
        \textbf{\underline{Case $1$: $A_p\in\Safeset\cap\Subopt$}}

        By the definition of the event $\mathcal{G}_{j,\phase}^\mu$, we have 
        $$|G_{j,\phase}^\mu|=\left|\left\{i\in A_\phase\setminus S^\star: T_i(\phase-1)\leq m_j(\Delta_{A_\phase},\omega_\mu)\right\}\right|\leq b_jK$$
        which indicates
        \begin{align}\label{equ:factofGmu}
            \mathbbm{1}\left\{A_\phase\in\Safeset\cap\Subopt,\mathcal{G}_{j,\phase}^\mu\right\}
            \leq
            \frac{1}{b_j K}\sum_{i\in A_\phase\setminus S^\star}\mathbbm{1}\left\{A_\phase\in\Safeset\cap\Subopt,T_i(\phase-1)\leq m_j(\Delta_{A_\phase},\omega_\mu)\right\}.
        \end{align}
        Given an item $i\in E\setminus S^\star$, assume it is included $v_i$ solutions in $\Safeset\cap\Subopt$, we index them according to the decreasing order of their mean gaps, i.e. $\{A^{i,k}\}_{k\in [v_i]}$ with $\Delta_{A^{i,1}}\geq \ldots \geq \Delta_{A^{i,v_i}}$.
        Therefore,
        \begin{align}
           &\sum_{\phase=1}^T
            \mathbbm{1}\left\{A_\phase\in\Safeset\cap\Subopt\right\}
            \Delta_{A_\phase}\cdot
             \mathbbm{1}
			\left\{\mathcal{E}\right\}
            \\
            &\overset{(a)}{\leq}
            \sum_{\phase=1}^T
            \mathbbm{1}\left\{A_\phase\in\Safeset\cap\Subopt,\mathcal{F}_\phase^\mu\right\}
            \Delta_{A_\phase}
            \\
            &\overset{(b)}{\leq}
            \sum_{\phase=1}^T   \sum_{j\in\mathbb{N}}
            \mathbbm{1}\left\{A_\phase\in\Safeset\cap\Subopt,\mathcal{G}_{j,\phase}^\mu\right\}
            \Delta_{A_\phase}
            \\
            &\leq
            \sum_{\phase=1}^T   \sum_{j\in\mathbb{N}} 
            \frac{1}{b_j K} \sum_{i\in A_\phase\setminus S^\star}
            \mathbbm{1}\left\{A_\phase\in\Safeset\cap\Subopt,T_i(\phase-1)\leq m_j(\Delta_{A_\phase},\omega_\mu)\right\}
            \Delta_{A_\phase}
            \\
            &\leq
            \sum_{\phase=1}^T  
            \sum_{j\in\mathbb{N}} 
            \frac{1}{b_j K} 
            \sum_{i\in E\setminus S^\star}
            \sum_{k\in[v_i]}
            \mathbbm{1}\left\{A^{i,k}=A_\phase,i\in A_\phase,T_i(\phase-1)\leq m_j(\Delta_{A^{i,k}},\omega_\mu)\right\}
            \Delta_{A^{i,k}}
            \\
            &\leq
            \sum_{i\in E\setminus S^\star}
            \sum_{j\in\mathbb{N}} 
            \sum_{\phase=1}^T 
            \sum_{k\in[v_i]}
            \frac{\Delta_{A^{i,k}}}{b_j K} 
            \mathbbm{1}\left\{A^{i,k}=A_\phase,i\in A^{i,k},T_i(\phase-1)\leq 
            \frac{a_j\cdot \gamma K^2}{\Delta_{A^{i,k}}^2}\left(2\ln\frac{1}{\omega_\mu}+\ln\ln_+\frac{1}{\Delta_{A^{i,k}}^2}+\upbdcstt\right)
            \right\}
            \\
            &\overset{(c)}{\leq}
            \sum_{i\in E\setminus S^\star}
            \sum_{j\in\mathbb{N}} 
            \sum_{\phase=1}^T 
            \sum_{k\in[v_i]}
            \frac{\Delta_{A^{i,k}}}{b_j K} 
            \mathbbm{1}\left\{A^{i,k}=A_\phase,i\in A^{i,k},T_i(\phase-1)\leq 
            \frac{a_j\cdot \gamma K^2}{\Delta_{A^{i,k}}^2}\left(2\ln\frac{1}{\omega_\mu}+\ln\ln_+\frac{1}{\Delta_{A^{i,v_i}}^2}+\upbdcstt\right)\right\}
            \\
            &\overset{(d)}{\leq}
            \sum_{i\in E\setminus S^\star}
            \sum_{j\in\mathbb{N}} 
            \frac{a_j\cdot \gamma K^2}{b_j K} 
            \left(2\ln\frac{1}{\omega_\mu}+\ln\ln_+\frac{1}{\Delta_{A^{i,v_i}}^2}+\upbdcstt\right)
            \left(
                \frac{1}{\Delta_{A^{i,v_i}}}+\sum_{k=2}^{v_i-1}\Delta_{A^{i,k+1}}
                \left(\frac{1}{\Delta^2_{A^{i,k+1}}}-\frac{1}{\Delta^2_{A^{i,k}}}\right)
            \right)
            \\
            &\leq
            \sum_{i\in E\setminus S^\star}
            \frac{2\abcstt\gamma K}{\Delta_{A^{i,v_i}}} 
            \left(2\ln\frac{1}{\omega_\mu}+\ln\ln_+\frac{1}{\Delta_{A^{i,v_i}}^2}+\upbdcstt\right)
		\end{align}
        where $(a)$ and $b$ make use of Lemma~\ref{lem:upbd_decomposition1} and Lemma~\ref{lem:FtoG}, $(c)$ is obtained by relaxing $\ln\ln_+\frac{1}{\Delta_{A^{i,k}}}$ to $\ln\ln_+\frac{1}{\Delta_{A^{i,v_i}}}$, $(d)$ is obtained by solving the optimization problem.
        
        \noindent
        \textbf{\underline{Case $2$: $A_p\in\Safeset^c\cap\Subopt$}}
        
        For the case where $\omega_\mu\leq\omega_\rmv^\prime$, denote
        $\bomega:=\sqrt{
            \frac{\ln\frac{1}{\omega_\rmv^\prime}}{\ln\frac{1}{\omega_\mu}}
            }
            $. 
        Given an item $i\in E$, assume it is included $v_i$ solutions in $\Safeset^c\cap\Subopt$, we index them according to the decreasing order of their gaps 
        $ \bDelta_{A}:=\max\left\{\bomega\Delta_{A},\frac{\Deltav_{A}}{3}\right\}$
        , i.e. $\{A^{i,k}\}_{k\in [v_i]}$ with $\bDelta_{A^{i,1}}\geq \ldots \geq \bDelta_{A^{i,v_i}}$.
        Denote $ c_{i}:=\max_{k\in[v_i]}\left(\frac{\Delta_{A^{i,k}}}{\bDelta_{{A^{i,k}}}}\right)^2$,
        $k_{i}=\argmax_{k\in[v_i]}\Delta_{A^{i,k}}$ and
        $d_i=\min_{k\in[v_i]}\Delta_{A^{i,k}}$, i.e., the minimum mean gap.

        We have 
        \begin{align}
           &\sum_{\phase=1}^T
            \mathbbm{1}\left\{A_\phase\in\Safeset^c\cap\Subopt\right\}
            \Delta_{A_\phase}\cdot
            \mathbbm{1}\left\{\mathcal{E}\right\}
            \\
            &\overset{(a)}{\leq}
            \sum_{\phase=1}^T
            \mathbbm{1}\left\{A_\phase\in\Safeset^c\cap\Subopt,\mathcal{F}_\phase^\mu,\mathcal{F}_\phase\left(\frac{\Deltav_{A_\phase}}{3},\omega_\rmv^\prime\right)\right\}
            \Delta_{A_\phase}
            \\
            &\overset{(b)}{\leq}
            \sum_{\phase=1}^T
            \mathbbm{1}\left\{A_\phase\in\Safeset^c\cap\Subopt,\mathcal{F}_\phase\left(
            \max\left\{\bomega\Delta_{A_\phase},\frac{\Deltav_{A_\phase}}{3}\right\}
            ,\omega_\rmv^\prime\right)\right\}
            \Delta_{A_\phase}
            \\
            &\overset{(c)}{\leq}
            \sum_{\phase=1}^T   \sum_{j\in\mathbb{N}}
            \mathbbm{1}\left\{A_\phase\in\Safeset^c\cap\Subopt,\mathcal{G}_{j,\phase}
            \left(
            \max\left\{\bomega\Delta_\mu,\frac{\Deltav_{A_\phase}}{3}\right\}
            ,\omega_\rmv^\prime\right)
            \right\}
            \Delta_{A_\phase}
            \\
            &\leq
            \sum_{\phase=1}^T   \sum_{j\in\mathbb{N}} 
            \frac{1}{b_j K} \sum_{i\in A_\phase}
            \mathbbm{1}\left\{A_\phase\in\Safeset^c\cap\Subopt,T_i(\phase-1)\leq 
            m_j\left(
                \bDelta_{A_\phase}
                ,\omega_\rmv^\prime
            \right)
            \right\}
            \Delta_{A_\phase}
            \\
            &=
            \sum_{\phase=1}^T  
            \sum_{j\in\mathbb{N}} 
            \frac{1}{b_j K} 
            \sum_{i\in E\setminus S^\star}
            \sum_{k\in[v_i]}
            \mathbbm{1}\left\{A^{i,k}=A_\phase,i\in A_\phase,T_i(\phase-1)\leq 
            m_j\left(
                \bDelta_{A_\phase}
                ,\omega_\rmv^\prime
            \right)
            \right\}
            \Delta_{A^{i,k}}
            \\
            &\leq
            \sum_{i\in E}
            \sum_{j\in\mathbb{N}} 
            \sum_{\phase=1}^T 
            \sum_{k\in[v_i]}
            \frac{\Delta_{A^{i,k}}}{b_j K} 
            \mathbbm{1}\left\{A^{i,k}=A_\phase,i\in A^{i,k},T_i(\phase-1)\leq 
            \frac{a_j\cdot \gamma K^2}{\bDelta_{A^{i,k}}^2}\left(2\ln\frac{1}{\omega_\rmv^\prime}+\ln\ln_+\frac{1}{\bDelta_{A^{i,k}}^2}+\upbdcstt\right)
            \right\}
            \\
            &\overset{(d)}{\leq}
            \sum_{i\in E}
            \sum_{j\in\mathbb{N}} 
            \sum_{\phase=1}^T 
            \sum_{k\in[v_i]}
            \frac{\Delta_{A^{i,k}}}{b_j K} 
            \mathbbm{1}\left\{A^{i,k}=A_\phase,i\in A^{i,k},T_i(\phase-1)\leq 
            \frac{a_j\cdot \gamma K^2}{\bDelta_{A^{i,k}}^2}\left(2\ln\frac{1}{\omega_\rmv^\prime}+\ln\ln_+\frac{1}{\bDelta_{A^{i,v_i}}^2}+\upbdcstt\right)\right\}
            \label{equ:ScapBoptimization}
            \\
            &\overset{(e)}{\leq}
            \sum_{i\in E}
            \sum_{j\in\mathbb{N}} 
            \frac{a_j\cdot \gamma K^2}{b_j K} 
            \left(2\ln\frac{1}{\omega_\rmv^\prime}+\ln\ln_+\frac{1}{\bDelta_{A^{i,v_i}}^2}+\upbdcstt\right)
            \left(
                \frac{c_{i}}{d_i}+c_{i}\cdot\left(\frac{1}{d_i}-\frac{1}{\Delta_{A^{i,k_i}}}\right)
            \right)
            \\
            &\leq
            \sum_{i\in E}
            \frac{2c_i\cdot \abcstt\gamma K}{d_i} 
            \left(2\ln\frac{1}{\omega_\rmv^\prime}+\ln\ln_+\frac{1}{\bDelta_{A^{i,v_i}}^2}+\upbdcstt\right)
		\end{align}
        where $(a)$ comes from Lemma~\ref{lem:upbd_decomposition1}, $(b)$ results from Lemma~\ref{lem:FFtoF}, $(c)$ is due to Lemma~\ref{lem:FtoG}, $(d)$ is obtained by relaxing $\ln\ln_+\frac{1}{\bDelta_{A^{i,k}}^2}$ to $\ln\ln_+\frac{1}{\bDelta_{A^{i,v_i}}^2}$ and $(e)$ is achieved by solving the optimization problem in \eqref{equ:ScapBoptimization}:
        \begin{align}
            &
            \sum_{\phase=1}^T 
            \sum_{k\in[v_i]}
            \frac{\Delta_{A^{i,k}}}{b_j K} 
            \mathbbm{1}\left\{A^{i,k}=A_\phase,i\in A^{i,k},T_i(\phase-1)\leq 
            \frac{a_j\cdot \gamma K^2}{\bDelta_{A^{i,k}}^2}\left(2\ln\frac{1}{\omega_\rmv^\prime}+\ln\ln_+\frac{1}{\bDelta_{A^{i,v_i}}^2}+\upbdcstt\right)\right\}
            \\
            &
            \leq
            \frac{a_j\cdot \gamma K^2}{b_j K} \left(2\ln\frac{1}{\omega_\rmv^\prime}+\ln\ln_+\frac{1}{\bDelta_{A^{i,v_i}}^2}+\upbdcstt\right)\cdot 
            \left(
                \frac{c_{i}}{d_i}+\int_{d_i}^{\Delta_{A^{i,k_i}}}\frac{c_i}{x^2}\,\mathrm{dx}
            \right)
            \\
            &=
            \frac{a_j\cdot \gamma K^2}{b_j K} 
            \left(2\ln\frac{1}{\omega_\rmv^\prime}+\ln\ln_+\frac{1}{\bDelta_{A^{i,v_i}}^2}+\upbdcstt\right)
            \left(
                \frac{c_{i}}{d_i}+c_{i}\cdot\left(\frac{1}{d_i}-\frac{1}{\Delta_{A^{i,k_i}}}\right)
            \right)
        \end{align}

        \textbf{In conclusion}, 
        for $i\in E\setminus S^\star$, denote 
        \begin{align}
            \Delta_{i,\Safeset\cap\Subopt,\min}:=\min_{S\ni i,S\in\Safeset\cap\Subopt}\Delta_S
        \end{align} 
        For $i\in E$, denote 
        \begin{align}
            \Delta_{i,\Safeset^c\cap\Subopt,\min}&:=\min_{S\ni i,S\in\Safeset^c\cap\Subopt}\Delta_S
            \\
            \bDelta_{i,\Safeset^c\cap\Subopt}^\prime&:=\min_{S\ni i,S\in\Safeset^c\cap\Subopt}\max
            \{\bomega\Delta_S,\Deltav_S/3\}
            \\
            c_i&:=\max_{S\ni i,S\in\Safeset^c\cap\Subopt}\left(\frac{\Delta_S}{\max\{\bomega\Delta_S,\Deltav_S/3\}}\right)^2
        \end{align}
        The regret due to suboptimality can be upper bounded by
        \begin{align}
            \rmR_1(T)&\leq
            \sum_{i\in E\setminus S^\star}
            \frac{2\abcstt\gamma K}{\Delta_{i,\Safeset\cap\Subopt,\min}} 
            \left(2\ln\frac{1}{\omega_\mu}+\ln\ln_+\frac{1}{\Delta_{i,\Safeset\cap\Subopt,\min}^2}+\upbdcstt\right)
            \\
            &\quad
            +
            \sum_{i\in E}
            \frac{2c_i\cdot \abcstt\gamma K}{\Delta_{i,\Safeset^c\cap\Subopt,\min}} 
            \left(2\ln\frac{1}{\omega_\rmv^\prime}+\ln\ln_+\frac{1}{(\bDelta_{i,\Safeset^c\cap\Subopt}^\prime)^2}+\upbdcstt\right)
        \end{align}

    \end{proof}
    
    \subsection{Regret due to safeness-checking}\label{sec:safe_reg}
    We firstly introduce two more technical lemmas in order to upper bound the regret.
    We characterize the number of sub-solutions $n_\phase$ in Algorithm~\ref{alg:split} by the following lemma.
	\begin{lem}\label{lem:np}
		At any phase $\phase$, we have $\mathbb{P}[n_\phase\leq Q]=1$. Furthermore, 
		if $U_{A_\phase}^\rmv(\phase-1)\in((m-1)\bsigma^2,m\bsigma^2]$ for some $m\in\mathbb{N}$, then $m \leq n_\phase\leq 2m-1.$
	\end{lem}
     	\begin{proof}[Proof of Lemma~\ref{lem:np}]
    		Recall that an absolutely safe solution $S$ is safe w.p. $1$ and $S\in\mathcal{S}_{\phase-1}$, thus $|A_{\phase,r}|\geq q,\forall r\in[n_\phase-1]$. 
    		If 
    		$n_\phase> Q$, i.e. $n_\phase\geq Q+1$, then
    		\begin{align}
    			K\geq |A_\phase|=\sum_{r=1}^{n_\phase}|A_{\phase,r}|> \sum_{r=1}^{Q}|A_{\phase,r}|\geq Q\cdot q\geq K
    		\end{align}
    		which constitutes a contradiction. Therefore, we have $\mathbb{P}[n_\phase\leq Q]=1$.
    
    		If $U_{A_\phase}^\rmv(\phase-1)\in((m-1)\bsigma^2,m\bsigma^2]$ for some $m\in\mathbb{N}_+$, we sequentially define $\{j_1,j_2,\ldots\}\subset [|A_\phase|]$ as follows: denote $j_0:=0$ and let $j_1$ be the integer such that 
    		$$\sum_{s=1}^{j_1-1}U_{i_s}^\rmv(\phase-1)\leq \bsigma^2\quad\mbox{and}\quad\sum_{s=1}^{j_1}U_{i_s}^\rmv(\phase-1)> \bsigma^2.$$
    		Then let $j_2>j_1$ be the integer such that
    		\begin{align}
    			\sum_{s=j_1+1}^{j_2-1}U_{i_s}^\rmv(\phase-1)\leq \bsigma^2\quad&\mbox{and}\quad\sum_{s=j_1+1}^{j_2}U_{i_s}^\rmv(\phase-1)> \bsigma^2.\\
    			&\vdots
    		\end{align}
    		The last integer $j_k=|A_\phase|$ satisfies
    		$$0<\sum_{s=j_{k-1}+1}^{j_k}U_{i_s}^\rmv(\phase-1)\leq \bsigma^2.$$
    		If $k\geq m+1$, we must have 
    		\begin{align}
    			U_{A_t}^\rmv(\phase-1)&=\sum_{l=1}^k\sum_{s=j_{l-1}+1}^{j_l}U_{i_s}^\rmv(\phase-1)\\
    			&\geq 
    			\sum_{l=1}^{m}\sum_{s=j_{l-1}+1}^{j_l}U_{i_s}^\rmv(\phase-1)>m\bsigma^2
    		\end{align}
    		which contradicts with $U_{A_\phase}^\rmv(\phase-1)\in((m-1)\bsigma^2,m\bsigma^2]$. Hence, $k\leq m$.
    		Then we construct the sub-solutions by 
    		\begin{align}
    			 &A_{\phase,l}=\{i_{j_{l-1}+1},\ldots,i_{j_l-1}\},\quad\forall l\in[k-1]\\
    			 &\mbox{and}\quad A_{\phase,k}=\{i_{j_{k-1}+1},\ldots,i_{j_k}\}.
    		\end{align}
    		There are $k-1$ items $\{i_{j_l}:l\in[k-1]\}$ left which will compose at most $k-1$ additional sub-solutions. 
    		In conclusion, we need at most $2m-1$ sub-solutions, i.e., $n_\phase\leq 2m-1$. Obviously, we need at least $m$ sub-solutions since $\bsigma^2>\sigma^2$. So $m\leq n_\phase\leq 2m-1$.
    	\end{proof}
    \begin{remark}
        Indexing the items in Line~$3$ of {\sc Greedy-Split} can be done arbitrarily, i.e., it does not require any specific order of the items. 
            As such, {\sc Greedy-Split} is an efficient greedy algorithm. We note that finding the optimal order that leads to the minimum number of sub-solutions $n_p$ is a combinatorial problem which is generally hard to solve.
    \end{remark}     
    Due to the fact that the upper confidence of any solution $S$ satisfies $U_S^\rmv(\phase)\leq Q\bsigma^2$, thus $n_\phase$ can at most be $2Q-1$.
    
        Lemma~\ref{lem:upbd_decomposition} implies the key to upper bound the regret due to safeness-checking is to upper bound $\sum_{r=1}^{Q-1}\mathbbm{1}\{\mathcal{U}_p(r)\}$ over the horizon $T$.
    From the definition of $\mathcal{U}_\phase(r):=\{U_{A_\phase}^\rmv(\phase-1)>r\bsigma^2\}$, for $r_1,r_2\in[Q-1]$ with $r_1>r_2$, event $\mathcal{U}_p(r_1)$ indicates event $\mathcal{U}_p(r_2)$. Thus in order to upper bound $\sum_{r=1}^{Q-1}\mathbbm{1}\{\mathcal{U}_p(r)\}$, it suffices to upper bound $\mathbbm{1}\{\mathcal{U}_p(r)\}$ for $r\in[Q-1]$. To be more specific, given $\addsln\in[Q-1]$, in order to compute the maximum number of times $\sum_{r=1}^{Q-1}\mathbbm{1}\{\mathcal{U}_\phase(r)\}\geq \addsln$, we only need to compute the maximum number of times event $\mathcal{U}_p(\addsln)$ occurs.

    In the following lemma, we show a necessary condition (in terms of event $\mathcal{F}_\phase(x,\omega)$) for event $\mathcal{U}_\phase(r)$.
	\begin{lem}\label{lem:upbd_decomposition2}
        On the event 
        $\mathcal{E} $, 
        \begin{itemize}
            \item for $A_p\in\Safeset$:
                \begin{align}
                    \mathbbm{1}\left\{\mathcal{U}_p(r)\right\}
                    \leq
                    \mathbbm{1}\left\{\mathcal{F}_p\left(\frac{(r-1)\bsigma^2+\Deltav_{A_p}}{3},\omega_\rmv\right)\right\}
                \end{align}
            \item for $A_p\in\Safeset^c$
                \begin{align}
                \mathbbm{1}\left\{\mathcal{U}_p(r)\right\}
                \leq
                \mathbbm{1}\left\{\mathcal{F}_p\left(\frac{(r-1)\bsigma^2-\Deltav_{A_p}}{3},\omega_\rmv\right)\right\}
            \end{align}
        \end{itemize}      
	\end{lem}
        \begin{proof}[Proof of Lemma~\ref{lem:upbd_decomposition2}]
             The proof is straightforward
            \begin{align}
                &U_{A_\phase}^\rmv(\phase-1)>r\bsigma^2
                \\
                \overset{(a)}{\Rightarrow}\quad
                &\hat{\sigma}_{A_\phase}^2(\phase-1)+\sum_{i\in A_\phase}\beta_\rmu(T_i(\phase-1))>r\bsigma^2
                \\
                \overset{(b)}{\Rightarrow}\quad
                &\sigma_{A_\phase}^2+2\sum_{i\in A_\phase}\beta_\rmu(T_i(\phase-1))>r\bsigma^2
                \\
                \Rightarrow\quad
                &2\sum_{i\in A_\phase}3\cdot \mathrm{lil}(T_i(\phase-1),\omega_\rmv)>(r-1)\bsigma^2+(\bsigma^2-\sigma_{A_\phase}^2)
            \end{align}
            where $(a)$ is due to the definition of the confidence bounds for a solution \eqref{equ:Aestimates}, and $(b)$ utilizes the event $\bigcap_{i\in A_{\phase}}\mathcal{E}_{i,T_i(\phase-1)}$.
            For $A_\phase\in\Safeset$, the above event is equivalent to
            $\mathcal{F}_\phase\left(\frac{(r-1)\bsigma^2+\Deltav_{A_\phase}}{3},\omega_\rmv\right)$; and for $A_p\in\Safeset^c$, it is equivalent to
            $\mathcal{F}_\phase\left(\frac{(r-1)\bsigma^2-\Deltav_{A_\phase}}{3},\omega_\rmv\right)$.
        \end{proof}
    The above lemma and Lemma~\ref{lem:upbd_decomposition1} upper bound the components in $\rmR_2(\Tp)$ by the $\mathcal{F}$ events.

    We are now ready to bound the regret due to safeness checking.

    \lemsafereg*
        \begin{proof}[Proof of Lemma~\ref{lem:safereg}]
    From Lemma~\ref{lem:upbd_decomposition}, the high-probability regret due to safeness-checking is
    \begin{align}
            \rmR_2(\Tp)
            &=
            \mu^\star \sum_{\phase=1}^{\Tp}
            \left[2\sum_{r=1}^{Q-1}\mathbbm{1}\left\{\mathcal{U}_p(r)\right\}\right]
            \\
            &=
            2\mu^\star \sum_{r=1}^{Q-1}
            \sum_{\phase=1}^{\Tp}
            \mathbbm{1}\left\{\mathcal{U}_p(r)\right\}
    \end{align} 
    In the following, given $r\in[Q-1]$, we are going to upper bound $ \sum_{\phase=1}^{\Tp}
            \mathbbm{1}\{\mathcal{U}_p(r)\}$ conditional on event $\mathcal{E}$.
    When $U_{A_\phase}^\rmv(p-1)>r\bsigma^2$ holds, there are at most $2r+1$ solutions being chosen at phase $p$, i.e. $n_p\leq 2r+1$, according to Lemma~\ref{lem:np}. 
    Therefore, we are also deriving an upper bound for number of phases in which there are at most $2r+1$ sub-solutions being sampled.

    

    The proof scheme is planned as follows:
    in \textbf{Step 1},
    we decompose the event $\mathcal{U}_p(r) $ into $4$ events according to where $A_\phase$ lies, i.e. (1) $A_\phase=S^\star$; (2) $\mathcal{S}\cap\mathcal{B}$; (3) $\mathcal{R}$ and (4) $\mathcal{S}^c\cap\mathcal{B}$. We will upper bound the regret under each of these cases.
    
    In  \textbf{Step 2}, we apply Lemma~\ref{lem:upbd_decomposition1} to upper bound the number of times a solution $A$ can be selected via the events $\mathcal{F}_\phase^\mu$ and $\mathcal{F}_\phase(x,\omega)$. Because there will be multiple $\mathcal{F}_\phase^\mu$ and $\mathcal{F}_\phase(x,\omega)$ events in the indicator function, Lemma~\ref{lem:FFtoF} will be adopted to merge them into one event.
    After that, Lemma~\ref{lem:FtoG} is utilized to bridge the number of times a solution is identified to the number of times of an item is sampled. At the end of this step, we conclude the number of times $\mathbbm{1}\{\mathcal{U}_p(r)\}$ occurs under the four cases.

    In \textbf{Step 3}, we upper bound $\sum_{r=1}^{Q-1}
            \sum_{\phase=1}^{\Tp}
            \mathbbm{1}\left\{\mathcal{U}_p(r)\right\}$
    based on the results from Step 2.

    \noindent
    \textbf{\underline{Step 1:}}
    We decompose  $ \sum_{\phase=1}^{\Tp}
            \mathbbm{1}\{\mathcal{U}_p(r)\}$ conditional on event $\mathcal{E}$ into four parts:
    \begin{align}
            &\sum_{\phase=1}^{\Tp}
            \mathbbm{1}\{\mathcal{U}_p(r)\}
            \\
            &\leq
            \sum_{\phase=1}^{\Tp}
            \mathbbm{1}\left\{ U_{A_\phase}^\rmv(p-1)>r\bsigma^2 \right\}
            \\
            &=
            \sum_{\phase=1}^{\Tp}
            \Big(\mathbbm{1}\left\{A_\phase=S^\star\right\}+\mathbbm{1}\left\{A_\phase\in\Safeset\cap\Subopt\right\}
            +\mathbbm{1}\left\{A_\phase\in\Riskset\right\}+\mathbbm{1}\left\{A_\phase\in\Safeset^c\cap\Subopt\right\}\Big)
               \cdot\mathbbm{1}\left\{ U_{A_\phase}^\rmv(p-1)>r\bsigma^2 \right\}
            \\
            &=
            \sum_{\phase=1}^{\Tp}
            \mathbbm{1}\left\{A_\phase=S^\star\right\}
                \mathbbm{1}\left\{ U_{A_\phase}^\rmv(p-1)>r\bsigma^2 \right\}
            +\sum_{\phase=1}^{\Tp}\mathbbm{1}\left\{A_\phase\in\Safeset\cap\Subopt\right\}
                \mathbbm{1}\left\{ U_{A_\phase}^\rmv(p-1)>r\bsigma^2 \right\}
            \\
            &\quad
            +\sum_{\phase=1}^{\Tp}\mathbbm{1}\left\{A_\phase\in\Riskset\right\}
                \mathbbm{1}\left\{ U_{A_\phase}^\rmv(p-1)>r\bsigma^2 \right\}
            +\sum_{\phase=1}^{\Tp}\mathbbm{1}\left\{A_\phase\in\Safeset^c\cap\Subopt\right\}\Big)
                \mathbbm{1}\left\{ U_{A_\phase}^\rmv(p-1)>r\bsigma^2 \right\}
            \\
    \end{align} 

    \noindent
    \textbf{\underline{Step 2:}} For each of the scenarios, we firstly upper bound the regret by the ``$\mathcal{F}$ events'' and they can be further bounded in terms of ``$\mathcal{G}$ events''.

    \noindent
    \underline{Case 1: $A_\phase=S^\star$}
    \begin{align}
        &\sum_{\phase=1}^{\Tp}
            \mathbbm{1}\left\{A_\phase=S^\star\right\}
                \mathbbm{1}\left\{ U_{A_\phase}^\rmv(\phase-1)>r\bsigma^2 \right\}
        \\
        &\overset{(a)}{\leq}
        \sum_{\phase=1}^{\Tp}
            \mathbbm{1}\left\{A_\phase=S^\star\right\}
                \mathbbm{1}\left\{\mathcal{F}_\phase\left(\frac{(r-1)\bsigma^2+\Deltav_{A_\phase}}{3},\omega_\rmv\right)\right\}
        \\
        &\overset{(b)}{\leq}
        \sum_{\phase=1}^{\Tp}
        \sum_{j\in\mathbb{N}}
            \mathbbm{1}\left\{A_\phase=S^\star\right\}
                \mathbbm{1}\left\{\mathcal{G}_{j,\phase}\left(\frac{(r-1)\bsigma^2+\Deltav_{A_\phase}}{3},\omega_\rmv\right)\right\}
        \\
        &\overset{(c)}{\leq}
        \sum_{\phase=1}^{\Tp}
        \sum_{j\in\mathbb{N}}\mathbbm{1}\left\{A_\phase=S^\star\right\}
        \frac{1}{b_j K} \sum_{i\in A_\phase}
                \mathbbm{1}\left\{T_i(\phase-1)\leq m_j\left(\frac{(r-1)\bsigma^2+\Deltav_{A_\phase}}{3},\omega_\rmv\right)\right\}
        \\
        &=
        \sum_{i\in S^\star}
        \sum_{j\in\mathbb{N}}
        \sum_{\phase=1}^{\Tp}
        \frac{1}{b_j K} 
                \mathbbm{1}\left\{T_i(\phase-1)\leq m_j\left(\frac{(r-1)\bsigma^2+\Deltav_{S^\star}}{3},\omega_\rmv\right)\right\}
        \\
        &\leq
        \sum_{i\in S^\star}
        \sum_{j\in\mathbb{N}}
        \frac{1}{b_j K} 
         m_j\left(\frac{(r-1)\bsigma^2+\Deltav_{S^\star}}{3},\omega_\rmv\right)
         \\
         &=
        \sum_{i\in S^\star}
        \sum_{j\in\mathbb{N}}
        \frac{a_j\cdot \gamma K^2}{b_j K} 
        \frac{9}{((r-1)\bsigma^2+\Deltav_{S^\star})^2}\left(2\ln\frac{1}{\omega_\rmv}+\ln\ln_+\frac{1}{((r-1)\bsigma^2+\Deltav_{S^\star})^2}+\upbdcstt\right)
         \\
         &\leq
        \sum_{i\in S^\star}
        \abcstt\cdot
        \frac{9\cdot\gamma K}{((r-1)\bsigma^2+\Deltav_{S^\star})^2}
        \left(2\ln\frac{1}{\omega_\rmv}+\ln\ln_+\frac{1}{{\Deltav_{S^\star}}^2}+\upbdcstt\right)
    \end{align}
    where $(a)$ utilizes Lemma~\ref{lem:upbd_decomposition2}, $(b)$ makes use of Lemma~\ref{lem:FtoG} and $(c)$ follows the definition of $\mathcal{G}_{j,\phase}$.
    For simplicity, we denote 
    \begin{align}
        g_{S^\star}(r,\Deltav_{S^\star})=
        \abcstt\cdot
        \frac{9\cdot\gamma K}{((r-1)\bsigma^2+\Deltav_{S^\star})^2}
        \left(2\ln\frac{1}{\omega_\rmv}+\ln\ln_+\frac{1}{{\Deltav_{S^\star}}^2}+\upbdcstt\right).
    \end{align}
    Thus
    \begin{align}
        \sum_{\phase=1}^{\Tp}
        \mathbbm{1}\left\{A_\phase=S^\star\right\}
        \mathbbm{1}\left\{ U_{A_\phase}^\rmv(\phase-1)>r\bsigma^2 \right\}
            \leq
            \sum_{i\in S^\star}g_{S^\star}(r,\Deltav_{S^\star})
    \end{align}
    \noindent
    \underline{Case 2: $A_\phase\in\Safeset\cap\Subopt$}
    
    Under this case, there will be a comparison between $\omega_\mu$ and $\omega_\rmv$ thus we denote
    $\tomega=\sqrt{
        \frac{\ln\frac{1}{\omega_\mu}}{\ln\frac{1}{\omega_\rmv}}}$.
    For $i\in E$, denote $\bDelta_{i,\Safeset\cap\Subopt}:=\min_{S\ni i,S\in\Safeset\cap\Subopt}\max\left\{\frac{\Delta_{S}}{\sqrt{\ln(1/\omega_\mu)}},\frac{\Deltav_{S}}{3\sqrt{\ln(1/\omega_\rmv)}}\right\}$ which is achieved by solution $S_{i,\Safeset\cap\Subopt}$,
    and assume $\bDelta_{i,\Safeset\cap\Subopt}\in(\frac{r_i\bsigma^2}{3\sqrt{\ln(1/\omega_\rmv)}},\frac{(r_i+1)\bsigma^2}{3\sqrt{\ln(1/\omega_\rmv)}}]$ for some $r_i\in\mathbb{N}$.
    
    \textbf{Scenario 1: $\omega_\mu\geq\omega_\rmv$}
    
    We firstly deal with the case where $\omega_\mu\geq\omega_\rmv$, i.e., $\tomega\leq 1$. 
    We have
    \begin{align}
        &\sum_{\phase=1}^{\Tp}
            \mathbbm{1}\left\{A_\phase\in\Safeset\cap\Subopt\right\}
                \mathbbm{1}\left\{ U_{A_\phase}^\rmv(\phase-1)>r\bsigma^2 \right\}
        \\
        &\overset{(a)}{\leq}
        \sum_{\phase=1}^{\Tp}
            \mathbbm{1}\left\{A_\phase\in\Safeset\cap\Subopt\right\}
                \mathbbm{1}\left\{\mathcal{F}_\phase^\mu,\mathcal{F}_\phase\left(\frac{(r-1)\bsigma^2+\Deltav_{A_\phase}}{3},\omega_\rmv\right)\right\}
        \\
        &\overset{(b)}{\leq}
        \sum_{\phase=1}^{\Tp}
            \mathbbm{1}\left\{A_\phase\in\Safeset\cap\Subopt\right\}
                \mathbbm{1}\left\{\mathcal{F}_\phase\left(\Delta_{A_\phase},\omega_\mu\right),\mathcal{F}_\phase\left(\frac{(r-1)\bsigma^2+\Deltav_{A_\phase}}{3},\omega_\rmv\right)\right\}
        \\
        &\overset{(c)}{\leq}
        \sum_{\phase=1}^{\Tp}
            \mathbbm{1}\left\{A_\phase\in\Safeset\cap\Subopt\right\}
                \mathbbm{1}\left\{\mathcal{F}_\phase\left(\max\left\{\Delta_{A_\phase},\tomega\frac{(r-1)\bsigma^2+\Deltav_{A_\phase}}{3}\right\},\omega_\mu\right)\right\}
        \\
        &\overset{(d)}{\leq}
        \sum_{\phase=1}^{\Tp}
        \sum_{j\in\mathbb{N}}
            \mathbbm{1}\left\{A_\phase\in\Safeset\cap\Subopt\right\}
                \mathbbm{1}\left\{\mathcal{G}_{j,\phase}\left(\max\left\{\Delta_{A_\phase},\tomega\frac{(r-1)\bsigma^2+\Deltav_{A_\phase}}{3}\right\},\omega_\mu\right)\right\}
        \\
        &\overset{(e)}{\leq}
       \sum_{\phase=1}^{\Tp}
        \sum_{j\in\mathbb{N}}\mathbbm{1}\left\{A_\phase\in\Safeset\cap\Subopt\right\}
        \frac{1}{b_j K} \sum_{i\in A_\phase}
                \mathbbm{1}\left\{T_i(\phase-1)\leq m_j\left(\max\left\{\Delta_{A_\phase},\tomega\frac{(r-1)\bsigma^2+\Deltav_{A_\phase}}{3}\right\},\omega_\mu\right)\right\}
        \\
        &=
        \sum_{i\in E}
        \sum_{j\in\mathbb{N}}\frac{1}{b_j K}
        \sum_{\phase=1}^{\Tp}
        \mathbbm{1}\left\{A_\phase\in\Safeset\cap\Subopt\right\}
                \mathbbm{1}\left\{i\in A_\phase,T_i(\phase-1)\leq m_j\left(\max\left\{\Delta_{A_\phase},\tomega\frac{(r-1)\bsigma^2+\Deltav_{A_\phase}}{3}\right\},\omega_\mu\right)\right\}
        \label{equ:discussiononri}
        \end{align}
        where $(a)$ utilizes Lemma~\ref{lem:upbd_decomposition2},
    $(b)$ and $(c)$ make use of Lemma~\ref{lem:FFtoF}, 
    $(d)$ is due to Lemma~\ref{lem:FtoG} 
    and $(e)$ follows the definition of $\mathcal{G}_{j,\phase}$.

    Given $i\in E$,
    
    (1) if $r\geq r_i+2$, for any $S\in\Safeset\cap\Subopt$ that contains item $i$,
        \begin{align}
            \max\left\{\Delta_{S},\tomega\frac{(r-1)\bsigma^2+\Deltav_{S}}{3}\right\}
        \geq
        \tomega\frac{(r-1)\bsigma^2}{3}
        \geq
        \tomega\frac{(r_i+1)\bsigma^2}{3}
        \geq
        \sqrt{\ln\frac{1}{\omega_\mu}}\cdot\bDelta_{i,\Safeset\cap\Subopt}.
        \end{align}
        Thus,
    \begin{align}
        &\sum_{j\in\mathbb{N}}\frac{1}{b_j K}
        \sum_{\phase=1}^{\Tp}
        \mathbbm{1}\left\{A_\phase\in\Safeset\cap\Subopt\right\}
                \mathbbm{1}\left\{i\in A_\phase,T_i(\phase-1)\leq m_j\left(\max\left\{\Delta_{A_\phase},\tomega\frac{(r-1)\bsigma^2+\Deltav_{A_\phase}}{3}\right\},\omega_\mu\right)\right\}
        \\
        &
        \sum_{j\in\mathbb{N}}\frac{1}{b_j K}
        \sum_{\phase=1}^{\Tp}
        \mathbbm{1}\left\{A_\phase\in\Safeset\cap\Subopt\right\}
                \mathbbm{1}\left\{i\in A_\phase,T_i(\phase-1)\leq m_j\left(\tomega\frac{(r-1)\bsigma^2}{3},\omega_\mu\right)\right\}
        \\
        &=
        \sum_{j\in\mathbb{N}}
        \frac{1}{b_j K} 
                m_j\left(\tomega\frac{(r-1)\bsigma^2}{3},\omega_\mu\right)
        \\
        &=
        \sum_{j\in\mathbb{N}}
        \frac{a_j\cdot \gamma K^2}{b_j K} 
        \frac{9}{((r-1)\bsigma^2)^2}\left(2\ln\frac{1}{\omega_\rmv}+\frac{1}{\tomega^2}\ln\ln_+\frac{9}{(\tomega(r-1)\bsigma^2)^2}+\frac{1}{\tomega^2}\upbdcstt\right)
         \\
         &\leq
        \abcstt\cdot
        \frac{9\cdot\gamma K}{((r-1)\bsigma^2)^2}
        \left(2\ln\frac{1}{\omega_\rmv}+\frac{1}{\tomega^2}\ln\ln_+\frac{1}{\ln(1/\omega_\mu)\bDelta_{i,\Safeset\cap\Subopt}^2}+\frac{1}{\tomega^2}\upbdcstt\right)
        \\
        &=
        \abcstt\cdot
        \frac{\gamma K}{(\frac{(r-1)\bsigma^2}{3\sqrt{\ln(1/\omega_\rmv)}})^2}
        \left(2+\frac{1}{\ln(1/\omega_\mu)}\ln\ln_+\frac{1}{\ln(1/\omega_\mu)\bDelta_{i,\Safeset\cap\Subopt}^2}+\frac{1}{\ln(1/\omega_\mu)}\upbdcstt\right)
    \end{align}
    (2) if $r\leq r_i+1$, for any $S\in\Safeset\cap\Subopt$ that contains item $i$
        \begin{align}
            \max\left\{\Delta_{S},\tomega\frac{(r-1)\bsigma^2+\Deltav_{S}}{3}\right\}
        \geq
        \max\left\{\Delta_{S},\tomega\frac{\Deltav_{S}}{3}\right\}
        \geq
        \sqrt{\ln\frac{1}{\omega_\mu}}\cdot\bDelta_{i,\Safeset\cap\Subopt}.
        \end{align}
    Thus,
    \begin{align}
        &
        \sum_{j\in\mathbb{N}}\frac{1}{b_j K}
        \sum_{\phase=1}^{\Tp}
        \mathbbm{1}\left\{A_\phase\in\Safeset\cap\Subopt\right\}
                \mathbbm{1}\left\{i\in A_\phase,T_i(\phase-1)\leq m_j\left(\max\left\{\Delta_{A_\phase},\tomega\frac{(r-1)\bsigma^2+\Deltav_{A_\phase}}{3}\right\},\omega_\mu\right)\right\}
        \\
        &\leq
        \sum_{j\in\mathbb{N}}\frac{1}{b_j K}
        \sum_{\phase=1}^{\Tp}
        \mathbbm{1}\left\{A_\phase\in\Safeset\cap\Subopt\right\}
                \mathbbm{1}\left\{i\in A_\phase,T_i(\phase-1)\leq m_j\left(\sqrt{\ln\frac{1}{\omega_\mu}}\cdot\bDelta_{i,\Safeset\cap\Subopt},\omega_\mu\right)\right\}
        \\
        & \overset{(a)} {\leq}
        \sum_{j\in\mathbb{N}}
        \frac{1}{b_j K} 
                m_j\left(\sqrt{\ln\frac{1}{\omega_\mu}}\cdot\bDelta_{i,\Safeset\cap\Subopt},\omega_\mu\right)
        \\
     &=
        \sum_{j\in\mathbb{N}}
        \frac{a_j\cdot \gamma K^2}{b_j K} 
        \frac{1}{(\sqrt{\ln\frac{1}{\omega_\mu}}\cdot\bDelta_{i,\Safeset\cap\Subopt})^2}\left(2\ln\frac{1}{\omega_\mu}+\ln\ln_+\frac{1}{(\sqrt{\ln\frac{1}{\omega_\mu}}\cdot\bDelta_{i,\Safeset\cap\Subopt})^2}+\upbdcstt\right)
         \\
         &\leq
        \abcstt\cdot
        \frac{\gamma K}{\bDelta_{i,\Safeset\cap\Subopt}^2}
        \left(2
        +\frac{1}{\ln(1/\omega_\mu)}\ln\ln_+\frac{1}{\ln(1/\omega_\mu)\bDelta_{i,\Safeset\cap\Subopt}^2}+\frac{1}{\ln(1/\omega_\mu)}\upbdcstt\right)
    \end{align}
    where $(a)$ is achieved by sampling $A_{i,\Safeset\cap\Subopt}$.
    Therefore, if we denote $g_{\Safeset\cap\Subopt,1}(r,\bDelta_{i,\Safeset\cap\Subopt}):= $
        \begin{align}
            \left\{
            \begin{aligned}
                &
                \frac{\abcstt\gamma K}{(\frac{(r-1)\bsigma^2}{3\sqrt{\ln(1/\omega_\rmv)}})^2}
                \left(2+\frac{1}{\ln(1/\omega_\mu)}\ln\ln_+\frac{1}{\ln(1/\omega_\mu)\bDelta_{i,\Safeset\cap\Subopt}^2}+\frac{1}{\ln(1/\omega_\mu)}\upbdcstt\right), 
                r\geq \left\lfloor\frac{3\sqrt{\ln(1/\omega_\rmv)}\cdot\bDelta_{i,\Safeset\cap\Subopt} }{\bsigma^2}\right\rfloor+2
                \\
                &
                \frac{\abcstt\gamma K}{\bDelta_{i,\Safeset\cap\Subopt}^2}
                \left(2
                +\frac{1}{\ln(1/\omega_\mu)}\ln\ln_+\frac{1}{\ln(1/\omega_\mu)\bDelta_{i,\Safeset\cap\Subopt}^2}+\frac{1}{\ln(1/\omega_\mu)}\upbdcstt\right),
                r\leq \left\lfloor\frac{3\sqrt{\ln(1/\omega_\rmv)}\cdot\bDelta_{i,\Safeset\cap\Subopt} }{\bsigma^2}\right\rfloor+1
            \end{aligned}
            \right.
        \end{align}
        then \eqref{equ:discussiononri} can be upper bounded by
        $$\sum_{i\in E}g_{\Safeset\cap\Subopt,1}(r,\bDelta_{i,\Safeset\cap\Subopt})$$

    \textbf{Scenario 2: $\omega_\mu\leq\omega_\rmv$}
    
    For the case where $\omega_\mu\leq\omega_\rmv$, i.e., $\tomega\geq 1$. 
    We have
    \begin{align}
        &\sum_{\phase=1}^{\Tp}
            \mathbbm{1}\left\{A_\phase\in\Safeset\cap\Subopt\right\}
                \mathbbm{1}\left\{ U_{A_\phase}^\rmv(\phase-1)>r\bsigma^2 \right\}
        \\
        &\overset{(a)}{\leq}
        \sum_{\phase=1}^{\Tp}
            \mathbbm{1}\left\{A_\phase\in\Safeset\cap\Subopt\right\}
                \mathbbm{1}\left\{\mathcal{F}_\phase^\mu,\mathcal{F}_\phase\left(\frac{(r-1)\bsigma^2+\Deltav_{A_\phase}}{3},\omega_\rmv\right)\right\}
        \\
        &\overset{(b)}{\leq}
        \sum_{\phase=1}^{\Tp}
            \mathbbm{1}\left\{A_\phase\in\Safeset\cap\Subopt\right\}
                \mathbbm{1}\left\{\mathcal{F}_\phase\left(\Delta_{A_\phase},\omega_\mu\right),\mathcal{F}_\phase\left(\frac{(r-1)\bsigma^2+\Deltav_{A_\phase}}{3},\omega_\rmv\right)\right\}
        \\
        &\overset{(c)}{\leq}
        \sum_{\phase=1}^{\Tp}
            \mathbbm{1}\left\{A_\phase\in\Safeset\cap\Subopt\right\}
                \mathbbm{1}\left\{\mathcal{F}_\phase\left(\max\left\{\frac{\Delta_{A_\phase}}{\tomega},\frac{(r-1)\bsigma^2+\Deltav_{A_\phase}}{3}\right\},\omega_\rmv\right)\right\}
        \\
        &\overset{(d)}{\leq}
        \sum_{\phase=1}^{\Tp}
        \sum_{j\in\mathbb{N}}
            \mathbbm{1}\left\{A_\phase\in\Safeset\cap\Subopt\right\}
                \mathbbm{1}\left\{\mathcal{G}_{j,\phase}\left(\max\left\{\frac{\Delta_{A_\phase}}{\tomega},\frac{(r-1)\bsigma^2+\Deltav_{A_\phase}}{3}\right\},\omega_\rmv\right)\right\}
        \\
        &\overset{(e)}{\leq}
       \sum_{\phase=1}^{\Tp}
        \sum_{j\in\mathbb{N}}\mathbbm{1}\left\{A_\phase\in\Safeset\cap\Subopt\right\}
        \frac{1}{b_j K} \sum_{i\in A_\phase}
                \mathbbm{1}\left\{T_i(\phase-1)\leq m_j\left(\max\left\{\frac{\Delta_{A_\phase}}{\tomega},\frac{(r-1)\bsigma^2+\Deltav_{A_\phase}}{3}\right\},\omega_\rmv\right)\right\}
        \\
        &=
        \sum_{i\in E}
        \sum_{j\in\mathbb{N}}\frac{1}{b_j K}
        \sum_{\phase=1}^{\Tp}
        \mathbbm{1}\left\{A_\phase\in\Safeset\cap\Subopt\right\}
                \mathbbm{1}\left\{i\in A_\phase,T_i(\phase-1)\leq m_j\left(\max\left\{\frac{\Delta_{A_\phase}}{\tomega},\frac{(r-1)\bsigma^2+\Deltav_{A_\phase}}{3}\right\},\omega_\rmv\right)\right\}
        \label{equ:discussiononri2}
        \end{align}

        Given $i\in E$,
        
        (1) if $r\geq r_i+2$, for any $S\in\Safeset\cap\Subopt$ that contains item $i$,
            \begin{align}
                \max\left\{\frac{\Delta_{S}}{\tomega},\frac{(r-1)\bsigma^2+\Deltav_{S}}{3}\right\}
            \geq
            \frac{(r-1)\bsigma^2}{3}
            \geq
            \frac{(r_i+1)\bsigma^2}{3}
            \geq
            \sqrt{\ln\frac{1}{\omega_\rmv}}\cdot\bDelta_{i,\Safeset\cap\Subopt}.
            \end{align}
            Thus, 
    \begin{align}
        &\sum_{j\in\mathbb{N}}\frac{1}{b_j K}
        \sum_{\phase=1}^{\Tp}
        \mathbbm{1}\left\{A_\phase\in\Safeset\cap\Subopt\right\}
                \mathbbm{1}\left\{i\in A_\phase,T_i(\phase-1)\leq m_j\left(\max\left\{\frac{\Delta_{A_\phase}}{\tomega},\frac{(r-1)\bsigma^2+\Deltav_{A_\phase}}{3}\right\},\omega_\rmv\right)\right\}
        \\
        &
        \leq\sum_{j\in\mathbb{N}}\frac{1}{b_j K}
        \sum_{\phase=1}^{\Tp}
        \mathbbm{1}\left\{A_\phase\in\Safeset\cap\Subopt\right\}
                \mathbbm{1}\left\{i\in A_\phase,T_i(\phase-1)\leq m_j\left(\frac{(r-1)\bsigma^2}{3},\omega_\rmv\right)\right\}
        \\
        &=
        \sum_{j\in\mathbb{N}}
        \frac{1}{b_j K} 
                m_j\left(\frac{(r-1)\bsigma^2}{3},\omega_\rmv\right)
        \\
     &=
        \sum_{j\in\mathbb{N}}
        \frac{a_j\cdot \gamma K^2}{b_j K} 
        \frac{9}{((r-1)\bsigma^2)^2}\left(2\ln\frac{1}{\omega_\rmv}+\ln\ln_+\frac{9}{((r-1)\bsigma^2)^2}+\upbdcstt\right)
         \\
         &\leq
        \abcstt\cdot
        \frac{9\cdot\gamma K}{((r-1)\bsigma^2)^2}
        \left(2\ln\frac{1}{\omega_\rmv}+\ln\ln_+\frac{1}{\ln(1/\omega_\rmv)\bDelta_{i,\Safeset\cap\Subopt}^2}+\upbdcstt\right)
        \\
        &=
        \abcstt\cdot
        \frac{\gamma K}{(\frac{(r-1)\bsigma^2}{3\sqrt{\ln(1/\omega_\rmv)}})^2}
        \left(2+\frac{1}{\ln(1/\omega_\rmv)}\ln\ln_+\frac{1}{\ln(1/\omega_\rmv)\bDelta_{i,\Safeset\cap\Subopt}^2}+\frac{1}{\ln(1/\omega_\rmv)}\upbdcstt\right)
    \end{align}

    (2) if $r\leq r_i+1$,  for any $S\in\Safeset\cap\Subopt$ that contains item $i$,
    \begin{align}
         \max\left\{\frac{\Delta_{S}}{\tomega},\frac{(r-1)\bsigma^2+\Deltav_{S}}{3}\right\}
    \geq
    \max\left\{\frac{\Delta_{S}}{\tomega},\frac{\Deltav_{S}}{3}\right\}
    \geq
    \sqrt{\ln\frac{1}{\omega_\rmv}}\cdot\bDelta_{i,\Safeset\cap\Subopt}.
    \end{align}
    Thus,
    \begin{align}
        &\sum_{j\in\mathbb{N}}\frac{1}{b_j K}
        \sum_{\phase=1}^{\Tp}
        \mathbbm{1}\left\{A_\phase\in\Safeset\cap\Subopt\right\}
                \mathbbm{1}\left\{i\in A_\phase,T_i(\phase-1)\leq m_j\left(\max\left\{\frac{\Delta_{A_\phase}}{\tomega},\frac{(r-1)\bsigma^2+\Deltav_{A_\phase}}{3}\right\},\omega_\rmv\right)\right\}
        \\
        &\leq
        \sum_{j\in\mathbb{N}}\frac{1}{b_j K}
        \sum_{\phase=1}^{\Tp}
        \mathbbm{1}\left\{A_\phase\in\Safeset\cap\Subopt\right\}
                \mathbbm{1}\left\{i\in A_\phase,T_i(\phase-1)\leq m_j\left(\sqrt{\ln\frac{1}{\omega_\rmv}}\cdot\bDelta_{i,\Safeset\cap\Subopt},\omega_\rmv\right)\right\}
        \\
        &=
        \sum_{j\in\mathbb{N}}
        \frac{1}{b_j K} 
                m_j\left(\sqrt{\ln\frac{1}{\omega_\rmv}}\cdot\bDelta_{i,\Safeset\cap\Subopt},\omega_\rmv\right)
        \\
     &=
        \sum_{j\in\mathbb{N}}
        \frac{a_j\cdot \gamma K^2}{b_j K} 
        \frac{1}{(\sqrt{\ln\frac{1}{\omega_\rmv}}\cdot\bDelta_{i,\Safeset\cap\Subopt})^2}\left(2\ln\frac{1}{\omega_\rmv}+\ln\ln_+\frac{1}{(\sqrt{\ln\frac{1}{\omega_\rmv}}\cdot\bDelta_{i,\Safeset\cap\Subopt})^2}+\upbdcstt\right)
         \\
         &\leq
        \abcstt\cdot
        \frac{\gamma K}{\bDelta_{i,\Safeset\cap\Subopt}^2}
        \left(2
        +\frac{1}{\ln(1/\omega_\rmv)}\ln\ln_+\frac{1}{\ln(1/\omega_\rmv)\bDelta_{i,\Safeset\cap\Subopt}^2}+\frac{1}{\ln(1/\omega_\rmv)}\upbdcstt\right)
    \end{align}
    Therefore, if we denote $g_{\Safeset\cap\Subopt,2}(r,\bDelta_{i,\Safeset\cap\Subopt}):=$
        \begin{align}
            \left\{
            \begin{aligned}
                &
                \frac{\abcstt\gamma K}{(\frac{(r-1)\bsigma^2}{3\sqrt{\ln(1/\omega_\rmv)}})^2}
                \left(2+\frac{1}{\ln(1/\omega_\rmv)}\ln\ln_+\frac{1}{\ln(1/\omega_\rmv)\bDelta_{i,\Safeset\cap\Subopt}^2}+\frac{1}{\ln(1/\omega_\rmv)}\upbdcstt\right), 
                r\geq \left\lfloor\frac{3\sqrt{\ln(1/\omega_\rmv)}\cdot\bDelta_{i,\Safeset\cap\Subopt} }{\bsigma^2}\right\rfloor+2
                \\
                &
                \frac{\abcstt\gamma K}{\bDelta_{i,\Safeset\cap\Subopt}^2}
                \left(2
                +\frac{1}{\ln(1/\omega_\rmv)}\ln\ln_+\frac{1}{\ln(1/\omega_\rmv)\bDelta_{i,\Safeset\cap\Subopt}^2}+\frac{1}{\ln(1/\omega_\rmv)}\upbdcstt\right),
                r\leq \left\lfloor\frac{3\sqrt{\ln(1/\omega_\rmv)}\cdot\bDelta_{i,\Safeset\cap\Subopt} }{\bsigma^2}\right\rfloor+1
            \end{aligned}\right.
        \end{align}
        then \eqref{equ:discussiononri2} can be upper bounded by
        $$\sum_{i\in E}g_{\Safeset\cap\Subopt,2}(r,\bDelta_{i,\Safeset\cap\Subopt}).$$

    In conclusion, for $i\in E$, we denote $\bDelta_{i,\Safeset\cap\Subopt}:=\min_{S\ni i,S\in\Safeset\cap\Subopt}\max\left\{\frac{\Delta_{S}}{\sqrt{\ln(1/\omega_\mu)}},\frac{\Deltav_{A}}{3\sqrt{\ln(1/\omega_\rmv)}}\right\}$, 
    $\omega_{\mu\rmv}:=\max\{\omega_\mu,\omega_\rmv\}$ and $g_{\Safeset\cap\Subopt}(r,\bDelta_{i,\Safeset\cap\Subopt}):=$
   \begin{align}
            \left\{
            \begin{aligned}
                &
                \frac{\abcstt\gamma K}{(\frac{(r-1)\bsigma^2}{3\sqrt{\ln(1/\omega_\rmv)}})^2}
                \left(2+\frac{1}{\ln(1/\omega_{\mu\rmv})}\left(\ln\ln_+\frac{1}{\ln(1/\omega_{\mu\rmv})\bDelta_{i,\Safeset\cap\Subopt}^2}+\upbdcstt\right)\right), 
                r\geq \left\lfloor\frac{3\sqrt{\ln(1/\omega_\rmv)}\cdot\bDelta_{i,\Safeset\cap\Subopt} }{\bsigma^2}\right\rfloor+2
                \\
                &
                \frac{\abcstt\gamma K}{\bDelta_{i,\Safeset\cap\Subopt}^2}
                \left(2
                +\frac{1}{\ln(1/\omega_{\mu\rmv})}\left(\ln\ln_+\frac{1}{\ln(1/\omega_{\mu\rmv})\bDelta_{i,\Safeset\cap\Subopt}^2}+\upbdcstt\right)\right),
                r\leq \left\lfloor\frac{3\sqrt{\ln(1/\omega_\rmv)}\cdot\bDelta_{i,\Safeset\cap\Subopt} }{\bsigma^2}\right\rfloor+1
            \end{aligned}\right.
    \end{align}
    then 
    \begin{align}
         &\sum_{\phase=1}^{\Tp}
            \mathbbm{1}\left\{A_\phase\in\Safeset\cap\Subopt\right\}
                \mathbbm{1}\left\{ U_{A_\phase}^\rmv(\phase-1)>r\bsigma^2 \right\}
        \leq\sum_{i\in E}g_{\Safeset\cap\Subopt}(r,\bDelta_{i,\Safeset\cap\Subopt})
    \end{align}

    \noindent
    \underline{Case 3: $A_\phase\in\Riskset$}

     Under this case, there will be a comparison between $\omega_\rmv$ and $\omega_\rmv^\prime$ thus we denote
    $\bomega=\sqrt{
        \frac{\ln\frac{1}{\omega_\rmv^\prime}}{\ln\frac{1}{\omega_\rmv}}}$ and $\omega_\rmsum:=\sqrt{\ln\frac{1}{\omega_\rmv^\prime}}+\sqrt{\ln\frac{1}{\omega_\rmv}}$.

    For $i\in E$, denote $\Deltav_{i,\Riskset}:=\min_{S\ni i,S\in\Riskset}\Deltav_{S}$ and assume
    $\Deltav_{i,\Riskset}\in (r_i\frac{\bomega}{\bomega+1}\bsigma^2,(r_i+1)\frac{\bomega}{\bomega+1}\bsigma^2]$.

    \textbf{Scenario 1: $\omega_\rmv\leq \omega_\rmv^\prime$}
        For the case $\omega_\rmv\leq \omega_\rmv^\prime$, we have $\bomega\leq 1$.



    \begin{align}
        &\sum_{\phase=1}^{\Tp}\mathbbm{1}\left\{A_\phase\in\Riskset\right\}
                \mathbbm{1}\left\{ U_{A_\phase}^\rmv(p-1)>r\bsigma^2 \right\}
        \\
        &\overset{(a)}{\leq}
        \sum_{\phase=1}^{\Tp}
            \mathbbm{1}\left\{A_\phase\in\Riskset\right\}
                \mathbbm{1}\left\{
                \mathcal{F}_\phase\left(\frac{\Deltav_{A_\phase}}{3},\omega_\rmv^\prime\right)
                ,
                \mathcal{F}_\phase\left(\frac{(r-1)\bsigma^2-\Deltav_{A_\phase}}{3},\omega_\rmv\right)
                \right\}
        \\
        &\overset{(b)}{\leq}
        \sum_{\phase=1}^{\Tp}
            \mathbbm{1}\left\{A_\phase\in\Riskset\right\}
                \mathbbm{1}\left\{
                \mathcal{F}_\phase\left(
                \max
                \left\{\frac{\Deltav_{A_\phase}}{3},\bomega\cdot\frac{(r-1)\bsigma^2-\Deltav_{A_\phase}}{3}
                \right\},
                \omega_\rmv^\prime\right)
                \right\}
        \\
        &\overset{(c)}{\leq}
        \sum_{\phase=1}^{\Tp}
        \sum_{j\in\mathbb{N}}
            \mathbbm{1}\left\{A_\phase\in\Riskset\right\}
                \mathbbm{1}\left\{\mathcal{G}_{j,\phase}\left(
                \max\left\{\frac{\Deltav_{A_\phase}}{3},\bomega\cdot\frac{(r-1)\bsigma^2-\Deltav_{A_\phase}}{3}
                \right\},
                \omega_\rmv^\prime\right)\right\}
        \\
        &\overset{(d)}{\leq}
       \sum_{\phase=1}^{\Tp}
        \sum_{j\in\mathbb{N}}\mathbbm{1}\left\{A_\phase\in\Riskset\right\}
        \frac{1}{b_j K} \sum_{i\in A_\phase}
                \mathbbm{1}\left\{T_i(\phase-1)\leq m_j\left(\max\left\{\frac{\Deltav_{A_\phase}}{3},\bomega\cdot\frac{(r-1)\bsigma^2-\Deltav_{A_\phase}}{3}
                \right\},\omega_\rmv^\prime\right)\right\}
        \\
        &=
        \sum_{i\in E}
        \sum_{j\in\mathbb{N}}\frac{1}{b_j K}
        \sum_{\phase=1}^{\Tp}
        \mathbbm{1}\left\{A_\phase\in\Riskset\right\}
                \mathbbm{1}\left\{i\in A_\phase,T_i(\phase-1)\leq  m_j\left(\max\left\{\frac{\Deltav_{A_\phase}}{3},\bomega\cdot\frac{(r-1)\bsigma^2-\Deltav_{A_\phase}}{3}
                \right\},\omega_\rmv^\prime\right)\right\}
        \label{equ:discussionOnriRisk}
        \end{align}
        where $(a)$ utilizes Lemma~\ref{lem:upbd_decomposition2},
        $(b)$ makes use of Lemma~\ref{lem:FFtoF}, 
        $(c)$ is due to Lemma~\ref{lem:FtoG} 
        and $(d)$ follows the definition of $\mathcal{G}_{j,\phase}$.

        Given $i\in E$,

        (1) if $r\geq r_i+2$, for any $S\in \Riskset$ that contains item $i$,
        \begin{align}
            \max\left\{\frac{\Deltav_{S}}{3},\bomega\cdot\frac{(r-1)\bsigma^2-\Deltav_{S}}{3}\right\}
            \geq 
            \max\left\{\frac{\Deltav_{S}}{3},\frac{\bomega}{1+\bomega}\cdot\frac{(r-1)\bsigma^2}{3}\right\}
            \geq 
            \frac{\bomega}{1+\bomega}\cdot\frac{(r-1)\bsigma^2}{3}
            \geq
            \frac{\Deltav_{i,\Riskset}}{3}
        \end{align}
        where the first inequality uses the fact that for $x,y\in\mathbb{R}_+$ and $z\in[0,1]$, , $\max\{x,y\}\geq \max\{x,xz+y(1-z)\}$. We take $x=\frac{\Deltav_{A_\phase}}{3},y=\bomega\cdot\frac{(r-1)\bsigma^2-\Deltav_{A_\phase}}{3}$ and $z=\frac{\bomega}{1+\bomega}$.
        Thus, 
        \begin{align}
            &\sum_{j\in\mathbb{N}}\frac{1}{b_j K}
                \sum_{\phase=1}^{\Tp}
                \mathbbm{1}\left\{A_\phase\in\Riskset\right\}
                \mathbbm{1}\left\{i\in A_\phase,T_i(\phase-1)\leq  m_j\left(\max\left\{\frac{\Deltav_{A_\phase}}{3},\bomega\cdot\frac{(r-1)\bsigma^2-\Deltav_{A_\phase}}{3}
                \right\},\omega_\rmv^\prime\right)\right\}
            \\
            &\leq
            \sum_{j\in\mathbb{N}}\frac{1}{b_j K}
                \sum_{\phase=1}^{\Tp}
                \mathbbm{1}\left\{A_\phase\in\Riskset\right\}
                \mathbbm{1}\left\{i\in A_\phase,T_i(\phase-1)\leq  m_j\left(\frac{\bomega}{1+\bomega}\cdot\frac{(r-1)\bsigma^2}{3},\omega_\rmv^\prime\right)\right\}
            \\
            &=
            \sum_{j\in\mathbb{N}}
            \frac{1}{b_j K} 
                    m_j\left(\frac{\bomega}{1+\bomega}\cdot\frac{(r-1)\bsigma^2}{3},\omega_\rmv^\prime\right)
            \\
            &=
            \sum_{j\in\mathbb{N}}
            \frac{a_j\cdot \gamma K^2}{b_j K} 
            \frac{1}{(\frac{\bomega}{1+\bomega}\cdot\frac{(r-1)\bsigma^2}{3})^2}\left(2\ln\frac{1}{\omega_\rmv^\prime}+\ln\ln_+\frac{1}{(\frac{\bomega}{1+\bomega}\cdot\frac{(r-1)\bsigma^2}{3})^2}+\upbdcstt\right)
             \\
             &\leq
            \frac{\abcstt\gamma K}{(\frac{\bomega}{1+\bomega}\cdot\frac{(r-1)\bsigma^2}{3})^2}\left(2\ln\frac{1}{\omega_\rmv^\prime}+\ln\ln_+\frac{1}{(\frac{\bomega}{1+\bomega}\cdot\frac{(r-1)\bsigma^2}{3})^2}+\upbdcstt\right)
            \\
             &\leq
            \frac{\abcstt\gamma K}{(\frac{(r-1)\bsigma^2}
            {3\omega_{\rmsum}})^2}\left(2+\frac{1}{\ln(1/\omega_\rmv^\prime)}\left(\ln\ln_+\frac{1}{\ln(1/\omega_\rmv^\prime)(\frac{\Deltav_{i,\Riskset}}{3\sqrt{\ln(1/\omega_\rmv^\prime)}})^2}+\upbdcstt\right)\right)
        \end{align}

         (2) if $r\leq r_i+1$, for any $S\in \Riskset$ that contains item $i$,
        \begin{align}
            \max\left\{\frac{\Deltav_{S}}{3},\bomega\cdot\frac{(r-1)\bsigma^2-\Deltav_{S}}{3}\right\}
            \geq 
            \max\left\{\frac{\Deltav_{S}}{3},\frac{\bomega}{1+\bomega}\cdot\frac{(r-1)\bsigma^2}{3}\right\}
            \geq
            \frac{\Deltav_{i,\Riskset}}{3}
        \end{align}
        Thus, 
        \begin{align}
            &\sum_{j\in\mathbb{N}}\frac{1}{b_j K}
                \sum_{\phase=1}^{\Tp}
                \mathbbm{1}\left\{A_\phase\in\Riskset\right\}
                \mathbbm{1}\left\{i\in A_\phase,T_i(\phase-1)\leq  m_j\left(\max\left\{\frac{\Deltav_{A_\phase}}{3},\bomega\cdot\frac{(r-1)\bsigma^2-\Deltav_{A_\phase}}{3}
                \right\},\omega_\rmv^\prime\right)\right\}
            \\
            &\leq
            \sum_{j\in\mathbb{N}}\frac{1}{b_j K}
                \sum_{\phase=1}^{\Tp}
                \mathbbm{1}\left\{A_\phase\in\Riskset\right\}
                \mathbbm{1}\left\{i\in A_\phase,T_i(\phase-1)\leq  m_j\left(\frac{\Deltav_{i,\Riskset}}{3},\omega_\rmv^\prime\right)\right\}
            \\
            &=
            \sum_{j\in\mathbb{N}}
            \frac{1}{b_j K} 
                    m_j\left(\frac{\Deltav_{i,\Riskset}}{3},\omega_\rmv^\prime\right)
            \\
            &=
            \sum_{j\in\mathbb{N}}
            \frac{a_j\cdot \gamma K^2}{b_j K} 
            \frac{1}{(\frac{\Deltav_{i,\Riskset}}{3})^2}\left(2\ln\frac{1}{\omega_\rmv^\prime}+\ln\ln_+\frac{1}{(\frac{\Deltav_{i,\Riskset}}{3})^2}+\upbdcstt\right)
             \\
             &\leq
            \frac{\abcstt\gamma K}{(\frac{\Deltav_{i,\Riskset}}{3})^2}\left(2\ln\frac{1}{\omega_\rmv^\prime}+\ln\ln_+\frac{1}{(\frac{\Deltav_{i,\Riskset}}{3})^2}+\upbdcstt\right)
            \\
            &=
            \frac{\abcstt\gamma K}{(\frac{\Deltav_{i,\Riskset}}{3\sqrt{\ln(1/\omega_\rmv^\prime)}})^2}
            \left(2+\frac{1}{\ln(1/\omega_\rmv^\prime)}\left(\ln\ln_+\frac{1}{\ln(1/\omega_\rmv^\prime)(\frac{\Deltav_{i,\Riskset}}{3\sqrt{\ln(1/\omega_\rmv^\prime)}})^2}+\upbdcstt\right)\right)
        \end{align}
        Therefore, if we denote 
        \begin{align}
            g_{\Riskset,1}(r,\Deltav_{i,\Riskset}):=\left\{
            \begin{aligned}
                &
                \frac{\abcstt\gamma K}{(\frac{(r-1)\bsigma^2}
                {3\omega_{\rmsum}})^2}\left(2+\frac{1}{\ln(1/\omega_\rmv^\prime)}\left(\ln\ln_+\frac{1}{\ln(1/\omega_\rmv^\prime)(\frac{\Deltav_{i,\Riskset}}{3\sqrt{\ln(1/\omega_\rmv^\prime)}})^2}+\upbdcstt\right)\right), 
                r\geq \left\lfloor\frac{\omega_\rmsum\cdot\Deltav_{i,\Riskset} }{\sqrt{\ln(1/\omega_\rmv^\prime)}\bsigma^2}\right\rfloor+2
                \\
                &
                \frac{\abcstt\gamma K}{(\frac{\Deltav_{i,\Riskset}}{3\sqrt{\ln(1/\omega_\rmv^\prime)}})^2}
                \left(2+\frac{1}{\ln(1/\omega_\rmv^\prime)}\left(\ln\ln_+\frac{1}{\ln(1/\omega_\rmv^\prime)(\frac{\Deltav_{i,\Riskset}}{3\sqrt{\ln(1/\omega_\rmv^\prime)}})^2}+\upbdcstt\right)\right),
                r\leq \left\lfloor\frac{\omega_\rmsum\cdot\Deltav_{i,\Riskset} }{\sqrt{\ln(1/\omega_\rmv^\prime)}\bsigma^2}\right\rfloor+1
            \end{aligned}\right.
        \end{align}
        then \eqref{equ:discussionOnriRisk} can be upper bounded by
        $$\sum_{i\in E}g_{\Riskset,1}(r,\Deltav_{i,\Riskset}).$$
        
    \textbf{Scenario 2: $\omega_\rmv\geq \omega_\rmv^\prime$}
        For the case $\omega_\rmv\geq \omega_\rmv^\prime$, we have $\bomega\geq 1$.

        \begin{align}
        &\sum_{\phase=1}^{\Tp}\mathbbm{1}\left\{A_\phase\in\Riskset\right\}
                \mathbbm{1}\left\{ U_{A_\phase}^\rmv(p-1)>r\bsigma^2 \right\}
        \\
        &\overset{(a)}{\leq}
        \sum_{\phase=1}^{\Tp}
            \mathbbm{1}\left\{A_\phase\in\Riskset\right\}
                \mathbbm{1}\left\{
                \mathcal{F}_\phase\left(\frac{\Deltav_{A_\phase}}{3},\omega_\rmv^\prime\right)
                ,
                \mathcal{F}_\phase\left(\frac{(r-1)\bsigma^2-\Deltav_{A_\phase}}{3},\omega_\rmv\right)
                \right\}
        \\
        &\overset{(b)}{\leq}
        \sum_{\phase=1}^{\Tp}
            \mathbbm{1}\left\{A_\phase\in\Riskset\right\}
                \mathbbm{1}\left\{
                \mathcal{F}_\phase\left(
                \max\left\{\frac{1}{\bomega}\frac{\Deltav_{A_\phase}}{3},\frac{(r-1)\bsigma^2-\Deltav_{A_\phase}}{3}
                \right\},
                \omega_\rmv\right)
                \right\}
        \\
        &\overset{(c)}{\leq}
        \sum_{\phase=1}^{\Tp}
        \sum_{j\in\mathbb{N}}
            \mathbbm{1}\left\{A_\phase\in\Riskset\right\}
                \mathbbm{1}\left\{\mathcal{G}_{j,\phase}\left(
                \max\left\{\frac{1}{\bomega}\frac{\Deltav_{A_\phase}}{3},\frac{(r-1)\bsigma^2-\Deltav_{A_\phase}}{3}
                \right\},
                \omega_\rmv\right)\right\}
        \\
        &\overset{(d)}{\leq}
       \sum_{\phase=1}^{\Tp}
        \sum_{j\in\mathbb{N}}\mathbbm{1}\left\{A_\phase\in\Riskset\right\}
        \frac{1}{b_j K} \sum_{i\in A_\phase}
                \mathbbm{1}\left\{T_i(\phase-1)\leq m_j\left(
                \max\left\{\frac{1}{\bomega}\frac{\Deltav_{A_\phase}}{3},\frac{(r-1)\bsigma^2-\Deltav_{A_\phase}}{3}
                \right\},
                \omega_\rmv\right)\right\}
        \\
        &=
        \sum_{i\in E}
        \sum_{j\in\mathbb{N}}\frac{1}{b_j K}
        \sum_{\phase=1}^{\Tp}
        \mathbbm{1}\left\{A_\phase\in\Riskset\right\}
                \mathbbm{1}\left\{i\in A_\phase,T_i(\phase-1)\leq  m_j\left(
                \max\left\{\frac{1}{\bomega}\frac{\Deltav_{A_\phase}}{3},\frac{(r-1)\bsigma^2-\Deltav_{A_\phase}}{3}
                \right\},
                \omega_\rmv\right)\right\}
        \label{equ:discussionOnriRisk2}
        \end{align}
        where $(a)$ utilizes Lemma~\ref{lem:upbd_decomposition2},
        $(b)$ makes use of Lemma~\ref{lem:FFtoF}, 
        $(c)$ is due to Lemma~\ref{lem:FtoG} 
        and $(d)$ follows the definition of $\mathcal{G}_{j,\phase}$.

        Given $i\in E$,

        (1) if $r\geq r_i+2$, for any $S\in \Riskset$ that contains item $i$,
        \begin{align}
           \max\left\{\frac{1}{\bomega}\frac{\Deltav_{S}}{3},\frac{(r-1)\bsigma^2-\Deltav_{S}}{3}
                \right\}
            \geq 
            \max\left\{\frac{1}{\bomega}\frac{\Deltav_{S}}{3},\frac{\bomega}{1+\bomega}\cdot\frac{(r-1)\bsigma^2}{3}\right\}
            \geq 
            \frac{1}{1+\bomega}\cdot\frac{(r-1)\bsigma^2}{3}
            \geq
            \frac{1}{\bomega}\frac{\Deltav_{i,\Riskset}}{3}
        \end{align}
        where the first inequality uses the fact that for $x,y\in\mathbb{R}_+$ and $z\in[0,1]$, , $\max\{x,y\}\geq \max\{x,xz+y(1-z)\}$. We take $x=\frac{1}{\bomega}\frac{\Deltav_{A_\phase}}{3},y=\frac{(r-1)\bsigma^2-\Deltav_{A_\phase}}{3}$ and $z=\frac{\bomega}{1+\bomega}$.
        Thus, 
        \begin{align}
            &
            \sum_{j\in\mathbb{N}}\frac{1}{b_j K}
            \sum_{\phase=1}^{\Tp}
            \mathbbm{1}\left\{A_\phase\in\Riskset\right\}
            \mathbbm{1}\left\{i\in A_\phase,T_i(\phase-1)\leq  m_j\left(
            \max\left\{\frac{1}{\bomega}\frac{\Deltav_{A_\phase}}{3},\frac{(r-1)\bsigma^2-\Deltav_{A_\phase}}{3}
            \right\},
            \omega_\rmv\right)\right\}
            \\
            &\leq
            \sum_{j\in\mathbb{N}}\frac{1}{b_j K}
                \sum_{\phase=1}^{\Tp}
                \mathbbm{1}\left\{A_\phase\in\Riskset\right\}
                \mathbbm{1}\left\{i\in A_\phase,T_i(\phase-1)\leq  m_j\left(\frac{1}{1+\bomega}\cdot\frac{(r-1)\bsigma^2}{3},\omega_\rmv\right)\right\}
            \\
            &=
            \sum_{j\in\mathbb{N}}
            \frac{1}{b_j K} 
                    m_j\left(\frac{1}{1+\bomega}\cdot\frac{(r-1)\bsigma^2}{3},\omega_\rmv\right)
            \\
            &=
            \sum_{j\in\mathbb{N}}
            \frac{a_j\cdot \gamma K^2}{b_j K} 
            \frac{1}{(\frac{1}{1+\bomega}\cdot\frac{(r-1)\bsigma^2}{3})^2}\left(2\ln\frac{1}{\omega_\rmv}+\ln\ln_+\frac{1}{(\frac{1}{1+\bomega}\cdot\frac{(r-1)\bsigma^2}{3})^2}+\upbdcstt\right)
             \\
             &\leq
            \frac{\abcstt\gamma K}{(\frac{1}{1+\bomega}\cdot\frac{(r-1)\bsigma^2}{3})^2}\left(2\ln\frac{1}{\omega_\rmv}+\ln\ln_+\frac{1}{(\frac{1}{1+\bomega}\cdot\frac{(r-1)\bsigma^2}{3})^2}+\upbdcstt\right)
             \\
             &\leq
             \frac{\abcstt\gamma K}{(\frac{(r-1)\bsigma^2}
            {3\omega_{\rmsum}})^2}\left(2+\frac{1}{\ln(1/\omega_\rmv)}\left(\ln\ln_+\frac{1}{\ln(1/\omega_\rmv)(\frac{\Deltav_{i,\Riskset}}{3\sqrt{\ln(1/\omega_\rmv^\prime)}})^2}+\upbdcstt\right)\right)
        \end{align}

         (2) if $r\leq r_i+1$, for any $S\in \Riskset$ that contains item $i$,
        \begin{align}
           \max\left\{\frac{1}{\bomega}\frac{\Deltav_{S}}{3},\frac{(r-1)\bsigma^2-\Deltav_{S}}{3}
                \right\}
            \geq 
            \max\left\{\frac{1}{\bomega}\frac{\Deltav_{A}}{3},\frac{\bomega}{1+\bomega}\cdot\frac{(r-1)\bsigma^2}{3}\right\}
            \geq
            \frac{1}{\bomega}\frac{\Deltav_{i,\Riskset}}{3}
        \end{align}
        Thus, 
        \begin{align}
            &\sum_{j\in\mathbb{N}}\frac{1}{b_j K}
            \sum_{\phase=1}^{\Tp}
            \mathbbm{1}\left\{A_\phase\in\Riskset\right\}
            \mathbbm{1}\left\{i\in A_\phase,T_i(\phase-1)\leq  m_j\left(
            \max\left\{\frac{1}{\bomega}\frac{\Deltav_{A_\phase}}{3},\frac{(r-1)\bsigma^2-\Deltav_{A_\phase}}{3}
            \right\},
            \omega_\rmv\right)\right\}
            \\
            &\leq
            \sum_{j\in\mathbb{N}}\frac{1}{b_j K}
                \sum_{\phase=1}^{\Tp}
                \mathbbm{1}\left\{A_\phase\in\Riskset\right\}
                \mathbbm{1}\left\{i\in A_\phase,T_i(\phase-1)\leq  m_j\left(\frac{1}{\bomega}\frac{\Deltav_{i,\Riskset}}{3},\omega_\rmv\right)\right\}
            \\
            &=
            \sum_{j\in\mathbb{N}}
            \frac{1}{b_j K} 
                    m_j\left(\frac{1}{\bomega}\frac{\Deltav_{i,\Riskset}}{3},\omega_\rmv\right)
            \\
            &=
            \sum_{j\in\mathbb{N}}
            \frac{a_j\cdot \gamma K^2}{b_j K} 
            \frac{1}{(\frac{1}{\bomega}\frac{\Deltav_{i,\Riskset}}{3})^2}\left(2\ln\frac{1}{\omega_\rmv}+\ln\ln_+\frac{1}{(\frac{1}{\bomega}\frac{\Deltav_{i,\Riskset}}{3})^2}+\upbdcstt\right)
             \\
             &\leq
            \frac{\abcstt\gamma K}{(\frac{1}{\bomega}\frac{\Deltav_{i,\Riskset}}{3})^2}\left(2\ln\frac{1}{\omega_\rmv}+\ln\ln_+\frac{1}{(\frac{1}{\bomega}\frac{\Deltav_{i,\Riskset}}{3})^2}+\upbdcstt\right)
            \\
            &=
               \frac{\abcstt\gamma K}{(\frac{\Deltav_{i,\Riskset}}{3\sqrt{\ln(1/\omega_\rmv^\prime)}})^2}
            \left(2+\frac{1}{\ln(1/\omega_\rmv)}\left(\ln\ln_+\frac{1}{\ln(1/\omega_\rmv)(\frac{\Deltav_{i,\Riskset}}{3\sqrt{\ln(1/\omega_\rmv^\prime)}})^2}+\upbdcstt\right)\right)
        \end{align}
        Therefore, if we denote 
        \begin{align}
            g_{\Riskset,2}(r,\Deltav_{i,\Riskset}):=\left\{
            \begin{aligned}
                &
                \frac{\abcstt\gamma K}{(\frac{(r-1)\bsigma^2}
                {3\omega_{\rmsum}})^2}\left(2+\frac{1}{\ln(1/\omega_\rmv)}\left(\ln\ln_+\frac{1}{\ln(1/\omega_\rmv)(\frac{\Deltav_{i,\Riskset}}{3\sqrt{\ln(1/\omega_\rmv^\prime)}})^2}+\upbdcstt\right)\right), 
                r\geq \left\lfloor\frac{\omega_\rmsum\cdot\Deltav_{i,\Riskset} }{\sqrt{\ln(1/\omega_\rmv^\prime)}\bsigma^2}\right\rfloor+2
                \\
                &
                \frac{\abcstt\gamma K}{(\frac{\Deltav_{i,\Riskset}}{3\sqrt{\ln(1/\omega_\rmv^\prime)}})^2}
                \left(2+\frac{1}{\ln(1/\omega_\rmv)}\left(\ln\ln_+\frac{1}{\ln(1/\omega_\rmv)(\frac{\Deltav_{i,\Riskset}}{3\sqrt{\ln(1/\omega_\rmv^\prime)}})^2}+\upbdcstt\right)\right),
                r\leq \left\lfloor\frac{\omega_\rmsum\cdot\Deltav_{i,\Riskset} }{\sqrt{\ln(1/\omega_\rmv^\prime)}\bsigma^2}\right\rfloor+1
            \end{aligned}\right.
        \end{align}
        then \eqref{equ:discussionOnriRisk2} can be upper bounded by
        $$\sum_{i\in E}g_{\Riskset,2}(r,\Deltav_{i,\Riskset}).$$

    In conclusion, for $i\in E$, we denote $\Deltav_{i,\Riskset}:=\min_{S\ni i,S\in\Riskset}\Deltav_S$, 
    $\omega_{\rmv\rmv^\prime}:=\max\{\omega_\rmv,\omega_\rmv^\prime\}$, $\omega_{\rmsum}:=\sqrt{\ln\frac{1}{\omega_\rmv^\prime}}+\sqrt{\ln\frac{1}{\omega_\rmv}}$and
   \begin{align}
            g_{\Riskset}(r,\Deltav_{i,\Riskset}):=\left\{
            \begin{aligned}
                &
                \frac{\abcstt\gamma K}{(\frac{(r-1)\bsigma^2}
                {3\omega_{\rmsum}})^2}\left(2+\frac{1}{\ln(1/\omega_{\rmv\rmv^\prime})}\left(\ln\ln_+\frac{1}{\ln(1/\omega_{\rmv\rmv^\prime})(\frac{\Deltav_{i,\Riskset}}{3\sqrt{\ln(1/\omega_\rmv^\prime)}})^2}+\upbdcstt\right)\right), 
                r\geq \left\lfloor\frac{\omega_\rmsum\cdot\Deltav_{i,\Riskset} }{\sqrt{\ln(1/\omega_\rmv^\prime)}\bsigma^2}\right\rfloor+2
                \\
                &
                \frac{\abcstt\gamma K}{(\frac{\Deltav_{i,\Riskset}}{3\sqrt{\ln(1/\omega_\rmv^\prime)}})^2}
                \left(2+\frac{1}{\ln(1/\omega_{\rmv\rmv^\prime})}\left(\ln\ln_+\frac{1}{\ln(1/\omega_{\rmv\rmv^\prime})(\frac{\Deltav_{i,\Riskset}}{3\sqrt{\ln(1/\omega_\rmv^\prime)}})^2}+\upbdcstt\right)\right),
                r\leq \left\lfloor\frac{\omega_\rmsum\cdot\Deltav_{i,\Riskset} }{\sqrt{\ln(1/\omega_\rmv^\prime)}\bsigma^2}\right\rfloor+1
            \end{aligned}\right.
    \end{align}
    then 
    \begin{align}
         &\sum_{\phase=1}^{\Tp}
            \mathbbm{1}\left\{A_\phase\in\Riskset\right\}
                \mathbbm{1}\left\{ U_{A_\phase}^\rmv(\phase-1)>r\bsigma^2 \right\}
        \leq\sum_{i\in E}g_{\Riskset}(r,\Deltav_{i,\Riskset})
    \end{align}

    \noindent
    \underline{Case 4: $A_\phase\in\Safeset^c\cap\Subopt$}

     Under this case, there will be a comparison among $\omega_\mu,\omega_\rmv$ and $\omega_\rmv^\prime$ thus we denote
     $\omega_{\max}=\max\{\omega_\mu,\omega_\rmv,\omega_\rmv^\prime\}$ and 
     $\omega_1=\sqrt{\frac{\ln\frac{1}{\omega_{\max}}}{\ln\frac{1}{\omega_\mu}}},
        \omega_2=\sqrt{\frac{\ln\frac{1}{\omega_{\max}}}{\ln\frac{1}{\omega_\rmv}}},
        \omega_3=\sqrt{\frac{\ln\frac{1}{\omega_{\max}}}{\ln\frac{1}{\omega_\rmv^\prime}}}$.
    For $S\in\Safeset^c\cap\Subopt$, we denote 
    $\bDelta_S:=\max\left\{\omega_1\Delta_{S}
                            ,
                            \omega_3\frac{\Deltav_{S}}{3}
                        \right\}  $.
                        
    For $i\in E$, denote 
    $$\bDelta_{i,\Safeset^c\cap\Subopt}
    :=
    \min_{S\ni i,S\in\Safeset^c\cap\Subopt} \max\left\{\frac{\Delta_S}{\sqrt{\ln(1/\omega_\mu)}}
    ,
    \frac{\Deltav_S}{3\sqrt{\ln(1/\omega_\rmv^\prime)}}
    \right\}
    =
    \sqrt{\frac{1}{\ln(1/\omega_{\max})}}\min_{S\ni i,S\in\Safeset^c\cap\Subopt}\bDelta_{S}.$$
    and assume 
    $\bDelta_{i,\Safeset^c\cap\Subopt}\in
    (
        r_i\frac{\bsigma^2/3}{\sqrt{\ln(1/\omega_\rmv)}+\sqrt{\ln(1/\omega_\rmv^\prime)}}
        ,
        (r_i+1)\frac{\bsigma^2/3}{\sqrt{\ln(1/\omega_\rmv)}+\sqrt{\ln(1/\omega_\rmv^\prime)}}
    ]$

    \begin{align}
        &\sum_{\phase=1}^{\Tp}\mathbbm{1}\left\{A_\phase\in\Safeset^c\cap\Subopt\right\}\Big)
                \mathbbm{1}\left\{ U_{A_\phase}^\rmv(p-1)>r\bsigma^2 \right\}\cdot
                \mathbbm{1}
			     \left\{\mathcal{E}\right\}
        \\
        &\overset{(a)}{\leq}
        \sum_{\phase=1}^{\Tp}
            \mathbbm{1}\left\{A_\phase\in\Safeset^c\cap\Subopt\right\}
                \mathbbm{1}\left\{
                \mathcal{F}_\phase^\mu
                ,
                \mathcal{F}_\phase\left(\frac{\Deltav_{A_\phase}}{3},\omega_\rmv^\prime\right)
                ,
                \mathcal{F}_\phase\left(\frac{(r-1)\bsigma^2-\Deltav_{A_\phase}}{3},\omega_\rmv\right)
                \right\}
        \\
        &\overset{(b)}{\leq}
        \sum_{\phase=1}^{\Tp}
            \mathbbm{1}\left\{A_\phase\in\Safeset^c\cap\Subopt\right\}
                \mathbbm{1}\left\{
                \mathcal{F}_\phase\left(\Delta_{A_\phase},\omega_\mu\right)
                ,
                \mathcal{F}_\phase\left(\frac{\Deltav_{A_\phase}}{3},\omega_\rmv^\prime\right)
                ,
                \mathcal{F}_\phase\left(\frac{(r-1)\bsigma^2-\Deltav_{A_\phase}}{3},\omega_\rmv\right)
                \right\}
        \\
        &\overset{(c)}{\leq}
        \sum_{\phase=1}^{\Tp}
            \mathbbm{1}\left\{A_\phase\in\Safeset^c\cap\Subopt\right\}
                \mathbbm{1}\left\{
                \mathcal{F}_\phase\left(
                \max\left\{\omega_1\Delta_{A_\phase}
                ,
                \omega_3\frac{\Deltav_{A_\phase}}{3}
                ,
                \omega_2\frac{(r-1)\bsigma^2-\Deltav_{A_\phase}}{3}
                \right\},\omega_{\max}\right)
                \right\}
        \\
        &\overset{(d)}{\leq}
        \sum_{\phase=1}^{\Tp}
            \mathbbm{1}\left\{A_\phase\in\Safeset^c\cap\Subopt\right\}
                \mathbbm{1}\left\{
                \mathcal{F}_\phase\left(
                \max\left\{
                    \max\left\{\omega_1\Delta_{A_\phase}
                            ,
                            \omega_3\frac{\Deltav_{A_\phase}}{3}
                        \right\}  
                ,
                \omega_2\frac{(r-1)\bsigma^2}{3}-\frac{\omega_2}{\omega_3}
                    \max\left\{\omega_1\Delta_{A_\phase}
                            ,
                            \omega_3\frac{\Deltav_{A_\phase}}{3}
                        \right\}  
                \right\}
                \right.\right.
                \\
                &\hspace{14em}
                ,\omega_{\max}
                \bigg)
                \bigg\}
        \\
        &\overset{(e)}{\leq}
        \sum_{\phase=1}^{\Tp}
        \sum_{j\in\mathbb{N}}
            \mathbbm{1}\left\{A_\phase\in\Safeset^c\cap\Subopt\right\}
                \mathbbm{1}\left\{\mathcal{G}_{j,\phase}\left(
                \max\left\{
                    \bDelta_{A_\phase}
                    ,
                    \omega_2\frac{(r-1)\bsigma^2}{3}-\frac{\omega_2}{\omega_3}\bDelta_{A_\phase}
                \right\},
                \omega_{\max}\right)\right\}
        \\
        &\overset{(f)}{\leq}
       \sum_{\phase=1}^{\Tp}
        \sum_{j\in\mathbb{N}}\mathbbm{1}\left\{A_\phase\in\Safeset^c\cap\Subopt\right\}
        \frac{1}{b_j K} \sum_{i\in A_\phase}
                \mathbbm{1}\left\{T_i(\phase-1)\leq m_j\left(
                \max\left\{
                    \bDelta_{A_\phase}
                    ,
                    \omega_2\frac{(r-1)\bsigma^2}{3}-\frac{\omega_2}{\omega_3}\bDelta_{A_\phase}
                \right\},
                \omega_{\max}\right)\right\}
        \\
        &=
        \sum_{i\in E}
        \sum_{j\in\mathbb{N}}\frac{1}{b_j K}
        \sum_{\phase=1}^{\Tp}
        \mathbbm{1}\left\{A_\phase\in\Safeset^c\cap\Subopt\right\}
                \mathbbm{1}\left\{i\in A_\phase,T_i(\phase-1)\leq  m_j\left(
                \max\left\{
                    \bDelta_{A_\phase}
                    ,
                    \omega_2\frac{(r-1)\bsigma^2}{3}-\frac{\omega_2}{\omega_3}\bDelta_{A_\phase}
                \right\},
                \omega_{\max}\right)\right\}
        \label{equ:discussionOnriUnsafecapSub}
        \end{align}

        Given $i\in E$
        
        (1) if $r\geq r_i+2$, for any $S\in \Safeset^c\cap\Subopt$ that contains item $i$,
        \begin{align}
            \max\left\{
                    \bDelta_{S}
                    ,
                    \omega_2\frac{(r-1)\bsigma^2}{3}-\frac{\omega_2}{\omega_3}\bDelta_{S}
                \right\}
            \geq 
            \max\left\{
                    \bDelta_{S}
                    ,
                    \frac{\omega_2\omega_3}{\omega_2+\omega_3}\frac{(r-1)\bsigma^2}{3}
                \right\}
            \geq 
            \frac{\omega_2\omega_3}{\omega_2+\omega_3}\frac{(r-1)\bsigma^2}{3}
            \geq
            \sqrt{\ln(1/\omega_{\max})}\bDelta_{i,\Safeset^c\cap\Riskset}
        \end{align}
        where the first inequality uses the fact that for $x,y\in\mathbb{R}_+$ and $z\in[0,1]$, , $\max\{x,y\}\geq \max\{x,xz+y(1-z)\}$. We take $x=\bDelta_{A_\phase},y=\omega_2\frac{(r-1)\bsigma^2}{3}-\frac{\omega_2}{\omega_3}\bDelta_{A_\phase}$ and $z=\frac{\omega_2}{\omega_2+\omega_3}$.
        
        Similar to the computations for the case $A_\phase\in\Riskset$, we have
        \begin{align}
                &\sum_{j\in\mathbb{N}}\frac{1}{b_j K}
                \sum_{\phase=1}^{\Tp}
                \mathbbm{1}\left\{A_\phase\in\Safeset^c\cap\Subopt\right\}
                \mathbbm{1}\left\{i\in A_\phase,T_i(\phase-1)\leq  m_j\left(
                \max\left\{
                    \bDelta_{A_\phase}
                    ,
                    \omega_2\frac{(r-1)\bsigma^2}{3}-\frac{\omega_2}{\omega_3}\bDelta_{A_\phase}
                \right\},
                \omega_{\max}\right)\right\}
                \\
                &\leq
                \sum_{j\in\mathbb{N}}\frac{1}{b_j K}
                \sum_{\phase=1}^{\Tp}
                \mathbbm{1}\left\{A_\phase\in\Safeset^c\cap\Subopt\right\}
                \mathbbm{1}\left\{i\in A_\phase,T_i(\phase-1)\leq  m_j\left(
                \frac{\omega_2\omega_3}{\omega_2+\omega_3}\frac{(r-1)\bsigma^2}{3},
                \omega_{\max}\right)\right\}
                \\
                &\leq
                 \frac{\abcstt\gamma K}{(\frac{\omega_2\omega_3}{\omega_2+\omega_3}\frac{(r-1)\bsigma^2}{3})^2}\left(2\ln\frac{1}{\omega_{\max}}+\ln\ln_+\frac{1}{(\frac{\omega_2\omega_3}{\omega_2+\omega_3}\frac{(r-1)\bsigma^2}{3})^2}+\upbdcstt\right)
                 \\
                 &\leq
                 \frac{\abcstt\gamma K}{(\frac{\omega_2\omega_3}{\omega_2+\omega_3}\frac{(r-1)\bsigma^2}{3})^2}\left(2\ln\frac{1}{\omega_{\max}}+\ln\ln_+\frac{1}{\ln(1/\omega_{\max})(\bDelta_{i,\Safeset^c\cap\Riskset})^2}+\upbdcstt\right)
                 \\
                 &=
                 \frac{\abcstt\gamma K}{(\frac{(r-1)\bsigma^2/3}{\sqrt{\ln(1/\omega_\rmv)}+\sqrt{\ln(1/\omega_\rmv^\prime)}})^2}\left(2+\frac{1}{\ln(1/\omega_{\max})}\left(\ln\ln_+\frac{1}{\ln(1/\omega_{\max})(\bDelta_{i,\Safeset^c\cap\Riskset})^2}+\upbdcstt\right)\right)
        \end{align}

        (2) if $r\leq r_i+1$, for any $S\in \Safeset^c\cap\Subopt$ that contains item $i$,
        \begin{align}
            \max\left\{
                    \bDelta_{S}
                    ,
                    \omega_2\frac{(r-1)\bsigma^2}{3}-\frac{\omega_2}{\omega_3}\bDelta_{S}
                \right\}
            \geq 
            \max\left\{
                    \bDelta_{S}
                    ,
                    \frac{\omega_2\omega_3}{\omega_2+\omega_3}\frac{(r-1)\bsigma^2}{3}
                \right\}
            \geq 
            \bDelta_{S}
            \geq
            \sqrt{\ln(1/\omega_{\max})}\bDelta_{i,\Safeset^c\cap\Riskset}
        \end{align}
        
        Similar to the computations for the case $A_\phase\in\Riskset$, we have
        \begin{align}
                &\sum_{j\in\mathbb{N}}\frac{1}{b_j K}
                \sum_{\phase=1}^{\Tp}
                \mathbbm{1}\left\{A_\phase\in\Safeset^c\cap\Subopt\right\}
                \mathbbm{1}\left\{i\in A_\phase,T_i(\phase-1)\leq  m_j\left(
                \max\left\{
                    \bDelta_{A_\phase}
                    ,
                    \omega_2\frac{(r-1)\bsigma^2}{3}-\frac{\omega_2}{\omega_3}\bDelta_{A_\phase}
                \right\},
                \omega_{\max}\right)\right\}
                \\
                &\leq
                \sum_{j\in\mathbb{N}}\frac{1}{b_j K}
                \sum_{\phase=1}^{\Tp}
                \mathbbm{1}\left\{A_\phase\in\Safeset^c\cap\Subopt\right\}
                \mathbbm{1}\left\{i\in A_\phase,T_i(\phase-1)\leq  m_j\left(
                \sqrt{\ln(1/\omega_{\max})}\bDelta_{i,\Safeset^c\cap\Riskset},
                \omega_{\max}\right)\right\}
                \\
                &\leq
                 \frac{\abcstt\gamma K}{(\sqrt{\ln(1/\omega_{\max})}\bDelta_{i,\Safeset^c\cap\Subopt})^2}\left(2\ln\frac{1}{\omega_{\max}}+\ln\ln_+\frac{1}{(\sqrt{\ln(1/\omega_{\max})}\bDelta_{i,\Safeset^c\cap\Subopt})^2}+\upbdcstt\right)
                 \\
                 &=
                 \frac{\abcstt\gamma K}{(\bDelta_{i,\Safeset^c\cap\Subopt})^2}\left(2+\frac{1}{\ln(1/\omega_{\max})}\left(\ln\ln_+\frac{1}{\ln(1/\omega_{\max})(\bDelta_{i,\Safeset^c\cap\Riskset})^2}+\upbdcstt\right)\right)
        \end{align}

In conclusion, for $i\in E$, we denote 
    $$\bDelta_{i,\Safeset^c\cap\Subopt}
    :=
    \min_{S\ni i,S\in\Safeset^c\cap\Subopt} \max\left\{\frac{\Delta_S}{\sqrt{\ln(1/\omega_\mu)}}
    ,
    \frac{\Deltav_S}{3\sqrt{\ln(1/\omega_\rmv^\prime)}}
    \right\}.
    $$
    and $\omega_{\max}:=\max\{\omega_\mu,\omega_\rmv,\omega_\rmv^\prime\}$
    and
    $\omega_{\rmsum}:=\sqrt{\ln\frac{1}{\omega_\rmv^\prime}}+\sqrt{\ln\frac{1}{\omega_\rmv}}$ and $g_{\Safeset^c\cap\Subopt}(r,\bDelta_{i,\Safeset^c\cap\Subopt}):=$
   \begin{align}
            \left\{
            \begin{aligned}
                &
                \frac{\abcstt\gamma K}{(\frac{(r-1)\bsigma^2}{3\omega_{\rmsum}})^2}\left(2+\frac{1}{\ln(1/\omega_{\max})}\left(\ln\ln_+\frac{1}{\ln(1/\omega_{\max})(\bDelta_{i,\Safeset^c\cap\Riskset})^2}+\upbdcstt\right)\right), 
                r\geq \left\lfloor\frac{\omega_\rmsum\cdot\bDelta_{i,\Safeset^c\cap\Subopt} }{\bsigma^2/3}\right\rfloor+2
                \\
                &
                 \frac{\abcstt\gamma K}{(\bDelta_{i,\Safeset^c\cap\Riskset})^2}\left(2+\frac{1}{\ln(1/\omega_{\max})}\left(\ln\ln_+\frac{1}{\ln(1/\omega_{\max})(\bDelta_{i,\Safeset^c\cap\Riskset})^2}+\upbdcstt\right)\right),
                r\leq \left\lfloor\frac{\omega_\rmsum\cdot\bDelta_{i,\Safeset^c\cap\Subopt} }{\bsigma^2/3}\right\rfloor+1
            \end{aligned}\right.
    \end{align}
    then 
    \begin{align}
         &\sum_{\phase=1}^{\Tp}
            \mathbbm{1}\left\{A_\phase\in\Safeset^c\cap\Subopt\right\}
                \mathbbm{1}\left\{ U_{A_\phase}^\rmv(\phase-1)>r\bsigma^2 \right\}
        \leq\sum_{i\in E}g_{\Safeset^c\cap\Subopt}(r,\bDelta_{i,\Safeset^c\cap\Subopt})
    \end{align}

    \textbf{Conclusion of Step 2:}
    
    \underline{For $S^\star$}: denote 
    \begin{align}\label{equ:gAstar}
        g_{S^\star}(r,\Deltav_{S^\star}):=
        \frac{9\cdot\abcstt\gamma K}{((r-1)\bsigma^2+\Deltav_{S^\star})^2}
        \left(2\ln\frac{1}{\omega_\rmv}+\ln\ln_+\frac{1}{{\Deltav_{S^\star}}^2}+\upbdcstt\right)
    \end{align}
    we have
    \begin{align}
        \sum_{\phase=1}^{\Tp}
        \mathbbm{1}\left\{A_\phase=S^\star\right\}
        \mathbbm{1}\left\{ U_{A_\phase}^\rmv(\phase-1)>r\bsigma^2 \right\}
            \leq
            \sum_{i\in S^\star}g_{S^\star}(r,\Deltav_{S^\star})
    \end{align}

    \underline{For $\Safeset\cap\Subopt$}: 
    for $i\in E$, denote $\bDelta_{i,\Safeset\cap\Subopt}:=\min_{S\ni i,S\in\Safeset\cap\Subopt}\max\left\{\frac{\Delta_{S}}{\sqrt{\ln(1/\omega_\mu)}},\frac{\Deltav_{S}}{3\sqrt{\ln(1/\omega_\rmv)}}\right\}$, 
    $\omega_{\mu\rmv}:=\max\{\omega_\mu,\omega_\rmv\}$ and
   \begin{align}\label{equ:gsafesub}
            &g_{\Safeset\cap\Subopt}(r,\bDelta_{i,\Safeset\cap\Subopt})
            \\
            &:=
            \left\{
            \begin{aligned}
                &
                \frac{\abcstt\gamma K}{(\frac{(r-1)\bsigma^2}{3\sqrt{\ln(1/\omega_\rmv)}})^2}
                \left(2+\frac{1}{\ln(1/\omega_{\mu\rmv})}\left(\ln\ln_+\frac{1}{\ln(1/\omega_{\mu\rmv})\bDelta_{i,\Safeset\cap\Subopt}^2}+\upbdcstt\right)\right), 
                r\geq \left\lfloor\frac{\bDelta_{i,\Safeset\cap\Subopt} }{\frac{\bsigma^2}{3\sqrt{\ln(1/\omega_\rmv)}}}\right\rfloor+2
                \\
                &
                \frac{\abcstt\gamma K}{\bDelta_{i,\Safeset\cap\Subopt}^2}
                \left(2
                +\frac{1}{\ln(1/\omega_{\mu\rmv})}\left(\ln\ln_+\frac{1}{\ln(1/\omega_{\mu\rmv})\bDelta_{i,\Safeset\cap\Subopt}^2}+\upbdcstt\right)\right),
                r\leq \left\lfloor\frac{\bDelta_{i,\Safeset\cap\Subopt} }{\frac{\bsigma^2}{3\sqrt{\ln(1/\omega_\rmv)}}}\right\rfloor+1
            \end{aligned}\right.
    \end{align}
    then 
    \begin{align}
         &\sum_{\phase=1}^{\Tp}
            \mathbbm{1}\left\{A_\phase\in\Safeset\cap\Subopt\right\}
                \mathbbm{1}\left\{ U_{A_\phase}^\rmv(\phase-1)>r\bsigma^2 \right\}
        \leq\sum_{i\in E}g_{\Safeset\cap\Subopt}(r,\bDelta_{i,\Safeset\cap\Subopt})
    \end{align}

    \underline{For $\Riskset$}: for $i\in E$, denote $\Deltav_{i,\Riskset}:=\min_{S\ni i,S\in\Riskset}\Deltav_S$, 
    $\omega_{\rmv\rmv^\prime}:=\max\{\omega_\rmv,\omega_\rmv^\prime\}$, $\omega_{\rmsum}:=\sqrt{\ln\frac{1}{\omega_\rmv^\prime}}+\sqrt{\ln\frac{1}{\omega_\rmv}}$ and
   \begin{align}\label{equ:grisk}
            &g_{\Riskset}(r,\Deltav_{i,\Riskset})
            \\
            &:=
            \left\{
            \begin{aligned}
                &
                \frac{\abcstt\gamma K}{(\frac{(r-1)\bsigma^2}
                {3\omega_{\rmsum}})^2}\left(2+\frac{1}{\ln(1/\omega_{\rmv\rmv^\prime})}\left(\ln\ln_+\frac{1}{\ln(1/\omega_{\rmv\rmv^\prime})(\frac{\Deltav_{i,\Riskset}}{3\sqrt{\ln(1/\omega_\rmv^\prime)}})^2}+\upbdcstt\right)\right), 
                r\geq \left\lfloor\frac{\omega_\rmsum\cdot\Deltav_{i,\Riskset} }{\sqrt{\ln(1/\omega_\rmv^\prime)}\bsigma^2}\right\rfloor+2
                \\
                &
                \frac{\abcstt\gamma K}{(\frac{\Deltav_{i,\Riskset}}{3\sqrt{\ln(1/\omega_\rmv^\prime)}})^2}
                \left(2+\frac{1}{\ln(1/\omega_{\rmv\rmv^\prime})}\left(\ln\ln_+\frac{1}{\ln(1/\omega_{\rmv\rmv^\prime})(\frac{\Deltav_{i,\Riskset}}{3\sqrt{\ln(1/\omega_\rmv^\prime)}})^2}+\upbdcstt\right)\right),
                r\leq \left\lfloor\frac{\omega_\rmsum\cdot\Deltav_{i,\Riskset} }{\sqrt{\ln(1/\omega_\rmv^\prime)}\bsigma^2}\right\rfloor+1
            \end{aligned}\right.
    \end{align}
    then 
    \begin{align}
         &\sum_{\phase=1}^{\Tp}
            \mathbbm{1}\left\{A_\phase\in\Riskset\right\}
                \mathbbm{1}\left\{ U_{A_\phase}^\rmv(\phase-1)>r\bsigma^2 \right\}
        \leq\sum_{i\in E}g_{\Riskset}(r,\Deltav_{i,\Riskset})
    \end{align}

    \underline{For $\Safeset^c\cap\Subopt$}: 
    for $i\in E$,  denote 
    $$\bDelta_{i,\Safeset^c\cap\Subopt}
    :=
    \min_{S\ni i,S\in\Safeset^c\cap\Subopt} \max\left\{\frac{\Delta_S}{\sqrt{\ln(1/\omega_\mu)}}
    ,
    \frac{\Deltav_S}{3\sqrt{\ln(1/\omega_\rmv^\prime)}}
    \right\}.
    $$
    and $\omega_{\max}:=\max\{\omega_\mu,\omega_\rmv,\omega_\rmv^\prime\}$
    and
    $\omega_{\rmsum}:=\sqrt{\ln\frac{1}{\omega_\rmv^\prime}}+\sqrt{\ln\frac{1}{\omega_\rmv}}$ and
   \begin{align}\label{equ:gunsafesub}
            &g_{\Safeset^c\cap\Subopt}(r,\bDelta_{i,\Safeset^c\cap\Subopt})
            \\
            &:=
            \left\{
            \begin{aligned}
                &
                \frac{\abcstt\gamma K}{(\frac{(r-1)\bsigma^2}{3\omega_{\rmsum}})^2}\left(2+\frac{1}{\ln(1/\omega_{\max})}\left(\ln\ln_+\frac{1}{\ln(1/\omega_{\max})(\bDelta_{i,\Safeset^c\cap\Subopt})^2}+\upbdcstt\right)\right), 
                r\geq \left\lfloor\frac{\omega_\rmsum\cdot\bDelta_{i,\Safeset^c\cap\Subopt} }{\bsigma^2/3}\right\rfloor+2
                \\
                &
                 \frac{\abcstt\gamma K}{(\bDelta_{i,\Safeset^c\cap\Subopt})^2}\left(2+\frac{1}{\ln(1/\omega_{\max})}\left(\ln\ln_+\frac{1}{\ln(1/\omega_{\max})(\bDelta_{i,\Safeset^c\cap\Subopt})^2}+\upbdcstt\right)\right),
                r\leq \left\lfloor\frac{\omega_\rmsum\cdot\bDelta_{i,\Safeset^c\cap\Subopt} }{\bsigma^2/3}\right\rfloor+1
            \end{aligned}\right.
    \end{align}
    then 
    \begin{align}
         &\sum_{\phase=1}^{\Tp}
            \mathbbm{1}\left\{A_\phase\in\Safeset^c\cap\Subopt\right\}
                \mathbbm{1}\left\{ U_{A_\phase}^\rmv(\phase-1)>r\bsigma^2 \right\}
        \leq\sum_{i\in E}g_{\Safeset^c\cap\Subopt}(r,\bDelta_{i,\Safeset^c\cap\Subopt})
    \end{align}

    \noindent
    \textbf{\underline{Step 3:}}

    According to the results in Step 2, given $r\in[Q-1]$, the event $\mathcal{U}_p(r)$ can happen at most 
    \begin{align}\label{equ:upbd_phase}
        &\sum_{\phase=1}^{\Tp}
        \mathbbm{1}\left\{\mathcal{U}_p(r)\right\}
        \\
        &\leq
        \min\left\{\Tp,\sum_{i\in S^\star}g_{S^\star}(r,\Deltav_{S^\star})
        +
        \sum_{i\in E}g_{\Safeset\cap\Subopt}(r,\bDelta_{i,\Safeset\cap\Subopt})
        +
        \sum_{i\in E} g_{\Riskset}(r,\Deltav_{i,\Riskset})
        +
        \sum_{i\in E}g_{\Safeset^c\cap\Subopt}(r,\bDelta_{i,\Safeset^c\cap\Subopt})
        \right\}
    \end{align}
    phases, where the $g$ functions are defined in \eqref{equ:gAstar}, \eqref{equ:gsafesub}, \eqref{equ:grisk} and \eqref{equ:gunsafesub}. By Lemma~\ref{lem:np},  
    (1) event $\mathcal{U}_p(r)$ indicates $ r+1\leq n_\phase$, thus \eqref{equ:upbd_phase} also indicates at most in this number of phases there are at least $r+1$ being pulled at each phase.
    (2) event $\mathcal{U}_p(r)\cap\mathcal{U}_p(r+1)^c$ indicates $ r+1\leq n_\phase\leq 2r+1$, i.e., there are at least $r+1$ and at most $2r+1$ solutions being pulled at each phase.
    For $r\in[Q]$, we denote 
    \begin{align}\label{equ:Tpr}
        \Tp_r:=\sum_{i\in S^\star}g_{S^\star}(r,\Deltav_{S^\star})
        +
        \sum_{i\in E}g_{\Safeset\cap\Subopt}(r,\bDelta_{i,\Safeset\cap\Subopt})
        +
        \sum_{i\in E} g_{\Riskset}(r,\Deltav_{i,\Riskset})
        +
        \sum_{i\in E}g_{\Safeset^c\cap\Subopt}(r,\bDelta_{i,\Safeset^c\cap\Subopt}),\quad r\in[Q-1]
    \end{align}
    and $\Tp_Q:=0,\Tp_0=\infty$.
    Note that the $g$ functions are increasing as $r$ decreases, so if there exists an $r^\prime\in[Q]$, such that $\Tp\in[\Tp_{r^\prime},\Tp_{r^\prime-1})$ (or $\Tp\leq\Tp_{r^\prime-1}$) , then 
     \begin{align}
        &\sum_{r=1}^{Q-1}
            \sum_{\phase=1}^{\Tp}
            \mathbbm{1}\left\{\mathcal{U}_p(r)\right\}
        \\
        &=
        \sum_{r=1}^{r^\prime-1}
            \sum_{\phase=1}^{\Tp}
            \mathbbm{1}\left\{\mathcal{U}_p(r)\right\}
            +
            \sum_{r=r^\prime}^{Q-1}
            \sum_{\phase=1}^{\Tp}
            \mathbbm{1}\left\{\mathcal{U}_p(r)\right\}
        \\
        &\leq
        \Tp\cdot (r^\prime-1)+
        \sum_{r=r^\prime}^{Q-1}
        \sum_{i\in S^\star}g_{S^\star}(r,\Deltav_{S^\star})
        +
        \sum_{i\in E}g_{\Safeset\cap\Subopt}(r,\bDelta_{i,\Safeset\cap\Subopt})
        +
        \sum_{i\in E} g_{\Riskset}(r,\Deltav_{i,\Riskset})
        +
        \sum_{i\in E}g_{\Safeset^c\cap\Subopt}(r,\bDelta_{i,\Safeset^c\cap\Subopt})
        \\
        &=
        \Tp\cdot (r^\prime-1)+
        \sum_{r=r^\prime}^{Q-1}
        \Tp_r
        \\
        &\leq
        \Tp\cdot (r^\prime-1)+
        \sum_{i\in S^\star}h_{S^\star}(r^\prime,\Deltav_{S^\star})
        +
        \sum_{i\in E}h_{\Safeset\cap\Subopt}(r^\prime,\bDelta_{i,\Safeset\cap\Subopt})
        +
        \sum_{i\in E} h_{\Riskset}(r^\prime,\Deltav_{i,\Riskset})
        +
        \sum_{i\in E}h_{\Safeset^c\cap\Subopt}(r^\prime,\bDelta_{i,\Safeset^c\cap\Subopt})
        \\
        &=\Tp\cdot (r^\prime-1)+H(r^\prime,\instance)
    \end{align}
    where 
    \begin{align}\label{equ:Hinstance}
        H(r^\prime,\instance):=
        \sum_{i\in S^\star}h_{S^\star}(r^\prime,\Deltav_{S^\star})
        +
        \sum_{i\in E}h_{\Safeset\cap\Subopt}(r^\prime,\bDelta_{i,\Safeset\cap\Subopt})
        +
        \sum_{i\in E} h_{\Riskset}(r^\prime,\Deltav_{i,\Riskset})
        +
        \sum_{i\in E}h_{\Safeset^c\cap\Subopt}(r^\prime,\bDelta_{i,\Safeset^c\cap\Subopt})
    \end{align} and  the $h$ functions are defined at the end of the proof.
    This indicates, when $\Tp\leq\Tp_{r^\prime-1}$, the upper bound of the high-probability regret due to safeness-checking $\rmR_2(T^\prime)$ (hence the the total regret) is lower bounded by a linear function with slope $r^\prime-1$ . 
    In particular, the upper bound of $\rmR_2(\Tp)$ is lower bounded by a linear function when $\Tp\leq\Tp_{1}$ and it remains a constant when $\Tp>\Tp_{1}$.
    We can compute an upper bound for the number of solutions being pulled during these $\Tp$ phases when $\Tp\in[\Tp_{r^\prime},\Tp_{r^\prime-1})$:
    \begin{align}
        &\Tp_{Q-1}\cdot (2Q-1)+(\Tp_{Q-2}-\Tp_{Q-1})\cdot (2Q-3)+\cdots+(\Tp_{r^\prime}-\Tp_{r^\prime+1})\cdot (2r^\prime+1)+(\Tp-\Tp_{r^\prime})(2r^\prime-1)
        \\
        &=
        \Tp+2\left(\Tp_{Q-1}\cdot (Q-1)+(\Tp_{Q-2}-\Tp_{Q-1})\cdot (Q-2)+\cdots+(\Tp_{r^\prime}-\Tp_{r^\prime+1})\cdot r^\prime+(\Tp-\Tp_{r^\prime})(r^\prime-1)\right)
        \\
        &=(2r^\prime-1)\Tp+2\sum_{r=r^\prime}^{Q-1}\Tp_r
        \\
        &\leq (2r^\prime-1)\Tp+2H(r^\prime,\instance)
    \end{align}
    Thus, the number of pulled solutions is at most $2H(1,\instance)$.\footnote{Note that the upper bound for the regret is $2\mu^\star \Tp\cdot (r^\prime-1)+2\mu^\star \cdot H(r^\prime,\instance)$ and the upper bound for the number of solutions is $(2r^\prime-1)\Tp+2H(r^\prime,\instance)$, which indicates we can roughly use $\min\{T\mu^\star ,2\mu^\star H(r^\prime,\instance)\}$ to bound the regret due to safeness-checking with $T$ time steps.}
    

    \textbf{In conclusion}, the regret due to safeness-checking can be upper bounded by 
    \begin{align}
        \rmR_2(\Tp)
            &=
            2\mu^\star \sum_{r=1}^{Q-1}
            \sum_{\phase=1}^{\Tp}
            \mathbbm{1}\left\{\mathcal{U}_p(r)\right\}
            \\
            &\leq
            2\mu^\star \Tp\cdot (r^\prime-1)
            \\
            &\quad+
            2\mu^\star \left(\sum_{i\in S^\star}h_{S^\star}(r^\prime,\Deltav_{S^\star})
            +
            \sum_{i\in E}h_{\Safeset\cap\Subopt}(r^\prime,\bDelta_{i,\Safeset\cap\Subopt})
            +
            \sum_{i\in E} h_{\Riskset}(r^\prime,\Deltav_{i,\Riskset})
            +
            \sum_{i\in E}h_{\Safeset^c\cap\Subopt}(r^\prime,\bDelta_{i,\Safeset^c\cap\Subopt})\right)
            \\
            &=2\mu^\star \Tp\cdot (r^\prime-1)+2\mu^\star \cdot H(r^\prime,\instance)
    \end{align}
    where $\Tp\in[\Tp_{r^\prime},\Tp_{r^\prime-1})$ with $\Tp_{r^\prime}$ defined in \eqref{equ:Tpr}, 
    for $i\in E$,
    $\bDelta_{i,\Safeset\cap\Subopt}:=\min_{S\ni i,S\in\Safeset\cap\Subopt}\max\left\{\frac{\Delta_{S}}{\sqrt{\ln(1/\omega_\mu)}},\frac{\Deltav_{S}}{3\sqrt{\ln(1/\omega_\rmv)}}\right\}$, $\Deltav_{i,\Riskset}:=\min_{S\ni i,S\in\Riskset}\Deltav_S$ and 
    $\bDelta_{i,\Safeset^c\cap\Subopt}
    :=
    \min_{S\ni i,S\in\Safeset^c\cap\Subopt} \max\left\{\frac{\Delta_S}{\sqrt{\ln(1/\omega_\mu)}}
    ,
    \frac{\Deltav_S}{3\sqrt{\ln(1/\omega_\rmv^\prime)}}
    \right\}.
    $

    
    \textbf{The $h$ functions:}

    (For convenience, we restate the notations:
    $\omega_{\mu\rmv}:=\max\{\omega_\mu,\omega_\rmv\}$, $\omega_{\rmv\rmv^\prime}:=\max\{\omega_\rmv,\omega_\rmv^\prime\}$,  $\omega_{\max}:=\max\{\omega_\mu,\omega_\rmv,\omega_\rmv^\prime\}$ and  $\omega_{\rmsum}:=\sqrt{\ln\frac{1}{\omega_\rmv^\prime}}+\sqrt{\ln\frac{1}{\omega_\rmv}}$.)
    
    For $S^\star$: for each $i\in S^\star$
    \begin{align}\label{equ:hAstar}
        &h_{S^\star}(r^\prime,\Deltav_{S^\star}):=\sum_{r=r^\prime}^{Q-1}
        g_{S^\star}(r,\Deltav_{S^\star})
        \\
        &=\sum_{r=r^\prime}^{Q-1}
        \frac{9\cdot\abcstt\gamma K}{((r-1)\bsigma^2+\Deltav_{S^\star})^2}
        \left(2\ln\frac{1}{\omega_\rmv}+\ln\ln_+\frac{1}{{\Deltav_{S^\star}}^2}+\upbdcstt\right)
        \\
        &\leq\left\{\begin{aligned}
            &\frac{18\cdot\abcstt\gamma K}{(r^\prime-1)\bsigma^4}
            \left(2\ln\frac{1}{\omega_\rmv}+\ln\ln_+\frac{1}{(\Deltav_{S^\star})^2}+\upbdcstt\right),\quad r^\prime\geq 2
            \\
            &\frac{18\cdot\abcstt\gamma K}{(\Deltav_{S^\star})^2}
            \left(2\ln\frac{1}{\omega_\rmv}+\ln\ln_+\frac{1}{(\Deltav_{S^\star})^2}+\upbdcstt\right),\quad r^\prime=1
        \end{aligned}\right.
    \end{align}

    For $\Safeset\cap\Subopt$: for each $i\in E$, there is a changing point $\left\lfloor\frac{\bDelta_{i,\Safeset\cap\Subopt} }{\frac{\bsigma^2}{3\sqrt{\ln(1/\omega_\rmv)}}}\right\rfloor+2$ in $g_{\Safeset\cap\Subopt}(r,\bDelta_{i,\Safeset\cap\Subopt})$, thus
    \begin{align}\label{equ:hsafesub}
        &h_{\Safeset\cap\Subopt}(r^\prime,\bDelta_{i,\Safeset\cap\Subopt}):=\sum_{r=r^\prime}^{Q-1}
        g_{\Safeset\cap\Subopt}(r,\bDelta_{i,\Safeset\cap\Subopt})
        \\
        &\leq
        \left\{\begin{aligned}
            &0,r^\prime=Q
            \\
            &(Q-r^\prime)\frac{\abcstt\gamma K}{\bDelta_{i,\Safeset\cap\Subopt}^2}
                \left(2
                +\frac{1}{\ln(1/\omega_{\mu\rmv})}\left(\ln\ln_+\frac{1}{\ln(1/\omega_{\mu\rmv})\bDelta_{i,\Safeset\cap\Subopt}^2}+\upbdcstt\right)\right),
                \left\lfloor\frac{\bDelta_{i,\Safeset\cap\Subopt} }{\frac{\bsigma^2}{3\sqrt{\ln(1/\omega_\rmv)}}}\right\rfloor\geq Q-3
            \\
            &\frac{2\cdot\abcstt\gamma K}{(r^\prime-1)(\frac{\bsigma^2}{3\sqrt{\ln(1/\omega_\rmv)}})^2}
                \left(2+\frac{1}{\ln(1/\omega_{\mu\rmv})}\left(\ln\ln_+\frac{1}{\ln(1/\omega_{\mu\rmv})\bDelta_{i,\Safeset\cap\Subopt}^2}+\upbdcstt\right)\right), 
                \left\lfloor\frac{\bDelta_{i,\Safeset\cap\Subopt} }{\frac{\bsigma^2}{3\sqrt{\ln(1/\omega_\rmv)}}}\right\rfloor+2
                \leq r^\prime<Q-1
            \\
             &\frac{3\cdot\abcstt\gamma K}{\bDelta_{i,\Safeset\cap\Subopt}^2}
            \left(2
            +\frac{1}{\ln(1/\omega_{\mu\rmv})}\left(\ln\ln_+\frac{1}{\ln(1/\omega_{\mu\rmv})\bDelta_{i,\Safeset\cap\Subopt}^2}+\upbdcstt\right)\right),
            \left\lfloor\frac{\bDelta_{i,\Safeset\cap\Subopt} }{\frac{\bsigma^2}{3\sqrt{\ln(1/\omega_\rmv)}}}\right\rfloor=0,r^\prime=1
            \\
            &\abcstt\gamma K\left(\frac{4}{\frac{\bsigma^2}{3\sqrt{\ln(1/\omega_\rmv)}}\cdot\bDelta_{i,\Safeset\cap\Subopt}}-\frac{r^\prime-1}{\bDelta_{i,\Safeset\cap\Subopt}^2}\right)
            \left(2
            +\frac{1}{\ln(1/\omega_{\mu\rmv})}\left(\ln\ln_+\frac{1}{\ln(1/\omega_{\mu\rmv})\bDelta_{i,\Safeset\cap\Subopt}^2}+\upbdcstt\right)\right),
            \text{otherwise}
        \end{aligned}\right.
    \end{align}

    For $\Riskset$: for each $i\in E$, there is a changing point $\left\lfloor\frac{\omega_\rmsum\cdot\Deltav_{i,\Riskset} }{\sqrt{\ln(1/\omega_\rmv^\prime)}\bsigma^2}\right\rfloor+2$, thus
    \begin{align}\label{equ:hrisk}
        &h_{\Riskset}(r^\prime,\Deltav_{i,\Riskset}):=\sum_{r=r^\prime}^{Q-1}
        g_{\Riskset}(r,\Deltav_{i,\Riskset})\leq
        \\
        &\left\{
        \begin{aligned}
                &0,r^\prime=Q
                \\
                &
                (Q-r^\prime)\frac{\abcstt\gamma K}{(\frac{\Deltav_{i,\Riskset}}{3\sqrt{\ln(1/\omega_\rmv^\prime)}})^2}
                \left(2+\frac{1}{\ln(1/\omega_{\rmv\rmv^\prime})}\left(\ln\ln_+\frac{1}{\ln(1/\omega_{\rmv\rmv^\prime})(\frac{\Deltav_{i,\Riskset}}{3\sqrt{\ln(1/\omega_\rmv^\prime)}})^2}+\upbdcstt\right)\right),
                \left\lfloor\frac{\omega_\rmsum\cdot\Deltav_{i,\Riskset} }{\sqrt{\ln(1/\omega_\rmv^\prime)}\bsigma^2}\right\rfloor\geq Q-3
                \\
                &
                \frac{2\abcstt\gamma K}{(r^\prime-1)(\frac{\bsigma^2}
                {3\omega_{\rmsum}})^2}\left(2+\frac{1}{\ln(1/\omega_{\rmv\rmv^\prime})}\left(\ln\ln_+\frac{1}{\ln(1/\omega_{\rmv\rmv^\prime})(\frac{\Deltav_{i,\Riskset}}{3\sqrt{\ln(1/\omega_\rmv^\prime)}})^2}+\upbdcstt\right)\right), 
                 \left\lfloor\frac{\omega_\rmsum\cdot\Deltav_{i,\Riskset} }{\sqrt{\ln\frac{1}{\omega_\rmv^\prime}}\bsigma^2}\right\rfloor+2\leq r^\prime<Q-1
                \\
                &
                \frac{3\abcstt\gamma K}{(\frac{\Deltav_{i,\Riskset}}{3\sqrt{\ln(1/\omega_\rmv^\prime)}})^2}
                \left(2+\frac{1}{\ln(1/\omega_{\rmv\rmv^\prime})}\left(\ln\ln_+\frac{1}{\ln(1/\omega_{\rmv\rmv^\prime})(\frac{\Deltav_{i,\Riskset}}{3\sqrt{\ln(1/\omega_\rmv^\prime)}})^2}+\upbdcstt\right)\right),
                \left\lfloor\frac{\omega_\rmsum\cdot\Deltav_{i,\Riskset} }{\sqrt{\ln(1/\omega_\rmv^\prime)}\bsigma^2}\right\rfloor=0,r^\prime=1
                \\
                &
                \abcstt\gamma K\left(
                \frac{3}{\frac{\bsigma^2}{3\omega_{\rmsum}}\cdot
                \frac{\Deltav_{i,\Riskset}}{3\sqrt{\ln(1/\omega_\rmv^\prime)}}}
                -\frac{r^\prime-1}{(\frac{\Deltav_{i,\Riskset}}{3\sqrt{\ln(1/\omega_\rmv^\prime)}})^2}
                \right)
                \left(2+\frac{1}{\ln(1/\omega_{\rmv\rmv^\prime})}\left(\ln\ln_+\frac{1}{\ln(1/\omega_{\rmv\rmv^\prime})(\frac{\Deltav_{i,\Riskset}}{3\sqrt{\ln(1/\omega_\rmv^\prime)}})^2}+\upbdcstt\right)\right),
                \text{else}
            \end{aligned}\right.
    \end{align} 

    For $\Safeset^c\cap\Subopt$: for each $i\in E$, there is a changing point $\left\lfloor\frac{\omega_\rmsum\cdot\bDelta_{i,\Safeset^c\cap\Subopt} }{\bsigma^2/3}\right\rfloor+2$, thus
    \begin{align}\label{equ:hunsafesub}
            &h_{\Safeset^c\cap\Subopt}(r^\prime,\bDelta_{i,\Safeset^c\cap\Subopt}):=\sum_{r=r^\prime}^{Q-1}g_{\Safeset^c\cap\Subopt}(r,\bDelta_{i,\Safeset^c\cap\Subopt})
            \\
            &\leq
            \left\{
            \begin{aligned}
                &0,r^\prime=Q
                \\
                &
                (Q-r^\prime)\frac{\abcstt\gamma K}{(\bDelta_{i,\Safeset^c\cap\Subopt})^2}\left(2+\frac{1}{\ln(1/\omega_{\max})}\left(\ln\ln_+\frac{1}{\ln(1/\omega_{\max})(\bDelta_{i,\Safeset^c\cap\Subopt})^2}+\upbdcstt\right)\right),\left\lfloor\frac{\omega_\rmsum\cdot\bDelta_{i,\Safeset^c\cap\Subopt} }{\bsigma^2/3}\right\rfloor\geq Q-3
                \\
                &
                \frac{2\abcstt\gamma K}{(r^\prime-1)(\frac{\bsigma^2}{3\omega_{\rmsum}})^2}\left(2+\frac{1}{\ln(1/\omega_{\max})}\left(\ln\ln_+\frac{1}{\ln(1/\omega_{\max})(\bDelta_{i,\Safeset^c\cap\Subopt})^2}+\upbdcstt\right)\right), 
                 \left\lfloor\frac{\omega_\rmsum\cdot\bDelta_{i,\Safeset^c\cap\Subopt} }{\bsigma^2/3}\right\rfloor+2\leq r^\prime <Q-1
                \\
                &
                 \frac{3\abcstt\gamma K}{(\bDelta_{i,\Safeset^c\cap\Subopt})^2}\left(2+\frac{1}{\ln(1/\omega_{\max})}\left(\ln\ln_+\frac{1}{\ln(1/\omega_{\max})(\bDelta_{i,\Safeset^c\cap\Subopt})^2}+\upbdcstt\right)\right),
                \left\lfloor\frac{\omega_\rmsum\cdot\bDelta_{i,\Safeset^c\cap\Subopt} }{\bsigma^2/3}\right\rfloor=0,r^\prime=1
                \\
                &
                \abcstt\gamma K\left(
                 \frac{3}{\frac{\bsigma^2}{3\omega_{\rmsum}}\bDelta_{i,\Safeset^c\cap\Subopt}}
                 -
                 \frac{r^\prime-1}{(\bDelta_{i,\Safeset^c\cap\Subopt})^2}
                 \right)
                 \left(2+\frac{1}{\ln(1/\omega_{\max})}\left(\ln\ln_+\frac{1}{\ln(1/\omega_{\max})(\bDelta_{i,\Safeset^c\cap\Subopt})^2}+\upbdcstt\right)\right),
                \text{otherwise}
            \end{aligned}\right.
    \end{align}

\end{proof}
 
    \subsection{Proofs of Theorem~\ref{thm:upbd} and Theorem~\ref{thm:problemindependent_upbd}}\label{sec:proofUppBd}

\thmUpperBdProbDep*
        \begin{proof}[Proof of Theorem~\ref{thm:upbd}]
             According to Lemma~\ref{lem:upbd_decomposition}, Lemma~\ref{lem:subreg} and Lemma~\ref{lem:safereg}, and take $\omega_\mu=\omega_\rmv^\prime=\frac{1}{T^2}$ and $\omega_\rmv=\frac{\delta_T}{T^2}$, the expected regret of $\Tp$ phases $\mathbb{E}[\rmR(\Tp)]$ can be upper bounded as
        \begin{align}\label{equ:proof_upbd1}
             \mathbb{E}[\rmR(\Tp)]
             &\leq
             \mathbb{E}[\rmR_1(T^\prime)|\mathcal{E}]+\mathbb{E}[\rmR_2(T^\prime)|\mathcal{E}]+\rmR_3(T)
             \\
             &\leq 
             O\left(\sum_{i\in E\setminus S^\star}
            \frac{K}{\Delta_{i,\Safeset\cap\Subopt,\min}} 
            \ln\frac{1}{\omega_\mu}+
            \sum_{i\in E}
            \frac{c_i K}{\Delta_{i,\Safeset^c\cap\Subopt,\min}} \ln\frac{1}{\omega_\rmv^\prime}\right)
            +2\mu^\star \left[\Tp\cdot (r^\prime-1)+H(r^\prime,\instance)\right] 
            \\
            &\quad
            +2\mu^\star  L+2\mu^\star TL\left(\xi(\omega_\mu)+2\xi(\omega_v)+2\xi(\omega_v^\prime)\right)
            \\
             &
             \leq 
             O\left(\sum_{i\in E\setminus S^\star}
            \frac{K}{\Delta_{i,\Safeset\cap\Subopt,\min}} 
            \ln T+
            \sum_{i\in E}
            \frac{c_i K}{\Delta_{i,\Safeset^c\cap\Subopt,\min}} \ln T\right)
            +2\mu^\star H(1,\instance)
            \\
            &\quad
            +2\mu^\star  L+2\mu^\star TL\left(3\xi(1/T^2)+2\xi(\delta_T/T^2)\right)
            \\
            &=\Regsub(T)+\Regsafe(T)+\Regfail(T)
        \end{align}
    On the other hand, the high-probability regret \eqref{equ:high_pro_regret} can be na\"ive bounded as
    \begin{align}
        \hat{\rmR}(\Tp)&=\sum_{\phase=1}^{\Tp}\sum_{r=1}^{n_\phase}(\mu^\star -\mu_{A_{\phase,r}})
        \\
        &\leq
        \sum_{\phase=1}^{\Tp}\sum_{r=1}^{n_\phase}\mu^\star 
        \\
        &=
        T\mu^\star .
    \end{align}
    Therefore, we have 
    \begin{align}\label{equ:proof_upbd2}
        \mathbb{E}[\rmR(\Tp)]
            &\leq \mathbb{E}[\hat{\rmR}(\Tp)|\mathcal{E}]+\Regfail(T)
            \\
             &\leq T\mu^\star +\Regfail(T)
    \end{align}
    \eqref{equ:proof_upbd1} and \eqref{equ:proof_upbd2} give the final upper bound with $T$ time steps.
             
     \end{proof}

        \thmproblemindependentupbd*
       \begin{proof}[Proof of Theorem~\ref{thm:problemindependent_upbd}]
        We firstly deal with the regret due to suboptimality $\rmR_1(\Tp)$.
        Let $\Deltamu$ be a constant that is to be chosen.
        \begin{align}
            \rmR_1(\Tp)
            &=\sum_{\phase=1}^{\Tp}\mathbbm{1}\{A_\phase\in\Subopt\}\Delta_{A_\phase}
            \\
            &=\sum_{\phase=1}^{\Tp}\mathbbm{1}\{\Delta_{A_\phase}\geq \Delta^\mu\}\mathbbm{1}\{A_\phase\in\Subopt\}\Delta_{A_\phase}
            +
            \sum_{\phase=1}^{\Tp}\mathbbm{1}\{\Delta_{A_\phase}<\Deltamu\}\mathbbm{1}\{A_\phase\in\Subopt\}\Delta_{A_\phase}
            \\
            &\leq
            \sum_{\phase=1}^{\Tp}\mathbbm{1}\{\Delta_{A_\phase}\geq \Delta^\mu\}\mathbbm{1}\{A_\phase\in\Subopt\}\Delta_{A_\phase}
            +
            T \cdot \Deltamu
        \end{align}
        where the second term makes use of the fact that $\Tp\leq T$ w.p. $1$.

        The first term can be upper bounded by adopting the proof of Lemma~\ref{lem:safereg} with the constraint that $\Delta_{A_\phase}\geq \Delta^\mu$. Thus,
        for $i\in E\setminus S^\star$,  
            $\Delta_{i,\Safeset\cap\Subopt,\min}\geq \Deltamu.$
        For $i\in E$,
        \begin{align}
            \Delta_{i,\Safeset^c\cap\Subopt,\min}&\geq\Deltamu
            \\
            \bDelta_{i,\Safeset^c\cap\Subopt}^\prime&\geq \max
            \{\bomega\Deltamu,\Deltav/3\}
            =
            \max
            \{\Deltamu,\Deltav/3\}
            =:\Delta^{\mu\rmv}
            \\
            c_i&\leq \frac{1}{\bomega^2}=1
        \end{align}
        where $\bomega:=\sqrt{
            \frac{\ln\frac{1}{\omega_\rmv^\prime}}{\ln\frac{1}{\omega_\mu}}
            }=1
            $. 
        The regret due to suboptimality can be upper bounded by
        \begin{align}
            \sum_{\phase=1}^{\Tp}\mathbbm{1}\{\Delta_{A_\phase}\geq \Delta^\mu\}\mathbbm{1}\{A_\phase\in\Subopt\}\Delta_{A_\phase}
            \leq&
            \sum_{i\in E\setminus S^\star}
            \frac{2\abcstt\gamma K}{\Delta_{i,\Safeset\cap\Subopt,\min}} 
            \left(2\ln\frac{1}{\omega_\mu}+\ln\ln_+\frac{1}{\Delta_{i,\Safeset\cap\Subopt,\min}^2}+\upbdcstt\right)
            \\
            &+
            \sum_{i\in E}
            \frac{2c_i\cdot \abcstt\gamma K}{\Delta_{i,\Safeset^c\cap\Subopt,\min}} 
            \left(2\ln\frac{1}{\omega_\rmv^\prime}+\ln\ln_+\frac{1}{(\bDelta_{i,\Safeset^c\cap\Subopt}^\prime)^2}+\upbdcstt\right)
            \\
            \leq&
            \sum_{i\in E\setminus S^\star}
            \frac{2\abcstt\gamma K}{\Deltamu} 
            \left(2\ln\frac{1}{\omega_\mu}+\ln\ln_+\frac{1}{(\Deltamu)^2}+\upbdcstt\right)
            \\
            &+
            \sum_{i\in E}
            \frac{2\abcstt\gamma K}{\Deltamu} 
            \left(2\ln\frac{1}{\omega_\rmv^\prime}+\ln\ln_+\frac{1}{(\Delta^{\mu\rmv})^2}+\upbdcstt\right)
            \\
            \leq&
            L\cdot
            \frac{8\abcstt\gamma K}{\Deltamu} 
            \left(2\ln T+\ln\ln_+\frac{1}{(\Deltamu)^2}+\upbdcstt\right)
        \end{align}
    By taking $\Deltamu=\sqrt{\frac{KL\ln (TL)}{T}}$, the regret due to suboptimality is bounded by
    \begin{align}
            \rmR_1(\Tp)
            &\leq
            \sum_{\phase=1}^{\Tp}\mathbbm{1}\{\Delta_{A_\phase}\geq \Delta^\mu\}\mathbbm{1}\{A_\phase\in\Subopt\}\Delta_{A_\phase}
            +
            T \cdot \Deltamu
            \\
            &\leq 
            16\abcstt\gamma\sqrt{KLT\ln (TL)}+8\abcstt \gamma\sqrt{\frac{KLT}{\ln T}}\ln\ln_+\frac{T}{KL\ln T}+8\abcstt \gamma\sqrt{\frac{KLT}{\ln T}}\upbdcstt+\sqrt{KLT\ln (TL)}
            \\
            &=O(\sqrt{KLT\ln (TL)})
            \\
            &=O(\sqrt{KLT\ln T}).
        \end{align}
    where we utilize $T\geq L$.
    
    We then cope with the regret due to safeness-checking $\rmR_2(\Tp).$ According to Lemma~\ref{lem:safereg}, we only need to upper bound $H(1,\instance)$, i.e., the $h$ functions defined in \eqref{equ:hAstar}, \eqref{equ:hsafesub}, \eqref{equ:hrisk} and \eqref{equ:hunsafesub}.

    For $h_{S^\star}(1,\Deltav_{S^\star})$:
    \begin{align}\label{equ:indendent_hAstar}
        h_{S^\star}(1,\Deltav_{S^\star})
        &\leq h_{S^\star}(1,\Deltav)
        \\
        &
        =
        \frac{18\cdot\abcstt\gamma K}{(\Deltav)^2}
        \left(2\ln\frac{1}{\omega_\rmv}+\ln\ln_+\frac{1}{(\Deltav)^2}+\upbdcstt\right)
        \\
        &=O\left(\frac{ K}{(\Deltav)^2}\ln\frac{1}{\omega_\rmv}\right)
        \\
        &=O\left(\frac{ K}{(\Deltav)^2}\ln\frac{T}{\delta}\right)
    \end{align}
    For $h_{\Safeset\cap\Subopt}(1,\bDelta_{i,\Safeset\cap\Subopt})$, define
    $\Delta^{\mu\rmv}
    :=\max\left\{\frac{\Deltamu}{\sqrt{\ln(1/\omega_\mu)}}
    ,
    \frac{\Deltav}{3\sqrt{\ln(1/\omega_\rmv)}}\right\}
    =
    \max\left\{\sqrt{\frac{KL}{T}}
    ,
    \frac{\Deltav}{3\sqrt{\ln(TL/\delta)}}\right\}$. 
    We have $\bDelta_{i,\Safeset\cap\Subopt}\geq \Delta^{\mu\rmv}$.
    The threshold 
    $\left\lfloor\frac{\Delta^{\mu\rmv}}{\frac{\bsigma^2}{3\sqrt{\ln(1/\omega_\rmv)}}}\right\rfloor
    =
     \left\lfloor
     \frac{
     \max\left\{ 3\sqrt{\frac{KL\ln(TL/\delta)}{T}}
    ,
    \Deltav
    \right\}}
    {\bsigma^2}
    \right\rfloor.
    $
    \begin{align}\label{equ:indendent_hsafesub}
        &
        h_{\Safeset\cap\Subopt}(1,\bDelta_{i,\Safeset\cap\Subopt})
        \leq
        h_{\Safeset\cap\Subopt}(1,\Delta^{\mu\rmv})
        \\
        &=
        \left\{\begin{aligned}
            &(Q-1)\frac{\abcstt\gamma K}{(\Delta^{\mu\rmv})^2}
                \left(2
                +\frac{1}{\ln(1/\omega_{\mu\rmv})}\left(\ln\ln_+\frac{1}{\ln(1/\omega_{\mu\rmv})(\Delta^{\mu\rmv})^2}+\upbdcstt\right)\right),
                \left\lfloor\frac{\Delta^{\mu\rmv}}{\frac{\bsigma^2}{3\sqrt{\ln(1/\omega_\rmv)}}}\right\rfloor\geq Q-3
            \\
            &\abcstt\gamma K\frac{4}{\frac{\bsigma^2}{3\sqrt{\ln(1/\omega_\rmv)}}\cdot\Delta^{\mu\rmv}}
            \left(2
            +\frac{1}{\ln(1/\omega_{\mu\rmv})}\left(\ln\ln_+\frac{1}{\ln(1/\omega_{\mu\rmv})(\Delta^{\mu\rmv})^2}+\upbdcstt\right)\right),
            0<\left\lfloor\frac{\Delta^{\mu\rmv} }{\frac{\bsigma^2}{3\sqrt{\ln(1/\omega_\rmv)}}}\right\rfloor<Q-3
            \\
             &\frac{3\cdot\abcstt\gamma K}{(\Delta^{\mu\rmv})^2}
            \left(2
            +\frac{1}{\ln(1/\omega_{\mu\rmv})}\left(\ln\ln_+\frac{1}{\ln(1/\omega_{\mu\rmv})(\Delta^{\mu\rmv})^2}+\upbdcstt\right)\right),
            \left\lfloor\frac{\Delta^{\mu\rmv} }{\frac{\bsigma^2}{3\sqrt{\ln(1/\omega_\rmv)}}}\right\rfloor=0
        \end{aligned}\right.
        \\
        &=
        \left\{\begin{aligned}
            &O\left((Q-1)\frac{K}{(\Delta^{\mu\rmv})^2}\right),
               && \left\lfloor\frac{\Delta^{\mu\rmv}}{\frac{\bsigma^2}{3\sqrt{\ln(1/\omega_\rmv)}}}\right\rfloor\geq Q-3
            \\
            &O\left(\frac{K}{\frac{\bsigma^2}{3\sqrt{\ln(1/\omega_\rmv)}}\cdot\Delta^{\mu\rmv}}\right),
           && 0<\left\lfloor\frac{\Delta^{\mu\rmv} }{\frac{\bsigma^2}{3\sqrt{\ln(1/\omega_\rmv)}}}\right\rfloor<Q-3
            \\
             &O\left(\frac{ K}{(\Delta^{\mu\rmv})^2}\right),
            &&\left\lfloor\frac{\Delta^{\mu\rmv} }{\frac{\bsigma^2}{3\sqrt{\ln(1/\omega_\rmv)}}}\right\rfloor=0
        \end{aligned}\right.
    \end{align}
    In particular, in the asymptotic case where $T\to \infty$ and $\ln\frac{1}{\delta}=\ln\frac{1}{\delta_T}=o(T^b),\forall b>0$ (this includes the scenario where $\delta$ is fixed with respect to $T$),  we have 
    \begin{align}
    h_{\Safeset\cap\Subopt}(1,\Delta^{\mu\rmv})
    =
    O\left(\frac{ K}{(\Delta^{\mu\rmv})^2}\right)
    =
    O\left(\frac{ K}{(\Deltav)^2}\ln\frac{T}{\delta}\right)
    \end{align}
    
For $h_{\Riskset}(1,\Deltav_{i,\Riskset})$, we have $\Deltav_{i,\Riskset}\geq\Deltav$. Furthermore, $\omega_\rmsum=\sqrt{\ln (TL)}+\sqrt{\ln (TL/\delta)}$ and the changing point $\left\lfloor\frac{\omega_\rmsum\cdot\Deltav }{\sqrt{\ln(1/\omega_\rmv^\prime)}\bsigma^2}\right\rfloor
=
\left\lfloor\frac{(\sqrt{\ln (TL)}+\sqrt{\ln (TL/\delta)})\cdot\Deltav }{\sqrt{\ln(TL)}\bsigma^2}\right\rfloor$.

Thus
    \begin{align}\label{equ:indendent_hrisk}
        &h_{\Riskset}(1,\Deltav_{i,\Riskset})
        \leq
        h_{\Riskset}(r^\prime,\Deltav)
        \\
        &=\left\{
        \begin{aligned}
                &
                (Q-1)\frac{\abcstt\gamma K}{(\frac{\Deltav}{3\sqrt{\ln(1/\omega_\rmv^\prime)}})^2}
                \left(2+\frac{1}{\ln(1/\omega_{\rmv\rmv^\prime})}\left(\ln\ln_+\frac{1}{\ln(1/\omega_{\rmv\rmv^\prime})(\frac{\Deltav}{3\sqrt{\ln(1/\omega_\rmv^\prime)}})^2}+\upbdcstt\right)\right),
                \left\lfloor\frac{\omega_\rmsum\cdot\Deltav }{\sqrt{\ln(1/\omega_\rmv^\prime)}\bsigma^2}\right\rfloor\geq Q-3
                \\
                &
                \abcstt\gamma K
                \frac{3}{\frac{\bsigma^2}{3\omega_{\rmsum}}\cdot
                \frac{\Deltav_{i,\Riskset}}{3\sqrt{\ln(1/\omega_\rmv^\prime)}}}
                \left(2+\frac{1}{\ln(1/\omega_{\rmv\rmv^\prime})}\left(\ln\ln_+\frac{1}{\ln(1/\omega_{\rmv\rmv^\prime})(\frac{\Deltav_{i,\Riskset}}{3\sqrt{\ln(1/\omega_\rmv^\prime)}})^2}+\upbdcstt\right)\right),
                0<\left\lfloor\frac{\omega_\rmsum\cdot\Deltav }{\sqrt{\ln(1/\omega_\rmv^\prime)}\bsigma^2}\right\rfloor\leq Q-3
                \\
                &
                \frac{3\abcstt\gamma K}{(\frac{\Deltav}{3\sqrt{\ln(1/\omega_\rmv^\prime)}})^2}
                \left(2+\frac{1}{\ln(1/\omega_{\rmv\rmv^\prime})}\left(\ln\ln_+\frac{1}{\ln(1/\omega_{\rmv\rmv^\prime})(\frac{\Deltav}{3\sqrt{\ln(1/\omega_\rmv^\prime)}})^2}+\upbdcstt\right)\right),
                \left\lfloor\frac{\omega_\rmsum\cdot\Deltav }{\sqrt{\ln(1/\omega_\rmv^\prime)}\bsigma^2}\right\rfloor=0
            \end{aligned}\right.
            \\
        &=\left\{
        \begin{aligned}
                &
                O\left((Q-1)\frac{K}{(\frac{\Deltav}{3\sqrt{\ln(1/\omega_\rmv^\prime)}})^2}\right),
                &&\left\lfloor\frac{\omega_\rmsum\cdot\Deltav }{\sqrt{\ln(1/\omega_\rmv^\prime)}\bsigma^2}\right\rfloor\geq Q-3
                \\
                &
                O\left(
                \frac{K}{\frac{\bsigma^2}{3\omega_{\rmsum}}\cdot
                \frac{\Deltav_{i,\Riskset}}{3\sqrt{\ln(1/\omega_\rmv^\prime)}}}\right),
                &&0<\left\lfloor\frac{\omega_\rmsum\cdot\Deltav }{\sqrt{\ln(1/\omega_\rmv^\prime)}\bsigma^2}\right\rfloor\leq Q-3
                \\
                &=
                O\left(\frac{K}{(\frac{\Deltav}{3\sqrt{\ln(1/\omega_\rmv^\prime)}})^2}\right),
                &&\left\lfloor\frac{\omega_\rmsum\cdot\Deltav }{\sqrt{\ln(1/\omega_\rmv^\prime)}\bsigma^2}\right\rfloor=0
            \end{aligned}\right.
    \end{align} 
    In particular, in the asymptotic case where $T\to \infty$ and $\ln\frac{1}{\delta}=\ln\frac{1}{\delta_T}=o(T^b),\forall b>0$, we have 
    \begin{align}
    h_{\Riskset}(1,\Delta^{\rmv})
    =
    O\left(\frac{ QK}{(\Delta^{\rmv})^2}\ln\frac{1}{\omega_\rmv^\prime}\right)
    =
    O\left(\frac{ QK}{(\Delta^{\rmv})^2}\ln T\right)
    \end{align}

    For $\Safeset^c\cap\Subopt$: define $\bDelta^{\mu\rmv}
    :=\max\left\{\frac{\Deltamu}{\sqrt{\ln(1/\omega_\mu)}}
    ,
    \frac{\Deltav}{3\sqrt{\ln(1/\omega_\rmv^\prime)}}
    \right\}
    =
    \max\left\{\sqrt{\frac{KL}{T}}
    ,
    \frac{\Deltav}{3\sqrt{\ln(TL)}}
    \right\}
    .$
    We have
    $\bDelta_{i,\Safeset^c\cap\Subopt}\geq \bDelta^{\mu\rmv} $ and the changing point 
    $\left\lfloor\frac{\omega_\rmsum\cdot\bDelta^{\mu\rmv} }{\bsigma^2/3}\right\rfloor
    =
    \left\lfloor\frac{(\sqrt{\ln (TL)}+\sqrt{\ln (TL/\delta)})\cdot\Delta^{\mu\rmv} }{\bsigma^2/3}\right\rfloor$
    thus
    \begin{align}\label{equ:indendent_hunsafesub}
            &h_{\Safeset^c\cap\Subopt}(1,\bDelta_{i,\Safeset^c\cap\Subopt})
            \leq
            h_{\Safeset^c\cap\Subopt}(1,\bDelta^{\mu\rmv})
            \\
            &=
            \left\{
            \begin{aligned}
                &
                (Q-1)\frac{\abcstt\gamma K}{(\bDelta^{\mu\rmv})^2}\left(2+\frac{1}{\ln(1/\omega_{\max})}\left(\ln\ln_+\frac{1}{\ln(1/\omega_{\max})(\bDelta^{\mu\rmv})^2}+\upbdcstt\right)\right),
                \left\lfloor\frac{\omega_\rmsum\cdot\bDelta^{\mu\rmv} }{\bsigma^2/3}\right\rfloor\geq Q-3
                \\
                &
                \abcstt\gamma K\left(
                 \frac{3}{\frac{\bsigma^2}{3\omega_{\rmsum}}\bDelta^{\mu\rmv}}
                 \right)
                 \left(2+\frac{1}{\ln(1/\omega_{\max})}\left(\ln\ln_+\frac{1}{\ln(1/\omega_{\max})(\bDelta^{\mu\rmv})^2}+\upbdcstt\right)\right),
                0<\left\lfloor\frac{\omega_\rmsum\cdot\bDelta^{\mu\rmv}}{\bsigma^2/3}\right\rfloor<Q-3
                \\
                &
                 \frac{3\abcstt\gamma K}{(\bDelta^{\mu\rmv})^2}\left(2+\frac{1}{\ln(1/\omega_{\max})}\left(\ln\ln_+\frac{1}{\ln(1/\omega_{\max})(\bDelta^{\mu\rmv})^2}+\upbdcstt\right)\right),
                \left\lfloor\frac{\omega_\rmsum\cdot\bDelta^{\mu\rmv} }{\bsigma^2/3}\right\rfloor=0,r^\prime=1
            \end{aligned}\right.
            \\
            & =
            \left\{
            \begin{aligned}
                &
                O\left((Q-1)\frac{K}{(\bDelta^{\mu\rmv})^2}\right),
                &&\left\lfloor\frac{\omega_\rmsum\cdot\bDelta^{\mu\rmv} }{\bsigma^2/3}\right\rfloor\geq Q-3
                \\
                &
                O\left(
                 \frac{K}{\frac{\bsigma^2}{3\omega_{\rmsum}}\cdot\bDelta^{\mu\rmv}}
                 \right),
                &&0<\left\lfloor\frac{\omega_\rmsum\cdot\bDelta^{\mu\rmv} }{\bsigma^2/3}\right\rfloor<Q-3
                \\
                &
                 O\left(\frac{K}{(\bDelta^{\mu\rmv})^2}\right),
                &&\left\lfloor\frac{\omega_\rmsum\cdot\bDelta^{\mu\rmv} }{\bsigma^2/3}\right\rfloor=0,r^\prime=1
            \end{aligned}\right.
    \end{align}
    In particular, in the asymptotic case where $T\to \infty$ and $\ln\frac{1}{\delta}=\ln\frac{1}{\delta_T}=o(T^b),\forall b>0$, we have 
    \begin{align}
    h_{\Safeset^c\cap\Subopt}(1,\bDelta^{\mu\rmv})
    =
    O\left(\frac{ QK}{(\Deltav)^2}\ln T\right)
    \end{align}

    Lastly,
    \begin{align}
        \rmR_3(T)
        &= 
        2\mu^\star  L+2\mu^\star TL\left(\xi(\omega_\mu)+2\xi(\omega_v)+2\xi(\omega_v^\prime)\right)
        \\
        &\leq
        2KL+K\delta+2KTL\cdot 4 \cdot\frac{2+\epsilon}{\epsilon}\big(\frac{\frac{1}{T^2}}{\ln (1+\epsilon)}\big)^{1+\epsilon}
        \\
        &\leq
        2KL+K\delta+4K\cdot\frac{2+\epsilon}{\epsilon}\big(\frac{1}{\ln (1+\epsilon)}\big)^{1+\epsilon}
        \\
        &= O(1)
    \end{align}
    where we utilize $T>L$ and the $O$ notation refers to the fact that the preceding term is bounded as a function of $T$.

    Note that $\mu^\star \leq K$, so $T\mu^\star +K\delta\leq TK+K\delta$.

    \textbf{In conclusion}, according to Theorem~\ref{thm:upbd}, for any $T>L$, the problem-independent upper bound is the minimum of $TK+K\delta$ and
    \begin{align}
        &O(\sqrt{KLT\ln T})
        +K\left( \sum_{i\in S^\star}h_{S^\star}(1,\Deltav)
        +\sum_{i\in E}h_{\Safeset\cap\Subopt}(1,\Delta^{\mu\rmv})
        +\sum_{i\in E}h_{\Riskset}(1,\Deltav)
        +\sum_{i\in E}h_{\Safeset^c\cap\Subopt}(1,\bDelta^{\mu\rmv})\right)
        \\
        &\leq 
        O(\sqrt{KLT\ln T})
        +K\left( Kh_{S^\star}(1,\Deltav)
        +Lh_{\Safeset\cap\Subopt}(1,\Delta^{\mu\rmv})
        +Lh_{\Riskset}(1,\Deltav)
        +Lh_{\Safeset^c\cap\Subopt}(1,\bDelta^{\mu\rmv})\right)
        \\
        &=
        O(\sqrt{KLT\ln T})
        +O\left(\frac{ K^3}{(\Deltav)^2}\ln\frac{T}{\delta}\right)
        +KL\left(h_{\Safeset\cap\Subopt}(1,\Delta^{\mu\rmv})
        +h_{\Riskset}(1,\Deltav)
        +h_{\Safeset^c\cap\Subopt}(1,\bDelta^{\mu\rmv})\right)
    \end{align}
    where the $h$ functions are defined in \eqref{equ:indendent_hsafesub}, \eqref{equ:indendent_hrisk} and \eqref{equ:indendent_hunsafesub}.
    In the asymptotic case where $T\to\infty$ and $\ln\frac{1}{\delta}=\ln\frac{1}{\delta_T}=o(T^b),\forall b>0$ (this includes the scenario where $\delta$ is fixed with respect to $T$), the asymptotic problem-independent upper bound is
    \begin{align}
        &
        O(\sqrt{KLT\ln T})
        +K\left(\sum_{i\in E}O\left(\frac{ K}{(\Deltav)^2}\ln\frac{T}{\delta}\right)
        +O\sum_{i\in E}\left(\frac{ QK}{(\Delta^{\rmv})^2}\ln T\right)\right)
        \\
        &= 
        O(\sqrt{KLT\ln T})
        +O\left(\frac{ LK^2}{(\Deltav)^2}\ln\frac{1}{\delta}\right)
    \end{align}
    where we utilize $\sqrt{T\ln T}\geq \frac{QK^2}{(\Deltav)^2}\ln T$ when $T$ is sufficiently large.      
    \end{proof}

\section{Proofs of the Lower Bounds}\label{sec:lower_bd}
\subsection{Preliminaries and the Impossibility Result}\label{sec:proof_LwBd_imp}
Let $\mathrm{KL}(\nu,\nu^\prime)$ denote the KL divergence between distributions $\nu$ and $\nu^\prime$, and 
	\begin{align*}
		d(x, y) := x \ln\bigg( \frac{x}{ y} \bigg)+(1-x) \ln\bigg( \frac{1-x}{1-y} \bigg)
	\end{align*}
	denote the Kullback--Leibler (KL) divergence between the Bernoulli distributions $\mathrm{Bern}(x)$ and  $\mathrm{Bern}(y)$.

 \begin{restatable}[Pinsker's and reverse Pinsker's Inequality]{lem}{lemPinsker}
\label{lem:pinsker}
Consider two probability mass functions $P_X, P_Y$ defined on the same discrete probability space ${\cal A} \subset [0, 1]$. The following inequalities hold:
\begin{align*}
\left|\mathbb{E}_{X\sim P_X}[X] - \mathbb{E}_{Y\sim P_Y}[Y]\right| & \leq \delta(P_X, P_Y) \leq  \sqrt{\frac{1}{2}\mathrm{KL}(P_X, P_Y)}
\le \frac{1}{ \sqrt{\alpha_Y} } \cdot  \delta(P_X,P_Y) ,
\nonumber.
\end{align*}
where
$\delta(P_X,P_Y) := \sup_{A\subseteq {\cal A} } \left\{\sum_{a\in A}P_X(a) - \sum_{a\in A}P_Y(a)\right\}
	=\frac{1}{2} \sum_{a\in  \mathcal{A}} |P_X(a) - P_Y(a)|$
	is the total variational distance,
	and  
	$\alpha_Y := \min_{a\in \mathcal{A}: P_Y(a)>0} Q(a) \vphantom{\bigg(}
	$.
\end{restatable}

\begin{restatable}[Lemma 1 in~\citet{Kaufmann2016}]{lem}{lemKLdecompPinsker} \label{lem:lb_KLdecomp}
    Assume the distributions under instance $\instance_1=(E,\mathcal{A}_K,\nu^{(1)},\bsigma^2)$ and instance $\instance_2=(E,\mathcal{A}_K,\nu^{(2)},\bsigma^2)$ are mutually absolutely continuous.  Given time budget $T$,
		\begin{align*}
			& 
			\sum_{i=1}^L \mathbb{E}_{ \instance_1} [N_{i}(T)] \cdot \rmKL( \nu_i^{(1)}, \nu_i^{(2)} )
			\ge 
			\sup_{\mathcal{E} \in \mathcal{H}_T^{ \instance_1}  } 
			d\big(    \mathbb{P}_{ \instance_1} (\mathcal{E})  ,~   \mathbb{P}_{ \instance_2} (\mathcal{E})   \big).
			\nonumber 
	\end{align*}
    where $N_i(t)$ denotes the number of time steps item $i$ is selected up to and including time step  $t$ and $\mathcal{H}_T^{ \instance_1}$ is all the possible events generated by instance $\instance_1$ and algorithm $\pi$ with $T$ time steps.
\end{restatable}

\begin{restatable}{lem}{lemSafeCheck}
    \label{lem:safe_check}
    Let solution $S$ containing $|S|=m (q<m\leq K)$ items be a safe solution under instance $\instance_1=(E,\mathcal{A}_K,\nu^{(1)},\bsigma^2)$. Each item in $S$ is i.i.d. with reward distribution $\nu_1$ , mean $\mu_1$ and variance $\sigma_1^2<\bsigma^2$. 
    Define event 
    $\calE_{(t,1)} = \{ S \text{ is identified as safe after time step $t$}\}$,
    $\calE_{(t,2)} = \{ S \text{ is chosen at least once after time step $t$}\}$,
    and $\calE_{(t)} = \calE_{(t,1)} \cap \calE_{(t,2)}$. 
    Assume there exists $\tau \le T$ such that $\bbP_{\instance_1}[ \calE_{(\tau)} ] \ge 1- \delta$ and $\bbP_{\instance_1}[ \calE_{(\tau-1,1)} ] < 1- \delta$.
    If $\tau$ exists, we have
    \begin{align*}
        \sum_{i\in S} \bbE_{\instance_1}[ N_i(\tau) ]  \ge\sup_{ \nu_2 \in E(\nu_1) } \frac{  d(\delta, 1-\delta ) }{ \rmKL( \nu_1, \nu_2 )  }.
    \end{align*}
    Furthermore, 
    \begin{align*}
        \bbE_{\instance_1}[ M(\tau) ] \ge \sup_{ \nu_2 \in E(\nu_1) }\frac{ 1 }{|S|-1} \cdot \frac{  d(\delta, 1-\delta ) }{ \rmKL( \nu_1, \nu_2 )  } : = T(\nu^{(1)}),
    \end{align*}
    where 
     $M(t)$ is the number of times that a solution $S^\prime\subset S$ is sampled up to and include time step $t$ and
    $E(\nu_1) = \{  \nu_2: \text{ the variance associated to } \nu_2 \text{ is larger than } \bar{\sigma}^2/|S| \}$.
   
\end{restatable}
\begin{proof}
     With $\sigma_2^2> \bar{\sigma}^2/|S|$, we construct an alternative instance $\instance_2=(E,\mathcal{A}_K,\nu_2,\bsigma^2)$, under which each item in $S$ is with reward distribution $\nu_2$ , mean $\mu_2$ and variance $\sigma_2^2$, while the distributions of other items remain unchanged.  
    
    Define event 
    $\calE_{(t,1)} = \{ S \text{ is identified as safe after time step $t$}\}$,
    $\calE_{(t,2)} = \{ S \text{ is chosen at least once after time step $t$}\}$,
    and $\calE_{(t)} = \calE_{(t,1)} \cap \calE_{(t,2)}$. 
    Assume there exists $\tau \le T$ such that $\bbP_{\instance_1} [\calE_{(\tau)} ] \ge 1- \delta$ and $\bbP_{\instance_1} [\calE_{(\tau-1,1)}]  < 1- \delta$.
    Since $S$ is unsafe under instance $\instance_2$ and all the solutions chosen $\{S_t\}_{t=1}^T \subset \calA_K$ are safe with probability at least $1 -\delta$, we have $\bbP_{\instance_2}[\calE_{(t)}] < \delta$ for all $t\le T$.
    



We now apply Lemma~\ref{lem:lb_KLdecomp}  to obtain that
\begin{align*}
    & \sum_{i\in S} \bbE_{\instance_1}[ N_i(\tau) ] \cdot \rmKL( \nu_1, \nu_2 ) \ge d(  \bbP_1( \calE_{(\tau)}^c ),  \bbP_2( \calE_{(\tau)}^c ) ) \ge d(\delta, 1-\delta )
    ~\Rightarrow~ 
    \sum_{i\in S} \bbE_{\instance_1}[ N_i(\tau) ]  \ge \frac{  d(\delta, 1-\delta ) }{ \rmKL( \nu_1, \nu_2 )  }.
\end{align*}
Since $\bbP_{\instance_1}[ \calE_{(\tau-1,1)} ] < 1-\delta $, we can select at most $m-1$ items at one time step among the first $\tau$ time steps. Therefore,
\begin{align*}
    \bbE_{\instance_1}[ M(\tau) ] \ge \frac{ 1 }{m-1}  \sum_{i\in P} \bbE_{\instance_1}[ N_i(\tau) ] \ge \frac{ 1 }{m-1} \cdot \frac{ d(\delta, 1-\delta ) }{ \rmKL( \nu_1, \nu_2 )  }.
\end{align*}
\end{proof}

\lemReqOnDeltaT*
\begin{proof}
    The proof is similar to the proof of Lemma~\ref{lem:safe_check}. We consider an alternative instance $\instance_2=(E,\mathcal{A}_K,\nu^{(2)},\bsigma^2)$ with the distributions of the items in the optimal safe solution $S^\star$ changed such that (assume the variances of all the items in $S^\star$ are changed (increased))
    \begin{align}
        \sum_{i\in S^\star}(\sigma_i^{(2)})^2\geq \bsigma^2
    \end{align}
    i.e., under instance $\instance_2$, this solution $S^\star$ is unsafe (thus not optimal safe). The other items remain unchanged.
    
    By a similar argument as the proof of Lemma~\ref{lem:safe_check}, we have
    \begin{align}
        &\sum_{i\in S^\star} \bbE_{\instance_1}[ N_i(\tau) ] \cdot \rmKL( \nu_i^{(1)}, \nu_i^{(2)} )\geq d(\delta, 1-\delta )\\
        \Rightarrow
        &
        \sum_{i\in S^\star} \bbE_{\instance_1}[ N_i(\tau) ] \geq \frac{d(\delta, 1-\delta )}{\min_{i\in S^\star} \rmKL( \nu_i^{(1)}, \nu_i^{(2)} )}
    \end{align}
    So the safeness checking of $S^\star$ will take $\Omega(\ln\frac{1}{2.4\delta_T})=\Omega(T^b)$

    Recall the probably anytime-safe constraint \eqref{constraint}: 
    \begin{align}
		  \mathbb{P}\big[ \forall\, t\in [T], S_t\in\Safeset \big]\geq 1-\delta_T.
	\end{align}
    This indicates at time step $t$, for any solution $S$,  if $\mathbb{P}_{ \mathcal{H}_{t}^{(0)} }[S\in\Safeset]<1- \delta_T$, $S$ will not be selected at this time step. Otherwise, \eqref{constraint} is violated. Therefore, before the safeness of the optimal safe solution $S$ is ascertained, it is not going to be sampled and the instantaneous regret will be lower bounded by $\min_{S\in \Safeset\cap\Subopt}\Delta_S$.

    In conclusion, the regret is at least $\ln\frac{1}{2.4\delta_T}\cdot\min_{S\in \Safeset\cap\Subopt}\Delta_S=\Omega(T^b)$.
\end{proof}


   We derive both the problem-dependent and problem independent lower bounds on the {\em $K$-path semi-bandit} problem. The items in the ground set are divided into $L_0$ paths: $P_1, \ldots, P_{L_0}$, each of which contains $K$ unique items. Path $P_j$ contains items $(j-1)K+1, \ldots, jK$. Without loss of generality, we assume that $L/K$ is an integer.
    A set $S$ is a solution if and only if $S\subset P_j$ for some $j$. In other words, the solution set $\calA_K =\{  S: S\subset P_j,\ \exists j =1, \ldots, L_0 \} $.
    We let $N_i(t)$ denote the number of time steps item $i$ is selected up to and including time step  $t$, $M_j(t)$ denote the number of time steps a safe subset in path $j$ is selected up to and including time step $t$, and $S_j (t) $ denote the time steps when a safe subset in path $j$ is selected,  i.e.,
\begin{align*}
    N_i(t) = \sum_{s=1}^t \mathbbm{1}\{ i \in S_s \},\quad
    M_j(t) = \sum_{s=1}^t \mathbbm{1}\{ S_s \subset P_j, S_s \text{ is safe } \},\quad
    S_j(t) = \{ 1\le s \le t:  S_s \subset P_j, S_s \text{ is safe } \}.
\end{align*}
We let $\Reg[j](t)$ denote the regret accumulated in $S_j(t)$. Since all the chosen solutions are safe with probability at
least $1 -\delta$, we have $\Reg(T) \ge (1-\delta) \cdot \sum_{j=1}^{L_0} \Reg[j](T)$.
In the following, we lower bound the regret accumulated in time steps where safe solutions are chosen.

    We will construct several instances to prove each of the lower bounds (which will be specified in the proof) such that 
	under instance $k$ ($\instance_k=(E,\mathcal{A}_K,\nu^{(k)},\bsigma^2)$), 
    the stochastic reward of items in path $j$~($1\le j \le L_0$) are {\em i.i.d.}, i.e. ,
    $\nu_i^{(k)}=\nu_{P_j}^{(k)},\forall i\in P_j$, 
    which will be specified in each case.
Under instance $k$, we define several other notations as follows: 
	\begin{itemize}
		\item Let $W_i(t)^{(k)}  $ be the random reward of arm $i$ at time step $t$. 
		\item Let $S_t^{(k)}$ be the pulled solution at time step $t$, and $\mathcal{H}_t^{(j)} = \{ (S_s^{(j)},  \{ W_i(s)^{(k)} \}_{s\in S_s^{(k)}}  )\}^t_{s=1}$ be the sequence of selected solutions and observed rewards up to and including time step $t$.  
	\end{itemize} 
	For simplicity, we abbreviate  
	$\bbE_{ \mathcal{H}_{T}^{(k)} }$, $\bbP_{ \mathcal{H}_{T}^{(k)} }$, as  
	$\bbE_k$, $\bbP_k$ respectively.

\subsection{Problem-dependent Lower bound}\label{sec:proof_LwBd_dep}

In order to provide a better understanding of the analysis, we first derive a lower bound with Gaussian distributions (unbounded) in Theorem~\ref{thm:low_bd_prob_dep_gauss}. With the same technique, we derive a lower bound with bounded Bernoulli distributions in Theorem~\ref{thm:low_bd_prob_dep}, which corroborates with our problem setup.
\begin{restatable}[Problem-dependent lower bound for sub-Gaussian instances]{thm}{thmLowBdProbDepGauss}
\label{thm:low_bd_prob_dep_gauss}
Let $\{\delta_T\}_{T=1}^\infty \in o(1)$ be a sequence that satisfies $\ln(1/\delta_T)=o(T^b)$ for all $b>0$. There exists an instance $\instance$, for any $\{\delta_T\}_{T\in\mathbb{N}} $-variance-constrained consistent algorithm $\pi$, the regret is lower bounded by 
\begin{align}
    \Omega&\left(\sum_{i\in E}\frac{\ln T}{\Delta_{i,\Safeset\cap\Subopt,\min}}
    \right)+
    \frac{\mu^\star }{K}
    \cdot\Omega\bigg(\frac{K\cdot\ln(1/\delta_T)}{(\Delta_{S^\star})^2}
    +\sum_{i\in E}\Big(\Psi^\prime_{i,\Safeset\cap\Subopt} 
    +\frac{\ln T}{(\Deltav_{i,\Riskset})^2}
    +\Phi_{i,\Safeset^c\cap\Subopt} \Big)
        \bigg).
\end{align}
\end{restatable}

\begin{proof}
    Given a vanishing sequence $\{\delta_T\}_{T\in\mathbb{N}} $, we consider a fixed $\{\delta_T\}_{T\in\mathbb{N}} $-variance-constrained consistent algorithm $\pi$ on the $K$-path semi-bandit problem. 
    For simplicity, the distributions of the items are assumed to be Gaussian in the proof, but the techniques can be applied to the Bernoulli case and we provide the instance design and the corresponding bound at the end of the proof.
    


    Under instance $\instance_0$ (the base instance), $\sigma^2=\frac{2\bsigma^2}{K}$ (so the absolutely safe solutions contain at most $K/2$ items) and the distributions of the items are 
    \begin{align}
        \nu_i^{(0)}=N(\mu_j^{(0)},(\sigma_i^2)^{(0)})=\left\{
        \begin{aligned}
            &\nu_{P_1}^{(0)}=N(\Delta,\frac{\bsigma^2-\epsilonv}{K}), &&\quad i\in P_1
            \\
            &
            \nu_{P_j}^{(0)}=N(\Delta-\frac{\epsilonmu}{K},\frac{\bsigma^2-\epsilonv}{K}), &&\quad i\in P_j,2\leq j\leq L_1+1
            \\
            &
            \nu_{P_j}^{(0)}=N(\Delta+\frac{\epsilonmu}{K},\frac{\bsigma^2+\epsilonv}{K}), &&\quad i\in P_j,L_1+2\leq j\leq L_2+1
            \\
            &
            \nu_{P_j}^{(0)}=N(\Delta-\frac{\epsilonmu}{K},\frac{\bsigma^2+\epsilonv}{K}), &&\quad i\in P_j,L_2+2\leq j\leq L_0
        \end{aligned}
        \right.
    \end{align}
    where $\epsilonmu<\frac{\Delta}{K}$ and $\epsilonv$ are small positive constants (e.g. $\epsilonmu=\frac{\Delta}{2K},\epsilonv\leq\frac{\bsigma^2}{K^2}$), and $L_1=L_2-L_1=L_0-L_2-1=\frac{L_0-1}{3}$ (assume it is an integer). In this case,
    \begin{itemize}
        \item  path $1$ is an optimal safe path, 
        \item path $2$ to path $L_1+1$ are the safe and suboptimal paths, 
        \item path $L_1+2$ to path $L_2+1$ are the risky paths
        \item path $L_2+2$ to path $L_3+1=L_0$ are the unsafe and suboptimal paths.
    \end{itemize}
    In the following, we will compute the minimum regret yielded from each of the paths.

    \noindent
    \textbf{\underline{Case 1: the optimal safe path $P_1$}}
     
    In order to achieve $o(T^a),\forall a>0$ regret, any algorithm has to identify the safeness of the optimal safe solution $P_1$ and sample $P_1$ $\Omega(T)$ times, otherwise, the regret is linear.
    According to Lemma~\ref{lem:safe_check},
    the expected number of time steps needed for the safeness identification of $P_1$ is lower bounded by
    \begin{align*}
        \bbE_0[ M_1(\tau) ] \ge \sup_{ \nu_{P_1}^{(1)} \in E( \nu_{P_1}^{(0)}) }\frac{ 1 }{K-1} \cdot \frac{  d(\delta_T, 1-\delta_T ) }{ \rmKL( \nu_{P_1}^{(0)}, \nu_{P_1}^{(1)}  )  } : = T(\nu_{P_1}^{(0)})
    \end{align*}
    where 
    $ E( \nu_{P_1}^{(0)})  = \{  \nu_{P_1}^{(1)}: \text{ the variance associated to } \nu_{P_1}^{(1)}\text{ is larger than } \bar{\sigma}^2/K \}$. 
    In particular, we let instance $\instance_1=(E,\mathcal{A}_K,\nu^{(1)},\bsigma^2)$ with
    \begin{align}
        \nu_i^{(1)}=\left\{
        \begin{aligned}
            &\nu_{P_1}^{(1)}=N\left(\Delta,\frac{\bsigma^2+\epsilonv_1}{K}\right), &&\quad i\in P_1
            \\
            &\nu_i^{(0)}, &&\quad i\notin P_1
        \end{aligned}
        \right.
    \end{align}
    where $\epsilonv_1$ is an arbitrarily small constant.
    Thus, we can take supremum over $\epsilonv_1>0$ and have
    \begin{align}
        T(\nu_{P_1}^{(0)})
        \geq 
        \frac{ 1 }{K-1} \cdot \frac{ \ln\frac{1}{2.4\delta_T} }{ \rmKL(N(\Delta,\frac{\bsigma^2-\epsilonv}{K}),N(\Delta,\frac{\bsigma^2}{K})  )  }
        \geq
         \frac{ 1 }{K-1} \cdot \frac{4K^2(\bsigma^2/K)^2}{(\epsilonv)^2}\ln\frac{1}{2.4\delta_T}
         =
         \frac{K(\sigma^2)^2}{(\epsilonv)^2}\ln\frac{1}{2.4\delta_T}
    \end{align}
    When a solution $S\subset P_1$ is chosen before this time step, the instantaneous regret is lower bounded by $\Delta$.
    The accumulative regret from $P_1$ is lower bounded by 
    \begin{align}
        \Reg[1](T)\geq \Delta \cdot T(\nu_{P_1}^{(0)})=\Omega(\frac{K\Delta}{(\Deltav_{P_1})^2}\ln\frac{1}{\delta_T})
    \end{align}

    \noindent
    \textbf{\underline{Case 2: the safe and suboptimal paths}}
    
    For any safe and suboptimal path $P_j (2\leq j\leq L_1+1)$,
    we define instance $\instance_j=(E,\mathcal{A}_K,\nu^{(j)},\bsigma^2)$ with
    \begin{align}
        \nu_i^{(j)}=\left\{
        \begin{aligned}
            &\nu_{P_j}^{(j)}=N(\Delta+\frac{\epsilonmu_j}{K},\frac{\bsigma^2-\epsilonv}{K}), &&\quad i\in P_j
            \\
            &\nu_i^{(0)}, &&\quad i\notin P_j
        \end{aligned}
        \right.
    \end{align}
    where $\epsilonmu_j<\Delta$ is an arbitrarily small constant. So $P_j$ is the optimal safe solution under instance $j$.

    Fix any item $i\in P_j$, consider the event $\mathcal{E}_j=\{N_i(T)\geq\frac{T}{2}\}$, under instance $1$, $\calE_j$ indicates the optimal safe solution $P_1$ is sampled less than $\frac{T}{2}$ times; under instance $j$, $\calE_j^c$ indicates the optimal safe solution $P_j$ is sampled less than $\frac{T}{2}$ times. Therefore,
    \begin{align}\label{equ:low_bd_safesub}
        \Reg_{\instance_0}(T)+\Reg_{\instance_j}(T)
        &\geq
        \frac{T}{2}\epsilonmu \bbP_0[\calE_j]+\frac{T}{2}\epsilonmu_j \bbP_j[\calE_j^c]
        \\
        &\geq
        \frac{T}{2}\min\{\epsilonmu,\epsilonmu_j\} \left(\bbP_0[\calE_j]+\bbP_j[\calE_j^c]\right)
    \end{align}
     According to Lemma~\ref{lem:pinsker} and Lemma~\ref{lem:lb_KLdecomp},  we have
    \begin{align}
        &\bbP_0[\calE_j]+\bbP_j[\calE_j^c]\geq \frac{1}{2}\exp\{-\rmKL(\bbP_1,\bbP_j)\}
        \\
        &
        \rmKL(\bbP_1,\bbP_j)=\sum_{i=1}^L  \bbE_0[ N_i(T) ] \cdot \rmKL \big(\nu_i^{(0)}, \nu_i^{(j)} \big)
        =\sum_{i\in P_j}  \bbE_0[ N_i(T) ] \cdot \rmKL \big(\nu_i^{(0)}, \nu_i^{(j)} \big)
    \end{align}

    Thus
    \begin{align}
        &\Reg_{\instance_0}(T)+\Reg_{\instance_j}(T)
        \geq 
        \frac{T}{2}\min\{\epsilonmu,\epsilonmu_j\}\cdot \frac{1}{2}\exp\{-\rmKL(\bbP_1,\bbP_j)\}
        \\
        \Longleftrightarrow \quad
        &
        \frac{\ln\left(\Reg_{\instance_0}(T)+\Reg_{\instance_j}(T)\right)}{\ln T}
        \geq
        1 + \frac{\min\{\epsilonmu,\epsilonmu_j\}/4}{\ln T}-\frac{\rmKL(\bbP_1,\bbP_j)}{\ln T}
        \\
        \overset{(*)}{\Longrightarrow} \quad
        &
        \frac{\rmKL(\bbP_1,\bbP_j)}{\ln T}  \geq  1
        \\
        \Longleftrightarrow \quad
        &
        \frac{\sum_{i\in P_j}  \bbE_0[ N_i(T) ] }{\ln T}  \geq  \frac{1}{\rmKL \big(\nu_{P_j}^{(0)}, \nu_{P_j}^{(j)} \big)}=\frac{2\frac{\bsigma^2-\epsilonv}{K}}{(\frac{\epsilonmu+\epsilonmu_j}{K})^2}
        \\
        \overset{(**)}{\Longrightarrow} \quad
        &
        \frac{\sum_{i\in P_j}  \bbE_0[ N_i(T) ] }{\ln T}  
        \geq  
        \frac{2\frac{\bsigma^2-\epsilonv}{K}}{(\frac{\epsilonmu}{K})^2}=:T(\nu_{P_j}^{(0)})
    \end{align}
    where in $(*)$ we let $T\to\infty$ and note that both $\Reg_{\instance_0}(T)$ and $\Reg_{\instance_j}(T)$ are of order $o(T^a),\forall a>0$; in $(**)$ we take supremum over $\epsilonmu_j>0$.

    We also have to take the safeness constraint into consideration. 
    According to Lemma~\ref{lem:safe_check}, in order to check the safeness of 
    $P_j$, the items in $P_j$ have to be sampled
    \begin{align}
        &\sum_{i\in P_j} \bbE_{\instance_0}[ N_i(\tau) ]  \ge\sup_{ \nu_{P_j}^\prime \in E(\nu_{P_j}^{(0)}) } \frac{ d(\delta_T, 1-\delta_T )  }{ \rmKL( \nu_{P_j}^{(0)}, \nu_{P_j}^\prime)  }.
        \\
        \Longleftrightarrow \quad
        &\frac{\sum_{i\in P_j} \bbE_{\instance_0}[ N_i(\tau) ] }{\ln T} 
        \ge
        \sup_{ \nu_{P_j}^\prime  \in E(\nu_{P_j}^{(0)}) } \frac{ 1 }{ \rmKL( \nu_{P_j}^{(0)}, \nu_{P_j}^\prime  )  }\cdot\frac{d(\delta_T, 1-\delta_T )}{\ln T}
        \geq
        \frac{4K^2(\bsigma^2/K)^2}{(\epsilonv)^2}\frac{\ln\frac{1}{2.4\delta_T}}{\ln T}
        :=T_{\safe}(\nu_{P_j}^{(0)}). \label{equ:low_safesub_safe}
    \end{align}
    
    \underline{If $T(\nu_{P_j}^{(0)})\leq T_\safe(\nu_{P_j}^{(0)})$}, it indicates the suboptimality of $P_j$ is identified before the safeness. Furthermore, whenever a solution $S\subset P_j$ is sampled, $S$ can have most $K-1$ items. So the regret is lower bounded by 
    \begin{align}
        &\frac{\Reg[j](T)}{\ln T}\geq \frac{T(\nu_{P_j}^{(0)})}{K-1}\cdot (\epsilonmu+\Delta-\frac{\epsilonmu}{K})
        \geq
        \frac{2K^2\frac{\bsigma^2-\epsilonv}{K}}{(\epsilonmu)^2}\cdot (\frac{\Delta}{K-1}+\frac{\epsilonmu}{K})
        \\
        \Longrightarrow\quad
        &
        \frac{\Reg[j](T)}{\ln T}\geq
        \frac{2K\frac{\bsigma^2-\epsilonv}{K}}{(\epsilonmu)^2}\cdot (\Delta+\epsilonmu)
        \geq
        \frac{K\sigma^2/2}{(\epsilonmu)^2}\cdot (\Delta+\epsilonmu)
        \\
        \Longrightarrow\quad
        &
        \Reg[j](T)=\Omega\left(\frac{K}{\Delta_{P_j}^2}(\Delta_{P_j}+\Delta)\ln T\right)\label{equ:low_safesub_1}
    \end{align}
    
    \underline{If $T(\nu_{P_j}^{(0)})\geq T_\safe(\nu_{P_j}^{(0)})$}, it indicates the suboptimality of $P_j$ is identified after the safeness. Thus, whenever a solution $S\subset P_j$ is sampled, $S$ can have most $K-1$ items before the safeness checking is finished. So the regret is lower bounded by 
    \begin{align}
        \frac{\Reg[j](T)}{\ln T}&\geq \frac{T_{\safe}(\nu_{P_j}^{(0)})}{K-1}\cdot (\epsilonmu+\Delta-\frac{\epsilonmu}{K})+\frac{T(\nu_{P_j}^{(0)})-T_{\safe}(\nu_{P_j}^{(0)})}{K}\cdot \epsilonmu
        \\
        &
        \geq \frac{T(\nu_{P_j}^{(0)})}{K}\cdot \epsilonmu+ 
        (\Delta-\frac{\epsilonmu}{K})\cdot\frac{T_{\safe}(\nu_{P_j}^{(0)})}{K-1}
        \\
        &
        \geq \frac{2K\frac{\bsigma^2-\epsilonv}{K}}{(\epsilonmu)^2}\cdot \epsilonmu
                +(\Delta-\frac{\epsilonmu}{K})\cdot 
                \frac{ 1 }{K-1} \cdot \frac{4K^2(\bsigma^2/K)^2}{(\epsilonv)^2}\frac{\ln\frac{1}{2.4\delta_T}}{\ln T}
        \\
        &\geq\frac{K\sigma^2/2}{\epsilonmu}
                +(\Delta-\frac{\epsilonmu}{K})\cdot 
                \frac{ 1 }{K-1} \cdot \frac{K^2(\sigma^2)^2}{(\epsilonv)^2}\frac{\ln\frac{1}{2.4\delta_T}}{\ln T}
        \\
        \Longrightarrow\quad
        &
        \Reg[j](T)=\Omega\left(\frac{K}{\Delta_{P_j}}\ln T + (\Delta-\frac{\Delta_{P_j}}{K}) \frac{K}{(\Deltav_{P_j})^2}\ln \frac{1}{\delta_T}\right)\label{equ:low_safesub_2}
    \end{align}

    Based on \eqref{equ:low_safesub_1} and \eqref{equ:low_safesub_2}, the regret is
    \begin{align}
        \Reg[j](T)=\Omega\left(\frac{K}{\Delta_{P_j}}\ln T + \Delta\cdot\min\left\{\frac{K}{\Delta_{P_j}^2}\ln T, \frac{K}{(\Deltav_{P_j})^2}\ln \frac{1}{\delta_T}\right\}
        \right)
    \end{align}

    \noindent
    \textbf{\underline{Case 3: the risky paths}}

    For any risky path $P_j$ ($L_1+2\leq j\leq L_2+1$, we define instance $\instance_j=(E,\mathcal{A}_K,\nu^{(j)},\bsigma^2)$ with
    \begin{align}
        \nu_i^{(j)}=\left\{
        \begin{aligned}
            &\nu_{P_j}^{(j)}=N(\Delta+\frac{\epsilonmu}{K},\frac{\bsigma^2-\epsilonv_j}{K}), &&\quad i\in P_j
            \\
            &\nu_i^{(0)}, &&\quad i\notin P_j
        \end{aligned}
        \right.
    \end{align}
    where $\epsilonv_j$ is an arbitrarily small constant. So $P_j$ is the optimal safe solution under instance $j$.

    Fix any item $i\in P_j$, consider the event $\mathcal{E}_j=\{N_i(T)\geq\frac{T}{2}\}$, under instance $1$, $\calE_j$ indicates the optimal safe solution $P_1$ is sampled less than $\frac{T}{2}$ times; under instance $j$, $\calE_j^c$ indicates the optimal safe solution $P_j$ is sampled less than $\frac{T}{2}$ times. Therefore,
    \begin{align}\label{equ:low_bd_risk}
        \Reg_{\instance_0}(T)+\Reg_{\instance_j}(T)
        &\geq
        \frac{T}{2}\left(\Delta-\frac{K-1}{K}\epsilonmu\right) \bbP_0[\calE_j]+\frac{T}{2}\epsilonmu \bbP_j[\calE_j^c]
        \\
        &\geq
        \frac{T}{2}\epsilonmu\cdot \left(\bbP_0[\calE_j]+\bbP_j[\calE_j^c]\right)
    \end{align}
     According to Lemma~\ref{lem:pinsker} and Lemma~\ref{lem:lb_KLdecomp},  we have
    \begin{align}
        &\bbP_0[\calE_j]+\bbP_j[\calE_j^c]\geq \frac{1}{2}\exp\{-\rmKL(\bbP_1,\bbP_j)\}
        \\
        &
        \rmKL(\bbP_1,\bbP_j)=\sum_{i=1}^L  \bbE_0[ N_i(T) ] \cdot \rmKL \big(\nu_i^{(0)}, \nu_i^{(j)} \big)
        =\sum_{i\in P_j}  \bbE_0[ N_i(T) ] \cdot \rmKL \big(\nu_i^{(0)}, \nu_i^{(j)} \big)
    \end{align}

    Thus
    \begin{align}
        &\Reg_{\instance_0}(T)+\Reg_{\instance_j}(T)
        \geq 
        \frac{T}{2}\epsilonmu\cdot \frac{1}{2}\exp\{-\rmKL(\bbP_1,\bbP_j)\}
        \\
        \Longleftrightarrow \quad
        &
        \frac{\ln\left(\Reg_{\instance_0}(T)+\Reg_{\instance_j}(T)\right)}{\ln T}
        \geq
        1 + \frac{\min\{\epsilonmu,\epsilonmu_j\}/4}{\ln T}-\frac{\rmKL(\bbP_1,\bbP_j)}{\ln T}
        \\
        \overset{(*)}{\Longrightarrow} \quad
        &
        \frac{\rmKL(\bbP_1,\bbP_j)}{\ln T}  \geq  1
        \\
        \Longleftrightarrow \quad
        &
        \frac{\sum_{i\in P_j}  \bbE_0[ N_i(T) ] }{\ln T}  \geq  \frac{1}{\rmKL \big(\nu_{P_j}^{(0)}, \nu_{P_j}^{(j)} \big)}
        =
        \frac{K^2(\sigma^2)^2}{(\epsilonv)^2}
        =:T(\nu_{P_j}^{(0)})
    \end{align}
    where in $(*)$ we let $T\to\infty$ and note that both $\Reg_{\instance_0}(T)$ and $\Reg_{\instance_j}(T)$ are of order $o(T^a),\forall a>0$.

    Note that $P_j$ is a risky path and if any solution $S\subset P_j$ is sampled, $|S|\leq K-1$. Thus, $\bbE[M_j(T)]\geq \frac{T(\nu_{P_j}^{(0)})}{K-1}\ln T$. The regret is lower bounded by 
    \begin{align}
        \Reg[j](T)&=
        \Omega\left(\left(\Delta-\frac{K-1}{K}\epsilonmu\right)\frac{1}{K-1}\frac{K^2(\sigma^2)^2}{(\epsilonv)^2}\ln T
        \right)
        \\
        &=
        \Omega\left(\frac{K\Delta}{(\Deltav_{P_j})^2}\ln T
        \right)
    \end{align}

    \noindent
    \textbf{\underline{Case 4: the unsafe and suboptimal paths}}
    
    For any unsafe and suboptimal path $P_j (L_2+2\leq j\leq L_0)$,
    we define instance $\instance_j=(E,\mathcal{A}_K,\nu^{(j)},\bsigma^2)$ with
    \begin{align}
        \nu_i^{(j)}=\left\{
        \begin{aligned}
            &\nu_{P_j}^{(j)}=N(\Delta+\frac{\epsilonmu_j}{K},\frac{\bsigma^2-\epsilonv_j}{K}), &&\quad i\in P_j
            \\
            &\nu_i^{(0)}, &&\quad i\notin P_j
        \end{aligned}
        \right.
    \end{align}
    where $\epsilonmu_j<\Delta, \epsilonv_j<\frac{\bsigma^2}{K}$ are arbitrarily small constants. So $P_j$ is the optimal safe solution under instance $j$.

    Fix any item $i\in P_j$, consider the event $\mathcal{E}_j=\{N_i(T)\geq\frac{T}{2}\}$, under instance $1$, $\calE_j$ indicates the optimal safe solution $P_1$ is sampled less than $\frac{T}{2}$ times; under instance $j$, $\calE_j^c$ indicates the optimal safe solution $P_j$ is sampled less than $\frac{T}{2}$ times. Therefore,
    \begin{align}\label{equ:low_bd_unsafesub}
        \Reg_{\instance_0}(T)+\Reg_{\instance_j}(T)
        &\geq
        \frac{T}{2}\epsilonmu \bbP_0[\calE_j]+\frac{T}{2}\epsilonmu_j \bbP_j[\calE_j^c]
        \\
        &\geq
        \frac{T}{2}\min\{\epsilonmu,\epsilonmu_j\} \left(\bbP_0[\calE_j]+\bbP_j[\calE_j^c]\right)
    \end{align}
     According to Lemma~\ref{lem:pinsker} and Lemma~\ref{lem:lb_KLdecomp},  we have
    \begin{align}
        &\bbP_0[\calE_j]+\bbP_j[\calE_j^c]\geq \frac{1}{2}\exp\{-\rmKL(\bbP_1,\bbP_j)\}
        \\
        &
        \rmKL(\bbP_1,\bbP_j)=\sum_{i=1}^L  \bbE_0[ N_i(T) ] \cdot \rmKL \big(\nu_i^{(0)}, \nu_i^{(j)} \big)
        =\sum_{i\in P_j}  \bbE_0[ N_i(T) ] \cdot \rmKL \big(\nu_i^{(0)}, \nu_i^{(j)} \big)
    \end{align}

    Thus
    \begin{align}
        &\Reg_{\instance_0}(T)+\Reg_{\instance_j}(T)
        \geq 
        \frac{T}{2}\min\{\epsilonmu,\epsilonmu_j\}\cdot \frac{1}{2}\exp\{-\rmKL(\bbP_1,\bbP_j)\}
        \\
        \Longleftrightarrow \quad
        &
        \frac{\ln\left(\Reg_{\instance_0}(T)+\Reg_{\instance_j}(T)\right)}{\ln T}
        \geq
        1 + \frac{\min\{\epsilonmu,\epsilonmu_j\}/4}{\ln T}-\frac{\rmKL(\bbP_1,\bbP_j)}{\ln T}
        \\
        \overset{(*)}{\Longrightarrow} \quad
        &
        \frac{\rmKL(\bbP_1,\bbP_j)}{\ln T}  \geq  1
        \\
        \Longleftrightarrow \quad
        &
        \frac{\sum_{i\in P_j}  \bbE_0[ N_i(T) ] }{\ln T}  \geq  \frac{1}{\rmKL \big(\nu_{P_j}^{(0)}, \nu_{P_j}^{(j)} \big)}
        =
        4K^2\left(\left(\frac{\epsilonv_j+\epsilonv}{(\bsigma^2-\epsilonv_j)/K}\right)^2
        +
        \frac{2(\epsilonmu_j+\epsilonmu)^2}{(\bsigma^2-\epsilonv_j)/K}\right)^{-1}
        \\
         \overset{(**)}{\Longrightarrow} \quad
         &
           \frac{\sum_{i\in P_j}  \bbE_0[ N_i(T) ] }{\ln T}  
           \geq  
        4K^2\left(\left(\frac{\epsilonv}{\sigma^2/2}\right)^2
        +
        \frac{2(\epsilonmu)^2}{\sigma^2/2}\right)^{-1}
        =:T(\nu_{P_j}^{(0)})
    \end{align}
    where in $(*)$ we let $T\to\infty$ and note that both $\Reg_{\instance_0}(T)$ and $\Reg_{\instance_j}(T)$ are of order $o(T^a),\forall a>0$; in $(**)$ we take the supremum over $\epsilonmu_j>0,\epsilonv_j>0$.

    Note that $P_j$ is an unsafe path and if any solution $S\subset P_j$ is sampled, $|S|\leq K-1$. Thus, $\bbE[M_j(T)]\geq \frac{T(\nu_{P_j}^{(0)})}{K-1}\ln T$. The regret is lower bounded by 
    \begin{align}
        \Reg[j](T)&\geq
        \left(\Delta+\frac{K-1}{K}\epsilonmu\right)\frac{4K^2}{K-1} \left(\left(\frac{\epsilonv}{\sigma^2/2}\right)^2
        +
        \frac{2(\epsilonmu)^2}{\sigma^2/2}\right)^{-1}\ln T
        \\
        &=
        \Omega\left(\frac{K\Delta+K\epsilonmu}{\max\{\epsilonmu,\epsilonv\}^2}\ln T
        \right)
        \\
        &=
        \Omega\left(\min\left\{\frac{K\Delta}{\Delta_{P_j}^2}\ln T,
        \frac{K\Delta}{(\Deltav_{P_j})^2}\ln T\right\}
        \right)
    \end{align}

\textbf{In conclusion},
the regret yielded from these paths is lower bounded by
\begin{align}
    \frac{\Reg(T)}{1-\delta_T}
    &\geq
    \Reg[1](T)+\sum_{j=2}^{L_1+1}\Reg[j](T)+\sum_{j=L_1+2}^{L_2+1}\Reg[j](T)+\sum_{j=L_2+2}^{L_0}\Reg[j](T)
    \\
    &
    \geq
    \Omega\left(\frac{K\Delta}{(\epsilonv)^2}\ln\frac{1}{\delta_T}\right)
    +\sum_{j=2}^{L_1+1} \Omega\left(\frac{K}{\epsilonmu}\ln T + \Delta\cdot\min\left\{\frac{K}{(\epsilonmu)^2}\ln T, \frac{K}{(\epsilonv)^2}\ln \frac{1}{\delta_T}\right\}
        \right)
    \\
    &
    \quad +
    \sum_{j=L_1+2}^{L_2+1}\Omega\left(\frac{K\Delta}{(\epsilonv)^2}\ln T
        \right)
    +\sum_{j=L_2+2}^{L_0}\Omega\left(\min\left\{\frac{K\Delta}{(\epsilonmu)^2}\ln T
    ,\frac{K\Delta}{(\epsilonv)^2}\ln T\right\}
        \right).
\end{align}
    Note
    \begin{itemize}
        \item $L_1=L_2-L_1=L_0-L_2-1=\frac{L_0-1}{3}$, $L_0=\frac{L}{K}$ and $\mu^\star =K\Delta$.
        \item for $S^\star$, $\epsilonv=\Deltav_{S^\star}$.
        \item for $S\in\Safeset\cap\Subopt$ and $i\in S$, we check that if $S\subset P_j,2\leq j\leq L_1+1$ , $ \Delta_{i,\Safeset\cap\Subopt,\min}=\epsilonmu$ and
        \begin{align}
            \Psi^\prime_{i, \Safeset \cap \Subopt}\geq \min\left\{ \frac{\ln T}{\Delta_S^2},\frac{9\ln (1/\delta_T)}{(\Delta_S^\rmv)^2}  \right\}.
        \end{align} 
        where the equality holds when $S=P_j,2\leq j\leq L_1+1$.  
        \item for $S\in\Riskset$ and $i\in S$, $\epsilonv=\Deltav_{i,\Riskset}$.
        \item $S\in\Safeset^c\cap\Subopt$ and $i\in S$, we can easily check $S=P_j$ for some $L_2+2\leq j\leq L_0$, thus $\Delta_{i,\Safeset^c\cap\Subopt,\min}=\epsilonmu$ and 
            \begin{align}
                \Phi_{i, \Safeset^c \cap \Subopt}= \min\left\{ \frac{\ln T}{\Delta_S^2},\frac{9\ln T}{(\Delta_S^\rmv)^2}  \right\}.
            \end{align}
    \end{itemize}
    Therefore, 
    \begin{align}
    \Reg(T)
    &\geq
    \Omega\left(\frac{L\ln T}{\Delta_{i,\Safeset\cap\Subopt,\min}} \right)+
    \frac{\mu^\star }{K}\cdot\Omega\left(\frac{K\cdot\ln(1/\delta_T)}{(\Delta_{S^\star})^2}
    +\sum_{i\in E}\Psi^\prime_{i,\Safeset\cap\Subopt} +
    \sum_{i\in E}\frac{\ln T}{(\Deltav_{i,\Riskset})^2}
    +\sum_{i\in E}\Phi_{i,\Safeset^c\cap\Subopt}
        \right).
\end{align}

\end{proof}

\thmLowBdProbDep*

\begin{proof}
    We divide the whole ground set into several paths. 
Meanwhile, we define four sets $E_1, \ldots, E_4$ such that $E_i \cap E_j = \emptyset$ for any $i\ne j$. $P_1 = E_1$, $P_j \subset E_j$ for $j=2,3,4$. For any $j>4$, there exists $k\ne 1$ such that $P_j \subset E_k$. 
For any path $P_j$, if $P_j \subset E_4$, $|P_j| = K_2$; otherwise, $|P_j| = K_1$. $P_j$ consists of arms
\begin{align*}
    \sum_{i=1}^{j-1} |P_i| + 1, \ldots, 
    \sum_{i=1}^{j-1} |P_i| + K_1\cdot \mathbbm{1}\{P_j \not\subset E_4\} + K_2 \cdot \mathbbm{1}\{ P_j \subset E_4 \}.
\end{align*}
The feasible solution set $\mathcal{A}_K$ consists of all subsets of each path.

We let $\bern(a)$ denote the Bernoulli distribution with parameter $a (a\in (0,1))$. Note that the variance of $\bern(a)$ is $a(1-a)$.
We construct an instance with $\nu_i  = \mathrm{Bern}( \mu_i  )$.
With $\varepsilon_1,\varepsilon_2>0$, we set
\begin{align*}
    &
    \mu_i = \Delta  \quad \forall i \in E_1,
    \\&
    \mu_i = \Delta - \varepsilon_1 \quad \forall i \in E_2,
    \\&
    \mu_i = \Delta + \varepsilon_2 \quad \forall i \in E_3,
    \\&
    \mu_i = \Delta - \varepsilon_3 \quad \forall i \in E_4.
\end{align*}
We let $2 \le K_1\ <  K_2 \le K$, $\Delta<1/2$ and
\begin{align*}
    & 
    K_1\Delta(1-\Delta) < \bar{\sigma}^2 \text{ (paths in $E_1$ and $E_2$ are safe)},
    \\&
    K_1(\Delta+\varepsilon_2)(1-\Delta-\varepsilon_2) > \bar{\sigma}^2 \text{ (paths in $E_3$ are unsafe)},
    \\&
    K_2(\Delta-\varepsilon_3)(1-\Delta+\varepsilon_3) > \bar{\sigma} \text{ (paths in $E_4$ are unsafe)},
    \\&
    K_2(\Delta-\varepsilon_3) < K_1\Delta  \text{ (paths in $E_4$ are suboptimal)},
    \\&
    (K_1-1) \cdot ( \Delta+\varepsilon_2 )
    < K_1  \cdot \Delta 
    ~\Leftrightarrow~
    (K_1-1) \cdot   \varepsilon_2 
    <  \Delta 
    ~\Leftrightarrow~
    \varepsilon_2
    <  \frac{  \Delta}{K_1-1} \text{ (paths in $E_3$ are suboptimal or unsafe)}.
\end{align*}
The conditions above indicate that
\begin{itemize}
    \item $P_1$ is the unique optimal safe set;
    \item if $P_j \subset E_2$, $P_j$ and its subsets are safe but suboptimal;
    \item if $P_j \subset E_3$, $P_j$ is risky, and its proper subsets are suboptimal;
    \item if $P_j \subset E_4$, $P_j$ and its subsets are suboptimal, and $P_j$ is unsafe.
\end{itemize}
Let 
\begin{align*}
	p_1 := \frac{1 - \sqrt{1 - \bar{\sigma}^2/K_1} }{2 }
	~\text{ and }~
	p_2 := \frac{1 - \sqrt{1 - \bar{\sigma}^2/K_2} }{2 },
\end{align*}
The relations between $\mu_i$'s are as in the following figure:
\begin{figure}[H]
    \centering
    \hspace*{1em}
    \includegraphics[width=.9\textwidth]{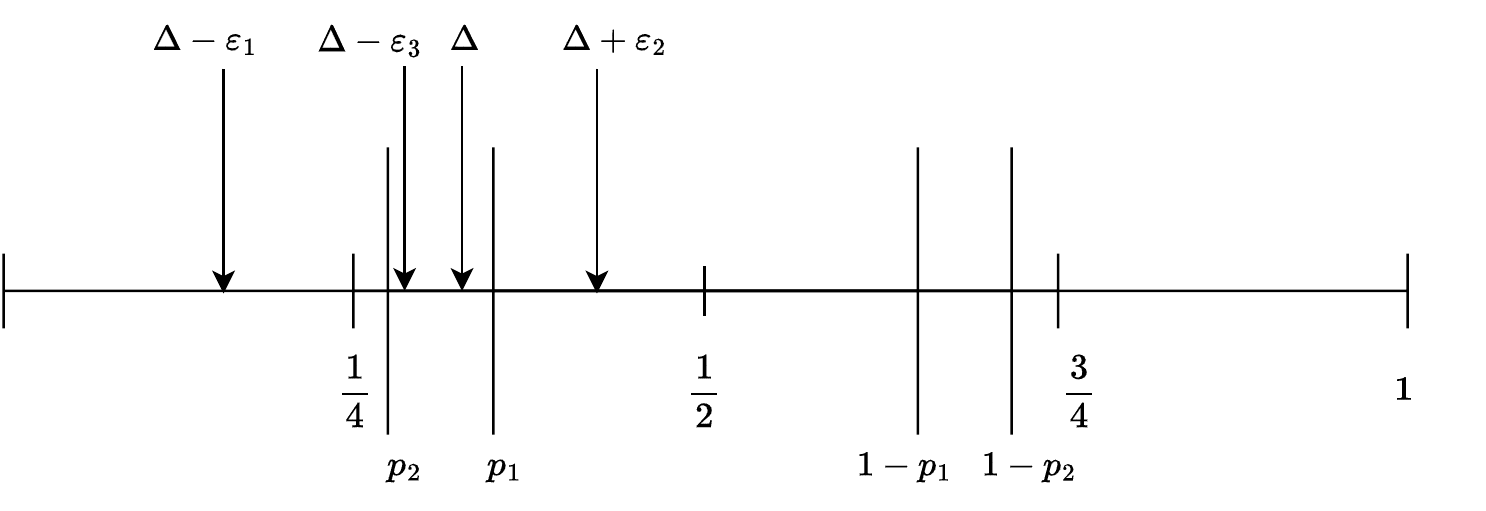}
\end{figure}
With a similar proof to that of Theorem~\ref{thm:low_bd_prob_dep_gauss}, we have,
the accumulative regret from the optimal $P_1$ is lower bounded by 
    \begin{align}
        \Reg[1](T)\geq 
        \Omega(\frac{K_1\Delta}{(\Deltav_{P_1})^2}\ln\frac{1}{\delta_T});
    \end{align}
the accumulative regret from a safe and suboptimal path in $E_2$ is lower bounded by  
    \begin{align}
        \Reg[j](T) \ge \Omega\left(\frac{K_1}{\Delta_{P_j}}\ln T + \Delta\cdot\min\left\{\frac{K_1}{\Delta_{P_j}^2}\ln T, \frac{K_1}{(\Deltav_{P_j})^2}\ln \frac{1}{\delta_T}\right\}
        \right);
    \end{align}
%
the accumulative regret from a risky path in $E_3$ is lower bounded by  
    \begin{align}
        \Reg[j](T)
        &\ge
        \Omega\left(\frac{K_1\Delta}{(\Deltav_{P_j})^2}\ln T
        \right);
    \end{align}

the accumulative regret from a unsafe and suboptimal path in $E_4$ is lower bounded by      
    \begin{align}
        \Reg[j](T)&\geq
        \Omega\left(\min\left\{\frac{K_2\Delta}{\Delta_{P_j}^2}\ln T,
        \frac{K_2\Delta}{(\Deltav_{P_j})^2}\ln T\right\}
        \right).
    \end{align}
We let
\begin{align*}
    \epsilon^\mu = \min_j \{ \Delta_{P_j} \},\quad
    \epsilonv = \min_j \{ \Deltav_{P_j} \};
\end{align*}
 and set $K_1$, $K_2$, $\varepsilon_1$, $\varepsilon_2$, $\varepsilon_3$ such that
 \begin{align*}
     & K_1 > \frac{3}{4} \cdot K_2, \quad K_2 = K, 
     \quad
     \min_j \{ \Delta_{P_j} \} < 2 \epsilon^\mu,
     \quad 
     \min_j \{ \Deltav_{P_j} \} < 2 \epsilonv.
 \end{align*}
\textbf{In conclusion},
the regret yielded from these paths is lower bounded by
\begin{align}
    \frac{\Reg(T)}{1-\delta_T}
    &\geq
    \Reg[1](T)+\sum_{P_j\in E_2}\Reg[j](T)+\sum_{P_j\in E_3}\Reg[j](T)+\sum_{P_j\in E_4}\Reg[j](T)
    \\
    &
    \geq
    \Omega\left(\frac{K \Delta}{(\epsilonv)^2}\ln\frac{1}{\delta_T}\right)
    + |E_2| \cdot \Omega\left(\frac{K}{\epsilonmu}\ln T + \Delta\cdot\min\left\{\frac{K }{(\epsilonmu)^2}\ln T, \frac{K}{(\epsilonv)^2}\ln \frac{1}{\delta_T}\right\}
        \right)
    \\
    &
    \quad +
    |E_3| \cdot \Omega\left(\frac{K \Delta}{(\epsilonv)^2}\ln T
        \right)
    +|E_4|\cdot \Omega\left(\min\left\{\frac{K \Delta}{(\epsilonmu)^2}\ln T
    ,\frac{K \Delta}{(\epsilonv)^2}\ln T\right\}
        \right).
\end{align}
 Note that
    \begin{itemize}
        \item $\mu^\star =K_1\Delta$.
        \item for $S^\star$, $\epsilonv=\Deltav_{S^\star}$.
        \item for $S\in\Safeset\cap\Subopt$ and $i\in S$, we check that if $S\subset P_j,j\in E_2$ , $ \Delta_{i,\Safeset\cap\Subopt,\min}=\epsilonmu$ and
        \begin{align}
            \Psi^\prime_{i, \Safeset \cap \Subopt}\geq \min\left\{ \frac{\ln T}{\Delta_S^2},\frac{9\ln (1/\delta_T)}{(\Delta_S^\rmv)^2}  \right\}.
        \end{align} 
        where the equality holds when $S=P_j,P_j\in E_2$.  
        \item for $S\in\Riskset$ and $i\in S$, $\epsilonv=\Deltav_{i,\Riskset}$.
        \item $S\in\Safeset^c\cap\Subopt$ and $i\in S$, we can easily check $S=P_j$ for some $P_j\in E_4$, thus $\Delta_{i,\Safeset^c\cap\Subopt,\min}=\epsilonmu$ and 
            \begin{align}
                \Phi_{i, \Safeset^c \cap \Subopt}= \min\left\{ \frac{\ln T}{\Delta_S^2},\frac{9\ln T}{(\Delta_S^\rmv)^2}  \right\}.
            \end{align}
    \end{itemize}
    Therefore, 
    \begin{align}
    \Reg(T)
    &\geq
    \Omega\left(\frac{L\ln T}{\Delta_{i,\Safeset\cap\Subopt,\min}} \right)+
    \frac{\mu^\star }{K}\cdot\Omega\left(\frac{K\cdot\ln(1/\delta_T)}{(\Delta_{S^\star})^2}
    +\sum_{i\in E}\Psi^\prime_{i,\Safeset\cap\Subopt} +
    \sum_{i\in E}\frac{\ln T}{(\Deltav_{i,\Riskset})^2}
    +\sum_{i\in E}\Phi_{i,\Safeset^c\cap\Subopt}
        \right).
\end{align}

\end{proof}

\subsection{Problem-independent Lower bound}\label{sec:proof_LwBd_indep}
\thmLowBdProbIndep*

\begin{proof}

Since the rewards of items  are bounded in $[0,1]$, the variance of each arm is at most ${1}/{4}$.
Therefore, when $\bar{\sigma}^2 \in[ K/4,\infty)$, any solution in $\calA_K$ is safe and there exists a generic lower bound \cite{Tight2015Kveton}.
When $\bar{\sigma}^2 \in (0,K/4)$, let 
\begin{align*}
	\bar{a} := \frac{1 + \sqrt{1 - \bar{\sigma}^2/K} }{2 }
	~\text{ and }~
	\underline{a} := \frac{1 - \sqrt{1 - \bar{\sigma}^2/K} }{2 } . 
\end{align*}
We consider the instances containing items with Bernoulli reward distributions. We let $\mathrm{Bern}(a)$ denote the Bernoulli distribution with mean $a$. An item $ i\in[L] $ with reward distribution $ \mathrm{Bern}(\mu_i) $ is with variance $\mu_i(1-\mu_i)$, which is smaller than $\bar{\sigma}^2/K$ if and only if $\mu_i\in (0, \underline{a}) \cup ( \bar{a},1 )$.

We will construct $3$ instances such that 
	under instance $k$~($0\le k \le 2$), the stochastic reward of an item in path $j$~($1\le j \le L_0$) is drawn from distribution 
	$  \nu_j^{(k)}  = \mathrm{Bern}(\mu_j^{(k)}) $,
	where $ \mu_j^{(k)} $ will be specified in each case.

Under instance $0$, with $\mu_0 < \underline{a}$, we let $ \mu_j^{(0)} = \mu_0$ for all $j$, i.e.,  the reward distribution of each item is $\mathrm{Bern}(\mu_0)$. Since $\mu_{P_j} = K \mu_0$ and $\sigma^2_{P_j} = K \mu_0 (1-\mu_0) < \bar{\sigma}^2$ for all $j$, each path is an identical safe and optimal solution.
Since all paths are equivalent under instance $0$, we have $\bbE_0[ M_j(t) ] = T/ L_0$ for all $j \in 1, \ldots, L_0$, where $\bbE_0$ denote the expectation under instance $0$.

We next construct instance $1$ such that 
\begin{align*}
    \mu_1^{(1)} = \mu_1, \qquad \mu_j^{(1)} = \mu_0  \quad j \ne 1,
\end{align*}
where $  \mu_0 < \mu_1 < \underline{a} $.%
\footnote{In this proof, $\mu_1$ are $\mu_2$ are redefined to minimize clutter; the previous definitions of them not used.}
Hence, $P_1$ is the unique optimal safe solution under instance $1$ while all other solutions are safe but suboptimal.

With an analysis similar to that of Lemma 6.4 in \citet{zhong2021thompson}, we can show that
\begin{restatable}{lem}{lemKLDecomp}
    Let the reward distribution of item $i$ be $\nu_i^{(j)}$ under instance $j (j=1,2)$, then
    \begin{align*}
        \rmKL \big( \calH_T^{(1)}, \calH_T^{(2)} \big) = \sum_{i=1}^L  \bbE_0[ N_i(T) ] \cdot \rmKL \big(\nu_i^{(1)}, \nu_i^{(2)} \big)
        .
    \end{align*}     

\end{restatable}
Hence, we  have
\begin{align*}
     \rmKL \big( \calH_T^{(0)}, \calH_T^{(1)} \big) = \sum_{i\in P_1}  \bbE_0[ N_i(T) ] \cdot \big(\nu_i^{(0)}, \nu_i^{(1)} \big)
        \overset{(a)}{\le } K \cdot \bbE_0[ M_1(t) ] \cdot d  (\mu_0, \mu_1)
        = K \cdot \frac{ T }{ L_0} \cdot  d (\mu_0, \mu_1),
\end{align*}
where $(a)$ follows from the fact that at most $K$ items are selected at one time step and the definition of instances.

Next, we apply Pinsker's Inequality to bound $\bbE_1[ M_1(T) ]$.
%
%
%
Lemma~\ref{lem:pinsker} indicates that
\begin{align*}
    & \left|  \bbE_0[ M_1(T) ]  - \bbE_1[ M_1(T) ] \right| \le \sqrt{ \frac{1}{2} \rmKL \big( \calH_T^{(0)}, \calH_T^{(1)} \big)},
    \\ \Rightarrow~ &
    \left|   \bbE_1[ M_1(T) ] - \frac{ T }{L_0} \right| \le \sqrt{ \frac{KT}{2 L_0} \cdot  d  (\mu_0, \mu_1) } . 
\end{align*}
Moreover, since the paths $P_j$ for $j \ne 1$ are identical under instance $1$, we have
\begin{align*}
    \bbE_1[ M_j(T) ] = \frac{1}{L_0 - 1} \big(T - \bbE_1[  M_1(T) ] \big) 
    \ge \frac{1}{L_0 - 1} \Bigg(T - \frac{T}{L_0} - \sqrt{ \frac{KT}{2 L_0} \cdot  d  (\mu_0, \mu_1) } ~\Bigg): =  \underline{M} .
\end{align*}

To learn the regret incurred by $P_2$ under instance $1$, we need to take the effects of the safety constraint into consideration. 
Lemma~\ref{lem:safe_check} indicates that
\begin{itemize}
    \item if $\underline{M} < T(\nu_0)$, at each of the first $\underline{M}$ time steps in $S_2(T)$, at most $K-1$ items are pulled  and  regret of at least $[ K(\mu_1 - \mu_2) + \mu_2] \cdot  \underline{M}$ is incurred, i.e.,
        \begin{align*}
            \Reg[2]\ge [ K(\mu_1 - \mu_2) + \mu_2] \cdot  \underline{M}.
        \end{align*}
    \item if $\underline{M} \ge T(\nu_0)$,  at each of the first $ T(\nu_0)$ time steps in $S_2(T)$, at most $K-1$ items are pulled   and  regret of at least $[ K(\mu_1 - \mu_2) + \mu_2] \cdot T(\nu_0)$ is incurred; besides, at the subsequent time steps in $S_2(T)$, regret of at least $K(\mu_1 - \mu_2)  \cdot [ \underline{M}  - T(\nu_0)]$ is incurred, i.e.,
        \begin{align*}
            \Reg[2] \ge [ K(\mu_1 - \mu_2) + \mu_2] \cdot T(\nu_0) + K(\mu_1 - \mu_2)  \cdot [ \underline{M}  - T(\nu_0)]
            =  K(\mu_1 - \mu_2)  \cdot   \underline{M} + \mu_2  \cdot T(\nu_0) 
            .
        \end{align*}
\end{itemize}
In short, we have
\begin{align*}
    \Reg[2]\ge K(\mu_1 - \mu_2)  \cdot   \underline{M} + \mu_2  \cdot  \min \{T(\nu_0) , \underline{M} \}. 
\end{align*}
We can lower bound $\Reg(j) $ for all $j=3,\ldots,L_0$ with the same method.

Besides, since 
\begin{align*}
    \bbE_1[ M_1(T) ]  
    \ge  T - \sqrt{ \frac{KT}{2 L_0} \cdot d (\mu_0, \mu_1) } 
\end{align*}
and at most $K-1$ items are selected at each of the first $ T(\nu_1)$ time steps in $S_1(T)$, we have
\begin{align*}
    \Reg[1] \ge \mu_1 \cdot  \min\bigg\{ T(\nu_1),  T - \sqrt{ \frac{KT}{2 L_0} \cdot  d  (\mu_0, \mu_1) }  \bigg\}.
\end{align*}
Therefore, under instance $1$, we have
\begin{align*}
    \frac{ \Reg(T) }{1-\delta}
    & \ge  \sum_{j=1}^{L_0} \Reg[j]
    \\ &
    \ge \mu_1 \cdot  \min\bigg\{ T(\nu_1),  T - \sqrt{ \frac{KT}{2 L_0} \cdot  d (\mu_0, \mu_1) }  \bigg\} 
        + (L_0-1) \cdot  [ K(\mu_1 - \mu_2)  \cdot   \underline{M} + \mu_2  \cdot  \min \{T(\nu_0) , \underline{M} \} ]
    \\ &
    = \mu_1 \cdot  \min\bigg\{ \sup_{ \nu_1' \in E(\nu_1) }\frac{ 1 }{K-1} \cdot \frac{ d(\delta, 1-\delta ) }{ \rmKL( \nu_1, \nu_1' )  },       \,     T - \sqrt{ \frac{KT}{2 L_0} \cdot  d  (\mu_0, \mu_1) }  \bigg\} 
        \\& \hspace{2em}
        + (L_0-1) \cdot  \bigg[ K(\mu_1 - \mu_2)  \cdot   \underline{M} + \mu_2  \cdot  \min \bigg\{  \sup_{ \nu_0' \in E(\nu_0) }\frac{ 1 }{K-1} \cdot \frac{  d(\delta, 1-\delta ) }{ \rmKL( \nu_0, \nu_0' )  }, \, \underline{M} \bigg\} \bigg]
\end{align*}
where
\begin{align*}
    &
    E(\nu_0) = \{  \nu(0'): \text{ the variance related to } \nu(0') \text{ is larger than } \bar{\sigma}^2/K \},
    \\&
    E(\nu_1) = \{  \nu(1'): \text{ the variance related to } \nu(1') \text{ is larger than } \bar{\sigma}^2/K \},
    \\&
     \underline{M} = \frac{1}{L_0 - 1} \Bigg(T -  \frac{T}{L_0} - \sqrt{ \frac{KT}{2 L_0} \cdot  d (\mu_0, \mu_1) } ~\Bigg).
\end{align*}
By   Pinsker's inequality (see Lemma~\ref{lem:pinsker}),
for $\mu_i\ge 1/2$,
	\begin{align*}
		d( \mu_i,  \bar{a} ) \le \frac{ ( \mu_i - \bar{a} )^2  \cdot (1- \mu_i -\bar{a} )^2 }{  \underline{a} (1- \mu_i -\bar{a} )^2 }  
		= \frac{   [ \mu_i(1-\mu_i)  - \bar{\sigma}^2/K ]^2 }{ \underline{a} (1- \mu_i -\bar{a} )^2 }
		\le  \frac{  [ \mu_i(1-\mu_i) - \bar{\sigma}^2/K ]^2 }{ \underline{a} ( \bar{a} -1/2)^2 }
		= \frac{  [ \mu_i(1-\mu_i)  - \bar{\sigma}^2/K ]^2 }{ \underline{a} ( 1/2 - \underline{a} )^2 };
	\end{align*}
	for $\mu_i < 1/2$,
	\begin{align*}
		d( \mu_i,  \underline{a} ) \le \frac{ ( \mu_i - \underline{a} )^2  \cdot (1- \mu_i -\underline{a} )^2 }{  \underline{a} (1- \mu_i -\underline{a} )^2 }  
		= \frac{  [ \mu_i(1-\mu_i) - \bar{\sigma}^2/K ]^2 }{ \underline{a} (1- \mu_i -\underline{a} )^2 }
		\le  \frac{    [ \mu_i(1-\mu_i)  - \bar{\sigma}^2/K]^2 }{ \underline{a} ( 1/2 - \underline{a} )^2 }.
	\end{align*}

Since $\nu_0 = \mathrm{Bern}(\mu_0)$, $\nu_1 = \mathrm{Bern}(\mu_1)$, and $0<\mu_0<\mu_1<\underline{a} < 1/2$, we have
\begin{align*}
     &
     \sup_{ \nu_0' \in E(\nu_0) }\frac{ 1 }{K-1} \cdot \frac{  d(\delta, 1-\delta ) }{ \rmKL( \nu_0, \nu_0' )  }
     = \frac{   d(\delta, 1-\delta ) }{K-1} \cdot \frac{   ( 1/2 - \underline{a} )^2 }{ [\mu_0(0-\mu_0) - \bar{\sigma}^2/K ]^2 },
     \\&
     \sup_{ \nu_1' \in E(\nu_1) }\frac{ 1 }{K-1} \cdot \frac{  d(\delta, 1-\delta ) }{ \rmKL( \nu_1, \nu_1' )  }
     = \frac{   d(\delta, 1-\delta ) }{K-1} \cdot \frac{   ( 1/2 - \underline{a} )^2 }{ [ \mu_1(1-\mu_1) - \bar{\sigma}^2/K]^2 }.
\end{align*}
We define the minimum variance gap $\Delta^v : = \min_{S\in \calA} \Delta_S^v$.
{
When $K^3\le L^2$ and $LK/T \le \underline{a}^2$, we can let $\mu_1 - \mu_2 = \sqrt{ L/KT }$.} Then we have
\begin{align*}
    \Reg (T)
    & = \Omega\left(   \min \left\{ \frac{ 1}{K-1} \cdot  \frac{d (\delta,1-\delta)  }{ (\Delta^v/K)^2 } , T \right\} +  \sqrt{KLT} + L_0 \cdot \min \left\{ \frac{ 1}{K-1} \cdot  \frac{d (\delta,1-\delta)  }{ (\Delta^v/K)^2 } , \frac{  TK}{L} \right\} \right)
    \\&
    =  \Omega\left(    \min \left\{ K \cdot  \frac{d (\delta,1-\delta)  }{ (\Delta^v/ )^2 } , T \right\} +   \sqrt{KLT} + L \cdot \min \left\{   \frac{d (\delta,1-\delta)  }{ (\Delta^v )^2 } , \frac{  T }{L} \right\} \right)
    \\&
    =  \Omega\left(   \sqrt{KLT} +  \min \left\{ L \cdot  \frac{d (\delta,1-\delta)  }{ (\Delta^v  )^2 } , T \right\}   \right).
\end{align*}
Moreover, we complete the proof with $d(\delta,1-\delta) \ge -\ln(2.4\delta)$ and $\delta_T= \delta$. 

\end{proof}

\section{Tightness of the Upper bound}\label{sec:tightness}
 \begin{restatable}[Tightness of problem-dependent bounds]{cor}{corTightProbdep_complete1}\label{cor:tightness}
Let $\{\delta_T\}_{T=1}^\infty \in o(1)$ be a sequence that satisfies $\ln(1/\delta_T)=o(T^b)$ for all $b>0$,
\begin{itemize}
    \item if $\ln(1/\delta_T)\in o(\ln T)$, in particular, if $\delta_T=\delta_0>0$ for all~$ T$, the regret  $\mathrm{Reg}(T)$ is
    \begin{align}
        &\Omega\left(\sum_{i\in E}\frac{\ln T}{\Delta_{i,\Safeset\cap\Subopt,\min}}+ \frac{\mu^\star \ln T/K}{(\Deltav_{i,\Riskset})^2}+\frac{\mu^\star }{K}\Phi_{i,\Safeset^c\cap\Subopt}\right)\\
        \cap\,&
        O\left(\sum_{i\in E}\frac{K\ln T}{\Delta_{i,\Safeset\cap\Subopt,\min}}
        +\frac{\mu^\star K\ln T}{(\Deltav_{i,\Riskset})^2}
        +\mu^\star K\Phi_{i,\Safeset^c\cap\Subopt}
        \right)
    \end{align}
    \item if $\ln(1/\delta_T)=\lambda \ln T$, i.e., $\delta_T=T^{-\lambda}$ with a fixed $\lambda>0$, the regret  $\mathrm{Reg}(T)$ is
    \begin{align}
        &\Omega\left(\sum_{i\in E}\frac{\ln T}{\Delta_{i,\Safeset\cap\Subopt,\min}}
        +\frac{\lambda\mu^\star \ln T}{(\Deltav_{S^\star})^2}
        +\frac{\mu^\star }{K}\sum_{i\in E}\Big(\Psi^\prime_{i,\Safeset\cap\Subopt}
        +\frac{\ln T}{(\Deltav_{i,\Riskset})^2}
        +\Phi_{i,\Safeset^c\cap\Subopt}\Big)
        \right)\\
        \cap\,&
        O\left(
        \sum_{i\in E}\frac{K\ln T}{\Delta_{i,\Safeset\cap\Subopt,\min}}
        +\frac{\lambda\mu^\star K^2\ln T}{(\Deltav_{S^\star})^2}
        +K\mu^\star \sum_{i\in E}\Big(\Psi_{i,\Safeset\cap\Subopt}+\frac{\ln T}{(\Deltav_{i,\Riskset})^2}
                    +\Phi_{i,\Safeset^c\cap\Subopt}\Big)
        \right)
    \end{align}
    where $\ln(1/\delta_T)$ in the $\Psi$ and $\Psi^\prime$ functions should be replaced by $\lambda\ln T$.
    \item if $\ln(1/\delta_T)\in \omega(\ln T)$, the regret  $\mathrm{Reg}(T)$ is
    \begin{align}
        \Reg(T)\in \Omega\left(\frac{\mu^\star \ln(1/\delta_T)}{(\Deltav_{S^\star})^2}\right)
        \cap
        O\left(\frac{\mu^\star  K^2\ln(1/\delta_T)}{(\Deltav_{S^\star})^2}\right)
    \end{align}
\end{itemize}
The upper bounds above are achieved by {\sc PASCombUCB}.
\end{restatable}



 \section{Additional Discussions and Future Research} \label{sec:additional}
 \subsection{Discussions on the Tightness Results}
 In terms of the problem-dependent bounds in Corollary~\ref{cor:tightnessDep}, we consider general instances where the rewards from the items are independent and the gap in $\Reg_1(T)$ can be closed when that the rewards from the items are correlated, as in the lower bound for the unconstrained combinatorial bandits in \citet{Tight2015Kveton}. This assumption also allows us to remove a factor of $K$ from the gap of $\Reg_2(T)$. 

 In terms of the problem-independent bounds in Corollary~\ref{cor:tightnessIndep}, the regret due to suboptimality is almost tight as in the unconstrained case~\citep{Tight2015Kveton}. The regret due to safeness checking is tight up to a factor of $K^2$. During each phase, \AlgSCombUCB selects and samples solutions which are disjoint subsets of $A_{p}$, and hence one item is sampled at most once during one phase. However, it is empirically feasible to sample some items  more than once during one phase, which will help reduce the regret but requires more delicate analysis.

 For future directions, it is of interest to close the gap (the factor $K$) in the regret due to safeness checking with improved analyses or additional assumptions on the instance.
 \subsection{Discussions on the Problem Formulation}
 \emph{Anytime safety} is important in safety-critical applications where \emph{at each point in time} the risk cannot exceed a certain threshold. For example, in a self-driving car that is scheduled to move from start point $x_0$ to end point $x_n$ via $(x_1, x_2,\ldots, x_{n-1})$ (the choice of these waypoints is a combinatorial problem), it is necessary that the car stay in its designated lane at all points in time, and not just ``on average'', otherwise a catastrophic accident might result at some point in time with non-negligible probability. In this example, we might want to choose a route that is safe at all times w.h.p. (in the sense that the car stays in its designated lane) at the cost of a longer travel time. 
 
 This work studies the anytime-safe constraint at the \emph{super-arm}\footnote{In this discussion, the terms ``super arm'' and ``base arm'' refer to ``solution'' and ''item'' respectively.} level and provides a first step to understanding risk in combinatorial semi-bandits. The \emph{sum} $\sum_{i\in A}\sigma_i^2$ of a set of items (base arms) in $A$ is adopted as the risk measure, which is a certain function of $\sigma_i^2,  i\in A$.
 It is also of interest to study the anytime safety at the \emph{base arm level}, where the risk function is  $\max_{i\in A}\sigma_i^2$. From a technical point of view, if the safeness of any single base arm has not been ascertained (as we need to learn the safeness/risk of the base arms), then pulling any base arm can be risky, in the sense that with non-vanishing probability, the anytime-safe constraint (or even the less stringent stagewise safety constraint~\citep{Khezeli2020Safe}) is violated when we do the exploration (by pulling any base arm) \emph{at the beginning}. Thus, this seems infeasible from a technical standpoint. 
 Nevertheless, we believe additional proper assumptions can be made to formulate a practical and feasible model that leads to future researches.

 \subsection{Comparisons}
  \underline{\textbf{Comparison with \citet{Wu2016Conservative,Kazerouni2017ConsContLB}:}}
  While the conservative (linear) constraint 
  $$
\mathbb{P}\left[\bigcap_{t\ge 1}\left\{\sum_{k=1}^t\left\langle X_k, \theta^*\right\rangle \geq(1-\alpha) tb_0\right\} \right] \geq 1-\delta
$$
requires the constraint should be satisfied w.p. $1-\delta$ over the whole horizon, which is similar to our probably anytime-safe constraint, the constraint is in terms of the \emph{cumulative} expected reward (up to time step $t$). The cumulative nature of the conservative constraint maintains a ``budget reservoir'' that makes this constraint less stringent than the stagewise safety constraint~\citep{Khezeli2020Safe}, in the sense that one algorithm may satisfy the conservative constraint but violate the stagewise safety constraint. Both the stagewise safety constraint and the probably anytime-safe constraint consider the reward/risk that is incurred at each \emph{single} time step.

 \underline{\textbf{Comparison with \citet{Khezeli2020Safe,Moradipari2020Stagewise}:}}
To the best of our knowledge, the \emph{stagewise safety} (or the \emph{stagewise conservative)} constraint \citep{Khezeli2020Safe,Moradipari2020Stagewise} is the most related risk-aware constraint to our anytime-safe constraint. \newline
(1) The stagewise conservative constraint is \textit{a constraint on the mean reward} (hence, only one statistics is involved in the problem), which originates from the conservative constraint~\citep{Wu2016Conservative}. In our setup, we post the anytime-safe constraint on the risk while minimizing the regret, which requires us to consider \textit{two statistics and the interaction between them}.\newline
(2) The stagewise safety constraint has only been utilized under the linear bandit setup in the literature, where the arm set is assumed to be a \emph{convex and compact} set in $\mathbb{R}^{n}$~\citep{Moradipari2020Stagewise}, and thus, the arms constitute an \emph{uncountable continuous} set. If an arm $A$ is known to be safe, then it is safe to pull any arm ``near'' $A$. However, in the combinatorial bandit setup, such a continuity property of the arm set does not hold since it is ``discrete''. Specifically, given that super arm $A$ is safe (but not absolutely safe), even the safeness of a nearby arm $\tilde{A}$, which is obtained by replacing one single base arm in $A$ with another base arm, cannot be guaranteed by the safeness of $A$. \newline
(3) In terms of the safety level, consider the stagewise safety constraint with a constant confidence parameter $\delta$ (independent of $T$); intuitively, the safety constraint can be violated approximately $\delta T$ times, which is \emph{linear} in $T$. In addition, consider an algorithm which does safeness checking first, followed by exploration-and-exploitation on the safe super arms, it takes $\Theta(\frac{1}{(\Delta_A^\rmv)^2}\ln\frac{1}{\delta})$ pulls to identify the unsafeness of an unsafe super arm $A$. Note that the time required is independent of $T$, so the regret due to safeness checking is $o(T^a)$ for all $a>0$. From this perspective, the stagewise safety constraint is not stringent enough and can be easily satisfied by such a naive algorithm. A more direct intuition (yet not completely rigorous) is, if the algorithm ignores the stagewise safety constraint, it only takes $o(T^a)$ for all $a>0$ to rule out the unsafe super arms, which indicates it satisfies the stagewise safety constraint (since the unsafe super arms are pulled $o(T^a)<\delta T$ times).
From another point of view, given a confidence parameter $\delta$, if the stagewise safety constraint is satisfied w.p. $1-\delta_t$ at time step $t$ with $\sum_{t=1}^T \delta_t=\delta$, then the probably anytime-safe constraint is met w.p. $1-\delta$.\newline
(4) Besides the constant confidence parameter $\delta$, we have investigated the \emph{whole spectrum} of $\delta$ in terms of the time horizon $T$. The tightness result (Corollary~\ref{cor:tightness}) indicates our algorithm is capable of dealing with \emph{an even stricter constraint} (in the sense that $\delta$ decreases with respect to $T$) and we provide a sharp threshold on the achievability of $o(T^a)$ (for all $a>0$) regret (see Theorem~\ref{thm:impos}).



 \end{document}